\pdfoutput=1
\documentclass{article}

\usepackage[final,nonatbib]{neurips_2023}

\usepackage[utf8]{inputenc} 
\usepackage[T1]{fontenc}    
\usepackage{hyperref}       
\usepackage{url}            
\usepackage{booktabs}       
\usepackage{amsfonts}       
\usepackage{nicefrac}       
\usepackage{microtype}      
\usepackage{xcolor}         
\usepackage{caption}

\usepackage{amsmath}
\usepackage{amssymb}
\usepackage{mathtools}
\usepackage{amsthm}

\usepackage[textsize=tiny]{todonotes}

\newcommand{\algemph}[3]{\algcolor{#1}{#2}{#3}}

\usepackage{xcolor}
\usepackage{comment}

\usepackage{graphicx}
\usepackage{soul}
\usepackage{subcaption}
\usepackage{booktabs}
\usepackage{tablefootnote}

\usepackage{amsmath,amsthm,amssymb,amsfonts}
\usepackage{enumerate}
\usepackage{cleveref}
\usepackage{comment}
\usepackage{bm}
\usepackage{pifont}
\usepackage{epsf}
\usepackage{graphics}
\usepackage{wrapfig}
\usepackage{psfrag}

\usepackage{color}

\usepackage{mathtools}
\usepackage{amsfonts}
\usepackage{amsthm}
\usepackage{amsmath}
\usepackage{amssymb}

\newtheorem{theorem}{Theorem}
\newtheorem{proposition}{Proposition}
\newtheorem{lemma}{Lemma}

\newtheorem{definition}{Definition}

\newtheorem{assumption}{Assumption}

\long\def\comment#1{}

\newcommand{\norm}[1]{\left\| #1 \right\|}
\newcommand{\pare}[1]{\left( #1 \right)}

\newcommand{\inprod}[2]{\ensuremath{\left\langle #1 , \, #2 \right\rangle}}

\newcommand{\E}{\ensuremath{{\mathbb{E}}}}

\DeclareMathOperator{\sign}{sign}

\newcommand{\R}{\mathbb{R}}

\newcommand{\cH}{\mathcal{H}}  

\newcommand{\cW}{\mathcal{W}}  
\newcommand{\cP}{\mathcal{P}}

\newcommand{\balpha}{\boldsymbol{\alpha}}
\newcommand{\bdelta}{\boldsymbol{\delta}}

\newcommand{\bx}{\boldsymbol{x}}

\newcommand{\bv}{\boldsymbol{v}}
\newcommand{\bw}{\boldsymbol{w}}
\newcommand{\bz}{\boldsymbol{z}}

\newcommand{\bg}{\boldsymbol{g}}

\newcommand{\mA}{\boldsymbol{A}}

\newcommand{\mS}{\boldsymbol{S}}

\newcommand{\mQ}{\boldsymbol{Q}}

\newcommand{\mB}{{\bf B}}

\newcommand{\mH}{{\bf H}}
\newcommand{\mI}{{\bf I}}

\definecolor{MITBrown}{RGB}{164, 31, 50}

\theoremstyle{plain}

\usepackage{algorithm}
\usepackage[ruled,algo2e,noend]{algorithm2e}
\usepackage{hyperref}
\SetCommentSty{small}

\hypersetup{colorlinks=true,linkcolor=cyan,filecolor=magenta,urlcolor=cyan,}

\definecolor{antiquewhite}{rgb}{0.98, 0.92, 0.84} 
\definecolor{blizzardblue}{rgb}{0.67, 0.9, 0.93}
\newcommand{\algcolor}[3]{\hspace*{-\fboxsep}\colorbox{#1}{\parbox{#2\linewidth}{#3}}}
\usepackage{enumitem}

\newcommand\scalemath[2]{\scalebox{#1}{\mbox{\ensuremath{\displaystyle #2}}}}

\title{Distributed Personalized Empirical Risk Minimization}

\author{
	Yuyang Deng \\
	Pennsylvania State University\\
	\texttt{yzd82@psu.edu} \\
	\And
	Mohammad Mahdi Kamani\\
	Wyze Labs \\
	\texttt{mmkamani@alumni.psu.edu} \\
	\And
	Pouria Mahdavinia \\
	Pennsylvania State University\\
	\texttt{pxm5426@psu.edu} \\
    \And
	Mehrdad Mahdavi \\
	Pennsylvania State University\\
	\texttt{mzm616@psu.edu}
}

\begin{document}

\maketitle

\begin{abstract}
This paper advocates a new paradigm  Personalized Empirical Risk Minimization (PERM) to facilitate learning from heterogeneous data sources without imposing stringent constraints on computational resources shared by participating devices. In PERM, we aim to learn a distinct model for each client by learning who to learn with and personalizing the aggregation of local empirical losses by effectively estimating the statistical discrepancy among data distributions, which entails optimal statistical accuracy for all local distributions and overcomes the data heterogeneity issue.  To learn personalized models at scale,  we propose a distributed algorithm that replaces the standard model averaging with model shuffling to simultaneously optimize 
PERM objectives for all devices. This also allows us to learn distinct model architectures (e.g., neural networks with different numbers of parameters) for different clients, thus confining underlying memory and compute resources of individual clients. We rigorously analyze the convergence of the proposed algorithm and conduct experiments that corroborate the effectiveness of the proposed paradigm.

\end{abstract}

\section{Introduction}


Recently \textit{federated learning} (FL) has emerged as an alternative paradigm to centralized learning to encourage federated model sharing and create a framework to support edge intelligence by shifting model training and inference from data centers to potentially scattered---and perhaps self-interested---systems where data is generated~\cite{kairouz2019advances}.  While undoubtedly being a better paradigm than centralized learning, enabling the widespread adoption of FL necessitates foundational advances
in the efficient use of statistical and computational resources to encourage a large pool of individuals or corporations to share their private data and resources. Specifically, due to \textit{heterogeneity} of data and compute resources among participants, it is necessary,  if not imperative, to develop distributed algorithms that are i) cognizant of statistical heterogeneity (\textit{data-awareness}) by designing algorithms that effectively deal with highly heterogeneous data distributions across devices; and ii) confined to learning models  that meet available computational resources of participant devices  (\textit{system-awareness}).

To mitigate the negative effect of data heterogeneity (non-IIDness), two common approaches are clustering and personalization. The key idea behind the clustering-based methods~\cite{mansour2020three,li2021federated,ghosh2020efficient,ma2022convergence} is to partition the devices into clusters (coalitions) of similar data distributions and  then learn a single shared model for all clients within
each cluster. While appealing, the partitioning methods are limited to heuristic ideas such as clustering based on the geographical distribution of devices without taking the actual data distributions into account and lack theoretical guarantees or postulate strong assumptions on initial models or data distributions~\cite{ghosh2020efficient,ma2022convergence}. In personalization-based methods~\cite{eichner2019semi,t2020personalized,huang2020personalized,mansour2020three,fallah2020personalized,smith2017federated,deng2020adaptive}, the idea is to learn a distinct \textit{personalized model}   for each device alongside   the global model, which can be unified as minimizing a bi-level optimization problem~\cite{hanzely2021personalized}.  Personalization aims to learn a model that
has the generalization capabilities of the global model but can also perform well on
the specific data distribution of each participant suffers from a few key limitations. First, as the number of clients grows, while the number of training data increases, the number of parameters to be learned increases which limits to increase in the number of clients beyond a certain point to balance data and overall model complexity tradeoff-- a phenomenon known as incidental parameters problem~\cite{lancaster2000incidental}. Moreover, since the knowledge transfer among data sources happens through a single global model, it might lead to suboptimal results. To see this, consider an extreme example, where half of the users have identical data distributions, say $\mathcal{D}$, while the other half share a data distribution that is substantially different, say $-\mathcal{D}$ (e.g.  two distributions with same marginal distribution on features but opposite labeling functions). In this case, the global model obtained by naively aggregating local models (e.g., fixed mixture weights) converges to a solution that suffers from low test accuracy on all local distributions which makes it  preferable to learn a model for each client  solely based on its local data or carefully chosen subset of data sources. 

Focusing on system heterogeneity, most existing works require learning models of identical architecture to be deployed across the clients and server (model-homogeneity)~\cite{mcmahan2017communication,karimireddy2020scaffold}, and mostly focus on reducing number~\cite{karimireddy2020scaffold} or size~\cite{hamer2020fedboost,haddadpour2021federated,sunfedspeed} of  communications  or sampling handling chaotic availability of clients~\cite{yang2022anarchic,wang2022unified}.  The requirement of the same model makes it infeasible to train large models due to \textit{system heterogeneity} where client devices have drastically different computational resources.  A few recent studies aim to overcome this issue either by leveraging knowledge distillation methods~\cite{wu2022communication,he2020group,lin2020ensemble,itahara2021distillation} or partial training (PT) strategies via model subsampling (either static~\cite{diao2020heterofl,horvath2021fjord}, random~\cite{caldas2018expanding}, or rolling~\cite{alam2022fedrolex}). However, KD-based methods require having access to a  public representative dataset of all local datasets at server  and ignore data heterogeneity in the distillation stage to a large extent. The focus of PT training methods is mostly  on learning a single  server model using heterogeneous resources of devices and does not aim at deploying a model onto each client after the global server model is trained (which is left as a future direction in~\cite{alam2022fedrolex}). The aforementioned  issues lead to a fundamental question: \textit{``What is the best strategy to learn  from heterogeneous  data sources  to achieve optimal  accuracy w.r.t. each data source, without imposing stringent constraints  on computational resources shared by participating devices?''}.

We answer this question affirmatively,  by proposing a new \textit{data\&system-aware} paradigm dubbed Personalized Empirical Risk Minimization (PERM), to facilitate learning from massively fragmented private data under resource constraints. Motivated by generalization bounds in multiple source domain adaptation~\cite{ben2010theory,mansour2014robust,konstantinov2019robust,crammer2008learning}, in PERM we aim to learn a distinct model for each client by   \textit{personalizing the aggregation} of empirical losses of different data sources which  enables each client  \textit{to learn who to learn with} using an effective method to empirically estimate the statistical discrepancy between their associated data distributions. We argue that PERM entails  optimal statistical accuracy for all local distributions, thus overcoming the data heterogeneity issue.  PERM can also be employed in other learning settings with  multiple heterogeneous sources of data such as domain adaptation and multi-task learning to entail optimal statistical accuracy.  While PERM overcomes the data heterogeneity issue, the number of optimization problems (i.e., distinct personalized ERMs) to be solved scales linearly with the number of data sources. To simultaneously optimize all  objectives in a scalable and computationally efficient manner, we propose a novel  idea that replaces the standard \textit{model averaging} in distributed learning with \textit{model shuffling} and establish its convergence rate. This also allows us to learn distinct model architectures (e.g., neural networks with different number of parameters) for  different clients, thus confining to  underlying  memory and compute resources of individual clients, and overcoming the system heterogeneity issue. This  addresses an open question in~\cite{alam2022fedrolex} where only  a single global model can be trained in a model-heterogeneous setting, while PERM allows deploying distinct models for different clients. We empirically evaluate the performance of PERM,  which corroborates the statistical benefits of PERM  in comparison to existing methods.  



\section{Personalized Empirical Risk Minimization}\label{sec:perm}
 
 In this section, we formally state the problem and introduce PERM  as an ideal paradigm for learning from heterogeneous data sources. We assume there are $N$ distributed devices where each holds a distinct data shard $\mathcal{S}_i = \{(\boldsymbol{x}_{i,j},y_{i,j})\}_{j=1}^{n_i}$ with $n_i$ training samples that are realized by   a local source distribution $\mathcal{D}_i$ over instance space $\Xi = \mathcal{X} \times \mathcal{Y}$. The data distributions across the devices are not independently and identically distributed (non-IID or  \textit{heterogeneous}), i.e., $\mathcal{D}_1 \neq \mathcal{D}_2 \neq \ldots \neq \mathcal{D}_N$, and each distribution corresponds to a local \textit{generalization error}  or \textit{true risk} $\mathcal{L}_i(h) = \mathbb{E}_{(\boldsymbol{x}, y) \sim \mathcal{D}_i}[\ell(h(\boldsymbol{x}), y)], i=1, 2, \ldots, N$ on \textit{unseen} samples for any model $h \in \mathcal{H}$, where $\mathcal{H}$ is the hypothesis set (e.g., a linear model or a deep neural network) and   $\ell:\mathcal{Y}\times \mathcal{Y} \rightarrow \mathbb{R}^+$ is a given convex or non-convex loss function. We use $\widehat{\mathcal{L}}_i(h) = (1/n_i)\sum_{(\boldsymbol{x}, y) \in \mathcal{S}_i}{\ell\left(h(\boldsymbol{x}), y)\right)}$ to denote the \textit{local empirical risk} or training loss  at $i$th data shard $\mathcal{S}_i$ with $n_i$ samples. 
 
 We seek to collaboratively learn a model or personalized models that entail a good generalization on all local distributions, i.e. minimizing true risk $\mathcal{L}_i(\cdot), i=1, \ldots, N$ for all data sources (all-for-all~\cite{even2022sample}). A simple  non-personalized solution, particularly in FL, aims to minimize a (weighted) \textit{empirical risk} over all data shards in a communication-efficient manner~\cite{konevcny2016federated}: 
\begin{equation}\tag{WERM}
\arg\min\nolimits_{h \in \mathcal{H}}\sum\nolimits_{i=1}^{N} p(i) \widehat{\mathcal{L}}_{i}(h) \; \text{with}   \;  \boldsymbol{p} \in \Delta_N,    \label{eq:werm}\vspace{-2mm}
\end{equation}
where $\Delta_N = \{\bm{p} \in  \mathbb{R}_{+}^{N} \; | \; \sum\nolimits_{i=1}^{N}{p(i)} = 1\}$  denotes the simplex set. 

It has been  shown that a \textit{single model} learned by~\ref{eq:werm}, for example by using fixed mixing weights $p(i)=n_i/n$, where $n$ is total number of training samples, or even agnostic to mixture of  distributions~\cite{mohri2019agnostic,deng2020distributionally},  while yielding a good  performance on the \textit{combined} datasets of all devices,  can suffer from a  poor generalization error on individual datasets by increasing the diversity among  distributions~\cite{li2019feddane,karimireddy2019scaffold,haddadpour2019convergence,yu2020salvaging}. To overcome this issue, there has been a surge of interest in developing  methods that personalize the global model to individual local distributions. These methods can be unified as the following bi-level  problem (a similar unification has been made in~\cite{hanzely2021personalized}):
\begin{align}\tag{BERM} \label{eq:loss}
& \arg\min\nolimits_{\substack{h_1, h_2, \ldots, h_m \in \mathcal{H}}}  \;\;
\textcolor{black}{\widehat{\mathcal{F}}}_{i}(h_i \textcolor{black}{\oplus} h_*) \quad \text{subject to} \quad  h_* = \arg\min\nolimits_{h \in \mathcal{H}}\sum\nolimits_{j=1}^{N}{\alpha(j) \widehat{\mathcal{L}}_j(h)}  
\end{align}
where $\oplus$ denotes the mixing operation to combine local and  global models,  and  \textcolor{black}{$\widehat{\mathcal{F}}_i$} is  a modified local loss which  is not necessarily  same as  local risk $\widehat{\mathcal{L}}_i$. By carefully designing the local loss $\widehat{\mathcal{F}}_i$ and mixing operation $\oplus$,   we can  develop different penalization schemes for FL including existing methods such as  linearly interpolating global and local models~\cite{deng2020adaptive,mansour2020three}, multi-task learning~\cite{smith2017federated} and  meta-learning~\cite{fallah2020personalized}  as  special cases. For example, \ref{eq:loss} reduces to  \textit{zero-personalization}  objective~\ref{eq:werm} when  $h_i \textcolor{black}{\oplus} h_* = h_*$, and $ \widehat{\mathcal{F}}_i = \widehat{\mathcal{L}}_i$. At the other end of the spectrum lies the \textit{zero-collaboration}  where the $i$th client trains its own model without any influence from other clients by setting $h_i \textcolor{black}{\oplus} h_* = h_i$, $ \widehat{\mathcal{F}}_i =  \widehat{\mathcal{L}}_i$. The personalized model with \textit{interpolation} of global and local models  can be recovered by setting $h_i \textcolor{black}{\oplus} h_* = \alpha h_i + (1-\alpha)h_*$,  and $ \widehat{\mathcal{F}}_i =  \widehat{\mathcal{L}}_i$. While more effective than a single global model learned via~\ref{eq:werm}, personalization methods suffer from three key issues: i) the global model is still obtained by minimizing the average empirical loss which might limit the statistical benefits of collaboration, ii) overall model complexity increases linearly with number of clients, and iii) a same model space is shared across servers and clients. 

To motivate our proposal,  let us consider the empirical loss  $\sum_{i=1}^{N} \alpha(i) \widehat{\mathcal{L}}_{i}(h)$ in \ref{eq:werm} (or the inner level objective in~\ref{eq:loss}) with fixed mixing weights $\boldsymbol{\alpha} \in \Delta_N$, and denote the optimal solution by $\widehat{h}_{\boldsymbol{\alpha}}$. The excess risk of  the learned model $\widehat{h}_{\boldsymbol{\alpha}}$ on $i$th local distribution $\mathcal{D}_i$ w.r.t. the optimal local model $h_i^* = \arg\min_{h \in \mathcal{H}} \mathcal{L}_{\mathcal{D}_i}(h)$  (i.e. \textit{all-for-one}) can be bounded by (informal)~\cite{konstantinov2019robust} 
\begin{equation}
\vspace{0mm}\begin{aligned}
\mathcal{L}_{i}\left(\widehat{h}_{\boldsymbol{\alpha}}\right) &\leq \mathcal{L}_{i}\left(h_{i}^{*}\right) + \sum_{j=1}^{N} \alpha(j) \mathrm{R}_{j}(\mathcal{H}) +2 \sum_{j=1}^{N} \alpha(j) \mathrm{disc}_{\mathcal{H}}\left(\mathcal{D}_{j}, \mathcal{D}_{i}\right) + C \sqrt{\sum_{j=1}^{N}\frac{\alpha(j)^2}{n_j}}
\end{aligned}\label{eqn:generalization} \tag{GEN}\vspace{-2mm}
\end{equation}
where  $\mathrm{R}_{j}(\mathcal{H})$ is the empirical Rademacher complexity $\mathcal{H}$ w.r.t.  $\mathcal{S}_j$, and $\mathrm{disc}_{\mathcal{H}}(\mathcal{D}_i, \mathcal{D}_{j})$   is a pseudo-distance on the set of probability measures on $\Xi$ to assess the discrepancy between the distributions $\mathcal{D}_{i}$ and $\mathcal{D}_{j}$ with respect to the hypothesis class $\mathcal{H}$ as defined below~\cite{ben2010theory}:
\begin{definition}\label{def:disc}
For a model space $\mathcal{H}$ and $\mathcal{D}, \mathcal{D}^{\prime}$ two probability distributions on $\Xi = \mathcal{X} \times \mathcal{Y}$, 
$$
\mathrm{disc}_{\mathcal{H}}\left(\mathcal{D}, \mathcal{D}^{\prime}\right)=  \sup_{h \in \mathcal{H}} |\mathbb{E}_{\xi \sim\mathcal{D}}(\ell(h, \xi)) - \mathbb{E}_{\xi' \sim \mathcal{D}'}(\ell(h, \xi'))|\vspace{0mm}$$ \vspace{-5mm}
\end{definition}
Intuitively, the discrepancy between the two distributions is large, if there exists a predictor that performs well on one of them and badly on the other. On the other hand, if all functions in the hypothesis class perform similarly on both, then $\mathcal{D}$ and $\mathcal{D}'$ have low discrepancy. The above metric which is  a special case of a popular family of distance measures in probability theory and mathematical statistics known as integral probability metrics (IPMs)~\cite{sriperumbudur2009integral}, can be estimated from finite data by replacing the expected losses with their empirical counterparts (i.e. $\mathcal{L}_i$ with $\widehat{\mathcal{L}}_i$).

From~\ref{eqn:generalization}, it can be observed that a mismatch between  pairs of distributions limits the benefits of ERM on all distributions. Indeed,  the generalization risk w.r.t. $\mathcal{D}_i$  will significantly increase  when the distribution divergence terms $\mathrm{disc}_{\mathcal{H}}(\mathcal{D}_j, \mathcal{D}_{i})$ are large. It  leads to an ideal sample complexity  $1/\sqrt{n}$  where $n= n_1 + n_2 + \ldots + n_N$ is the total number of samples, which could have been obtained in the IID setting with $\alpha(j) = 1/N$ when the divergence is small as  the pairwise discrepancies  disappear. Also, we note that even if the global model achieves a small training error over the union of all data (e.g., over parametrized setting) and can entail a good generalization error with respect to \textit{average distribution}, the divergence term still remains which illustrates the poor performance of the global model on all local distributions $\mathcal{D}_i, i =1, 2, \ldots,N$. This implies that even personalization of the global model as in~\ref{eq:loss} can not entail a good generalization on all local distributions as there is no effective transfer of positive knowledge among data sources  in the presence of high data heterogeneity among local distributions (similar impossibility results even under seemingly generous assumptions on how distributions relate have  been made  in multisource domain adaptation as well~\cite{hanneke2022no}).

Interestingly the bound suggests that seeking optimal accuracy on \textit{all} local distributions requires   choosing a distinct mixing   of local losses for each client $i$ that minimizes the right-hand side of~\ref{eqn:generalization}. This indicates that in an ideal setting (i.e. \textit{all-for-all}), we can achieve the best accuracy for each local distribution $\mathcal{D}_i$ by \textit{personalizing} the \ref{eq:werm}, i.e., (i) first estimating $\boldsymbol{\alpha}_i, i=1,2, \ldots, N$ for each  client individually, then (ii) solving a variant of \ref{eq:werm} for each client with obtained mixing parameters:
\begin{equation}\tag{PERM}
\arg\min_{h \in \mathcal{H}_i}\sum\nolimits_{j=1}^{N} \alpha_{i}(j) \widehat{\mathcal{L}}_{j}(h) \; \quad  \text{for} \; \quad  i =1, 2,  \ldots, N.    \label{eq:perm}
\end{equation}
By doing this each device  achieves the optimal local generalization error by \textbf{learning  who to learn with} based on the number of samples at each source and the mismatch between its data distribution with other clients.  We also note that compared to~\ref{eq:werm} and~\ref{eq:loss}, in~\ref{eq:perm} since we solve a different aggregated empirical loss for each client, we can pick a different model space/model  architecture $\mathcal{H}_i$ for each client to meet its available computational resources.  

While this two-stage method is guaranteed to entail optimal test accuracy for all local distributions $\mathcal{D}_i$, however, making it scalable requires overcoming  two issues. First, estimating the statistical discrepancies between each pair of data sources (i.e., $\boldsymbol{\alpha}_i, i = 1, \ldots, N$) is a computing burden as it requires solving $O(N^2)$ difference of (non)-convex functions in a distributed manner and requires enough samples form each source to entail good accuracy on estimating pairwise discrepancies~\cite{sriperumbudur2009integral}. Second, we  need to solve $N$ variants of the optimization problem in~\ref{eq:perm}, possibly each with a different model space,   which is infeasible when the number of devices is huge (e.g., cross-device federated learning). In the next section, we propose a simple yet effective idea to overcome these issues in a computationally efficient manner.

\section{PERM at Scale via Model Shuffling}\label{sec:perm-shuffling}
In this section, we propose a method to efficiently estimate the empirical discrepancies among data sources followed by a model shuffling idea to simultaneously solve  $N$ versions of ~\ref{eq:perm} to learn a personalized model for each client. We first start by proposing a two-stage algorithm: estimating mixing parameters followed by model shuffling. Then, we propose a single loop unified algorithm that  enjoys the same computation and communication overhead as~\ref{eq:loss} (twice communication of FedAvg). For ease of exposition, we  discuss the proposed algorithms by assuming all the clients share the same model architecture and later on discuss  the generalization to heterogeneous model spaces. Specifically, we assume that the model space $\mathcal{H}$ is a parameterized by a convex set $\mathcal{W} \subseteq \mathbb{R}^d$ and use $f_i(\bm{w}) := \widehat{\mathcal{L}}_i(\bm{w}) = \sum\nolimits_{(\bm{x}, y) \in \mathcal{S}_i}^{}{\ell(\bm{w}; (\bm{x}, y))}$ to denote the empirical loss at $i$th data shard.

\vspace{-4pt}\subsection{Warmup: a two-stage algorithm}\vspace{-4pt}
We start by proposing a two-stage method for solving $N$ variants of PERM in parallel. In the first stage, we propose an efficient method to learn the mixing parameters for all clients. Then, in stage two, we propose a model shuffling method to solve all personalized empirical losses in parallel. 

\noindent\textbf{Stage 1: Mixing parameters estimation.}~In the first stage we aim to efficiently estimate the pairwise discrepancy among local distributions to construct mixing parameters $\bm{\alpha}_i, i= 1,2 ,\ldots N$.   From generalization bound~\ref{eqn:generalization} and Definition~\ref{def:disc}, a direct solution to estimate $\bm{\alpha}_i$ is to solve   the following  convex-nonconcave minimax problem for each client:
\begin{align}
   \bm{\alpha}_{i}^{*} =  \arg\min_{\balpha \in \Delta_N} \sum\nolimits_{j=1}^{N}{\alpha(j) \max_{\bm{w} \in \mathcal{W}} |f_i(\bm{w}) - f_j(\bm{w})|}+  \sum\nolimits_{j=1}^{N}{\alpha(j)^2}/{n_j}\label{eq:minimax}
\end{align}
where we estimate the true risks in pairwise discrepancy terms with their empirical counterparts and drop the complexity term as it becomes 
 identical for all sources by fixing the hypothesis space $\cH$ and bounding it with a  computable distribution-independent quantity such as VC dimension~\cite{shalev2014understanding}, or  it can be controlled by  choice of $\cH$ or through data-dependent 
regularization. However, solving the above minimax problem itself is already challenging: the inner maximization loop is a nonconcave (or difference of convex) problem, so most of the existing minimax algorithms will fail on this problem. To our best knowledge, the only provable deterministic algorithm is~\cite{xu2023unified}, and it is hard to generalize it to stochastic and distributed fashion.  Moreover, since we have $N$ clients, we need to solve $N$ variants of~(\ref{eq:minimax}), which makes designing a scalable algorithm even harder.


To overcome  aforementioned issues,  we make two relaxations to estimate the per client mixing parameters.  First, we optimize an upper bound of pairwise empirical discrepancies $\sup_{\bw} |f_i(\bw) - f_j(\bw)|$ in terms of gradient dissimilarity between local objectives $\norm{\nabla f_i( \bw) -\nabla f_j(\bw)}$~\cite{dandi2022implicit}, which quantifies how different the local empirical losses are and widely used in analysis of  learning from heterogeneous losses as in FL~\cite{zhao2018federated}. Second, given that the discrepancy measure based on the supremum could be excessively pessimistic in real-world scenarios, and drawing inspiration from the concept of average drift at the optimal point as a right  metric to measure the effect of data heterogeneity in federated learning~\cite{wang2022unreasonable}, we propose to measure discrepancy at the optimal solution  obtained by solving~\ref{eq:werm}, i.e., $\bw^*  := \arg\min_{\bw \in \cW}  (1/N)\sum_{i=1}^N f_i (\bw)$. By doing this, the problem reduces to a simple minimization for each client, given the optimal global solution. These two  relaxations lead to solving the following tractable optimization problem to decide the per-client mixing parameters:
\begin{align}
     \arg\min_{\balpha \in \Delta_N} g_i(\bw^*, \balpha):=\sum\nolimits_{j=1}^{N}{\alpha(j)\|\nabla f_i(\bm{w}^*) - \nabla f_j(\bm{w}^*)\|^2}+ \lambda \sum\nolimits_{j=1}^{N} \alpha(j)^2/n_j\label{eq:relaxed-alpha-final}
\end{align}
where we added a regularization parameter $\lambda$ and used the squared of gradient dissimilarity for computational convenience. Thus, obtaining all $N$ mixing parameters  requires solving a single ERM to obtain optimal global solution and $N$ variants of (\ref{eq:relaxed-alpha-final}). To get the optimal solution in a communication-reduced manner,  we  adapt the Local SGD algorithm~\cite{stich2018local} (or FedAvg~\cite{mcmahan2017communication}) and find the optimal solution in intermittent communication  setting~\cite{woodworth2018graph} where the clients work in parallel and are allowed to make $K$ stochastic updates  between two
communication rounds for $R$ consecutive rounds. The detailed steps are given in Algorithm~\ref{algorithm: alpha} in  Appendix~\ref{sec:app:two-stages} for completeness. After obtaining the global model $\bw^R$ we optimize over $\balpha$ in $g_i(\bw^R,\balpha)$ using  $T_{\balpha}$ iterations of GD to get $\hat{\balpha}_i$. Actually, we will show that as long as $\bw^R$ converge to $\bw^*$,  $\hat{\balpha}_i, i=1, \ldots, N$ converges to solution of~\eqref{eq:relaxed-alpha-final} very fast. Our proof idea is based on the following Lipschitzness observation: 
\begin{equation*}
\scalemath{0.93}{
\norm{\balpha^*_{g_i}(\bw^R) - \balpha^*_{g_i}(\bw^*)}^2 
    \leq   4L^2 \kappa_g^2 \sum_{j=1}^N\pare{ 2\norm{\nabla f_i( \bw^*) -\nabla f_j(\bw^*)}^2 + 4L^2\norm{  \bw^R - \bw^*}^2} \norm{ \bw^R-\bw^* }^2}
\end{equation*}
where $\balpha^*_{g_i}(\bw) := \arg\min_{\balpha \in \Delta_N} g_i(\bw, \balpha)$ and $\kappa_g:= {n_{\max}}/{(2\lambda)}$ is the condition number of $g_i(\bw,\cdot)$ where $n_{\max} = \max_{i \in [N]}n_i$. The Lipschitz constant mainly depends on {\em gradient dissimilarity at optimum}. As $\bw^R$ tends to $\bw^*$, the $\balpha^*_{g_i}(\cdot)$ becomes more Lipschitz continuous, i.e., the coefficient in front of $\norm{ \bw^R-\bw^* }^2$ getting smaller, thus leading to more accurate mixing parameters.

To establish the convergence, we make the following standard assumptions.
\begin{assumption}[{\sffamily{Smoothness and strong convexity}}]\label{assump: smooth} We assume $f_i(\boldsymbol{x})$'s are L-smooth and $\mu$-strongly convex, i.e.,
$$
\forall \boldsymbol{x}, \boldsymbol{y}:\left\|\nabla f_i(\boldsymbol{x})-\nabla f_i(\boldsymbol{y})\right\| \leq L\|\boldsymbol{x}-\boldsymbol{y}\| .
$$
$$
\forall \boldsymbol{x}, \boldsymbol{y}: f_i(\boldsymbol{y}) \geq f_i(\boldsymbol{x})+\langle\nabla f_i(\boldsymbol{x}), \boldsymbol{y}-\boldsymbol{x}\rangle+\frac{1}{2} \mu\|\boldsymbol{y}-\boldsymbol{x}\|^2
$$
\end{assumption}
 
We denote the condition number by $\kappa={L}/{\mu}$. 

\begin{assumption}[{\sffamily{Bounded variance}}]\label{assumption: bounded var}
The variance of stochastic gradients  computed at each local function is bounded, i.e.,  $\forall i \in [N], \forall \bw \in \cW, \mathbb{E}[\|\nabla f_i(\bw;\xi) - \nabla f_i(\bw)\|^2]  \leq \delta^2$.
\end{assumption}

\begin{assumption}[{\sffamily{Bounded domain}}]\label{assumption: bounded domain}
The domain $\cW \subset \R^d$ is a bounded convex set, with diameter $D$ under $\ell_2$ metric, i.e., $\forall \bw, \bw' \in \cW, \norm{\bw-\bw'} \leq D$.

\end{assumption}
\begin{definition}[{\sffamily{Gradient dissimilarity}}]\label{def: dissimilarity}
We define the following quantities to measure the gradient dissimilarity among local  functions:
{\begin{align}
  &\zeta_{i,j}(\bw) :=   \left\|  \nabla f_i(\bm{w}) -   \nabla f_j(\bm{w})\right\|^2, \quad \bar\zeta_i(\bw)  := \frac{1}{N} \sum_{j=1}^N\zeta_{i,j}(\bw),\nonumber\\
    &\zeta := \sup\nolimits_{\bw \in \cW} \max\nolimits_{i\in[N]} \|  \nabla  f_i(\bm{w}) -  ({1}/{N})\sum\nolimits_{j=1}^N \nabla f_j (\bw) \|^2.\nonumber
\end{align}}
 \end{definition}
The following theorem gives the convergence rate of estimated discrepancies to optimal counterparts. 
\begin{theorem}\label{thm: 2stage alpha}
   Under Assumptions~\ref{assump: smooth}-\ref{assumption: bounded domain}, if we run Algorithm~\ref{algorithm: alpha} on $F(\bw):= \frac{1}{N}\sum_{j=1}^N f_j (\bw) $ with $\gamma = \Theta\pare{\frac{\log(RK)}{\mu RK}}$ for $R$ rounds with synchronization gap $K$, for $\kappa_g = {1}/({\lambda n_{\min}})$, it holds that 
    \begin{align*}
        \E  \|\balpha^R_i -  \balpha^*_i\|^2 \leq \tilde{O}\pare{ \exp\pare{-\frac{T_{\balpha}}{\kappa_g}}+  \kappa_g^2 \bar{\zeta}_i(\bw^*)   L^2\pare{\frac{D^2}{RK} +     \frac{  \kappa\zeta^2}{\mu^2 R^2  }  +   \frac{  \delta^2 }{\mu^2 N RK}  }} \quad \forall i \in [N].
    \end{align*}
\end{theorem}
An immediate  implication of Theorem~\ref{thm: 2stage alpha} is that even we solve~\eqref{eq:relaxed-alpha-final} at $\bw^R$, the algorithm will eventually converge to optimal solution of~\eqref{eq:relaxed-alpha-final} at $\bw^*$. The core technique in the proof, as we mentioned, is to show that for a parameter within a small region centered at $\bw^*$, the function $\balpha^*_{g_i}(\bw)$ becomes `more Lipschitz'. The rigorous characterization of this property is captured by Lemma~\ref{lem: lipschitz alpha} in appendix.

\noindent\textbf{Stage 2: Scalable personalized optimization with model shuffling.}~After obtaining the per client mixing parameters, in the second stage we aim at solving $N$ different personalized variants of~\ref{eq:perm}  denoted by 
    $\Phi (\hat{\balpha}_1,\bv),  \Phi (\hat{\balpha}_2,\bv),  \ldots 
      \Phi (\hat{\balpha}_N ,\bv)$
to  learn local  models where 
\begin{align}
    \min_{\bv \in \cW}& \Phi (\hat{\balpha}_i,\bv) :=  \frac{1}{N}\sum\nolimits_{j=1}^N \hat{\alpha}_{i}(j) f_j(\bv). \label{eq: phi definition}
\end{align}
Here we devise an iterative algorithm based on distributed SGD with periodic averaging (a.k.a. Local SGD~\cite{stich2018local}) to solve these $N$ optimization problems in parallel with \textit{no extra overhead}. The idea is to replace the model averaging in vanilla distributed (Local) SGD with \textit{model shuffling}.  Specifically, as shown in Algorithm~\ref{algorithm: KSGD} the  algorithm proceeds for $R$ epochs where each epoch runs for $N$ communication rounds. At the beginning of each epoch $r$ the server generates a random permutation $\sigma_r$ over $N$ clients. At each communication round $j$ within the epoch, the server sends the model of client $i$  to client $i_j = (i+j)\mod N$ in the permutation $\sigma_r$ along with $\alpha_i(i_j)$. After receiving a model from the server, the client updates the received model for $K$ local steps and returns it back to the server. As it can be seen, the updates of each loss $\Phi (\hat{\balpha}_i,\bv), i=1, 2, \ldots, N$ during an epoch  is equivalent to  sequentially processing individual losses in (\ref{eq: phi definition}) which can be considered as  permutation-based SGD but with the different that each component now is updated for $K$ steps. By \textit{interleaving the  permutations}, we are able to simultaneously optimize all $N$ objectives. We note that the computation and communication complexity of the proposed algorithm is the same as Local SGD with two differences: the model averaging is replaced with model shuffling, and the algorithms run over a permutation of devices.  The convergence rate of Local SGD  is well-established in literature~\cite{woodworth2020minibatch,mishchenko2022proxskip,yuan2020federated,haddadpour2019local,gorbunov2021local}, but here we establish  the convergence of permutation-based variant which is  interesting by its own.

\begin{assumption}[Bounded Gradient]\label{assumption: bounded grad}
The variance of stochastic gradients  computed at each local function is bounded, i.e.,  $\forall i \in [N], \sup_{\bv\in\cW}\|\nabla f_i(\bv)\|\leq G$.
\end{assumption}
We note that the Assumption~\ref{assumption: bounded grad} can be realized since we work with a bounded domain $\cW$.
\begin{theorem}\label{thm:2stagec perm}
Let Assumptions~\ref{assump: smooth}-~\ref{assumption: bounded grad} hold. Assume $\balpha_i^*$ is the solution of~\eqref{eq:relaxed-alpha-final}. Then if we run {Algorithm~\ref{algorithm: KSGD}} on the $\hat{\balpha}_i$ obtained from {Algorithm~\ref{algorithm: alpha}}, then Algorithm~\ref{algorithm: KSGD} with $\eta = \Theta\pare{\frac{\log(NKR^3)}{\mu R}}$  will output the solution $\hat{\bv}_i$, $\forall i \in [N]$, such that with probability at least $1-p$, the following statement holds:
    \begin{align*}
        \E[ \Phi( {\balpha}^*_i, \hat{\bv}_i)- \Phi( {\balpha}^*_i, \bv^*({\balpha}^*_i) )]
          &\leq   \tilde{O} \left( \frac{D^2L }{NKR^2} \right) +  \frac{L\delta^2}{\mu^2  R}    +   \pare{ \frac{L^4+N}{\mu^4R^2}  } LG^2 {N\log(1/p)}  \\
         & \ +  \kappa^2_{\Phi} L\tilde{O}\pare{ \exp\pare{-\frac{T_{\balpha}}{\kappa_g}}  +  \kappa_g^2 \bar{\zeta}_i(\bw^*)   L^2\pare{      \frac{  \kappa\zeta^2}{\mu^2 R^2  }  +   \frac{  \delta^2 }{\mu^2 N RK}  }},  
         \end{align*}
    where $\kappa = \frac{L}{\mu}, \kappa_g = \frac{ n_{\max}}{2\lambda}$ and $ \kappa_\Phi = \frac{\sqrt{N}G}{\mu}$, and the expectation is taken over randomness of Algorithm~\ref{algorithm: alpha}. That is, to guarantee $\E[\Phi({\balpha}^*_i, \hat{\bv}_i) - \Phi( {\balpha}^*_i,  \bv^*_i)] \leq \epsilon$, we choose  $R = O\pare{\max\left\{\frac{L\delta^2}{\mu \epsilon},\frac{\kappa^2_\Phi\kappa_g^2  L^3  \bar\zeta_i(\bw^*)  D^2}{\epsilon } \right\} }$ and $T_{\balpha} = O\pare{\kappa_g \log\pare{\frac{L\kappa^2_\Phi}{\epsilon}}}$.\vspace{-2mm}
\end{theorem}
The above theorem shows that even though we run the optimization on $\Phi(\hat{\balpha}_i,\bv)$, our obtained model $\hat{\bv}_i$ will still converge to the optimal solution of $\Phi( {\balpha}^*_i,\bv)$.
The convergence rate is contributed from two parts: convergence of $\hat\balpha_i$ (Algorithm~\ref{algorithm: alpha}) and convergence of personalized model $\hat \bv_i$ (Algorithm~\ref{algorithm: KSGD}). Notice that, for the convergence rate of $\hat \bv_i$, we roughly recover the optimal rate of shuffling SGD~\cite{ahn2020sgd}, which is $O({1}/{R^2})$. However, we suffer from a $O({\delta^2}/{R})$ term since each client runs vanilla SGD on their local data (the \texttt{SGD-Update} procedure in Algorithm~\ref{algorithm: KSGD}). One medication for this variance term could be deploying variance reduction or shuffling data locally at each client before applying SGD. We notice that there is a recent work~\cite{cho2023convergence} also considering the client-level shuffling idea, but our work differs from it in two aspects: 1) they work with local SGD type algorithm and the shuffling idea is employed for model averaging within a subset of clients, while in our algorithm, during each local update period, each client runs shuffling SGD directly  on other's model 2) from a theoretical perspective, we are mostly interested in investigating whether the algorithm can converge to the true optimal solution of $\Phi( {\balpha}^*_i,\bv)$ if we only optimize on a surrogate function $\Phi(\hat{\balpha}_i,\bv)$. 

One drawback of Algorithm~\ref{algorithm: KSGD} is that we have to wait for Algorithm~\ref{algorithm: alpha} to finish and output $\hat{\balpha}_i$, so that we can proceed with Algorithm~\ref{algorithm: KSGD}. However, if we are not satisfied with the precision of $\hat{\balpha}_i$, we may not have a chance to go back to refine it. Hence in the next subsection, we propose to interleave Algorithm~\ref{algorithm: KSGD} and Algorithm~\ref{algorithm: alpha}, and introduce a single-loop variant of  PERM which will jointly optimize mixture weights and learn personalized models in an interleaving fashion.

 \begin{algorithm2e}[t]
	\DontPrintSemicolon
    \caption{{Shuffling Local SGD}}
	\label{algorithm: KSGD}
 
    	\textbf{Input:} Clients $1,...,N$, Number of Local Steps $K$ , Number of Epoch $R$, mixing parameter $\hat{\balpha}_1,...,\hat{\balpha}_{N}$ 
	\\

       \textbf{Epoch} \For{$r = 0,...,R-1$} 
        {  
        
        Server generates permutation $\sigma_r: [N] \mapsto [N]$.\\
        \textbf{parallel} \For{$i = 1,...,N$}	
 {
        Client $i$ sets initial model $\bv^{r,0}_i = \bv^{r}_i$.\\
	   \For{$j = 1,...,N$}
	    { Set indices $i_j = \sigma_r((i+j) \mod N)$.\\
        Server sends $\bv_i^{r,j}$ to Client $i_j$.  \\
        $ \bv^{r,j+1}_i = \texttt{SGD-Update}(\bv_i^{r,j} , \eta, i_j,K,\hat{\balpha}_i)$.\\
        
        } 
        Client $i$ does projection: $\bv^{r+1}_i = \cP_{\cW}(\bv^{r,N}_i)$.\\  
        
     }
     
 } 
 \textbf{Output:} $\hat{\bv}_i = \cP_{\cW}( \bv_i^{R}-(1/L) \nabla_{\bv} \Phi(\hat{\balpha}_i, \bv_i^{R}) ), \forall i \in [N]$. \\
 \setcounter{AlgoLine}{0}
  \SetKwProg{myproc}{\texttt{SGD-Update}($\bv,\eta,j,K, {\balpha}$)}{}{}
  \myproc{}{
   {
   Initialize $\bv^0 = \bv$\\
   \For{$t = 0,...,K-1 $}
   {$\bv^t = \bv^{t-1} - \eta  {\alpha}(j) N {\nabla} f_j(\bv^{t-1};\xi_j^{t-1})$
   }
   } 
  \textbf{Output:} $\bv^K \;$ }
\end{algorithm2e}

 

        

\vspace{-4pt}\subsection{A unified single loop algorithm}\vspace{-4pt}
We now turn to  introducing a single-stage algorithm  
that jointly optimizes $\bm{\alpha}_i$s and $\bv_i$s as depicted in Algorithm~\ref{algorithm: Single PERM} by intertwining the two stages in Algorithm~\ref{algorithm: alpha} and Algorithm~\ref{algorithm: KSGD} in a single unified method.  The idea is to learn the global model, which is used to estimate mixing parameters, concurrent to personalized models. At each communication round, the clients compute gradients on the global model, on their data, after the server collects these gradients does a step mini-batch SGD update on the global model, and then updates the mixing parameters. Then  we proceed to update the personalized models similar to Algorithm~\ref{algorithm: KSGD}. We note that, unlike the two-stage method where the mixing parameters are computed at the final global model,  here the mixing parameters are updated adaptively based on intermediate global models. 

 \begin{algorithm2e}[t]
	\DontPrintSemicolon
    \caption{{Single Loop PERM}}
	\label{algorithm: Single PERM}
 
    	\textbf{Input:} Clients $1,...,N$, Number of Local Steps $K$ , Number of Epoch $R$, Initial mixing parameter $\balpha^0_{1}=,...,\balpha^0_{{N}} = \bar{\balpha} = [1/N,...,1/N]$.
	\\

\textbf{Epoch} \For { $r = 0,...,R-1$}
	  {  
 Server generates permutation $\sigma_r: [N] \mapsto [N]$.\\
       
 \textbf{parallel} \For{ {\em Client} $i = 1,...,N$}	
 {  Client $i$ sets initial model $\bv^{r,0}_i = \bv^{r}_i$.\\
    \For{$j = 1,...,N$}
	    {
     \algemph{antiquewhite}{0.82}{Set indices $i_j = \sigma_r((i+j) \mod N)$.\\
        Server sends $\bv_i^{r,j}$ to client $i_j$ .  \\
        $ \bv^{r,j+1}_i = \texttt{SGD-Update}(\bv_i^{r,j} , \eta, i_j,K,\balpha^r_i)$.
          \quad \tcp{Personalized model update} }

        }  Client $i$ does projection: $\bv^{r+1}_i = \cP_{\cW}(\bv_i^{r,N})$.
         
     }
       \algemph{blizzardblue}{0.94}{ $ \bw^{r+1}  =  \cP_{\cW}( \bw^{r} - \gamma \frac{1}{N}\sum_{i=1}^{N}\frac{1}{M}\sum_{j=1}^{M}\nabla f_i(\bw^{r},\xi_{i,j}^r ))$
        \quad \tcp{Global model update}  
       Compute $\balpha^{r+1}_i$ by running $T_{\balpha}$ steps GD on $g_i(\bw^{r+1},\balpha)$  \ \quad  \tcp{\texttt{$\balpha$ update}}
       }
    }
   \textbf{Output:} $\hat{\bv}_i = \cP_{\cW}( \bv_i^{R}-(1/L) \nabla_{\bv} \Phi( {\balpha}^R_i, \bv_i^{R}) )$, $\hat{\balpha}_i = \balpha^{R}_i, \forall i \in [N]$.
  \setcounter{AlgoLine}{0}
  \SetKwProg{myproc}{\texttt{SGD-Update}($\bv,\eta,j,K,\balpha$)}{}{}
  \myproc{}{
    {
   Initialize $\bv_j^0 = \bv$\\
   \For{$t = 0,...,K-1 $}
   {$\bv^t = \bv^{t-1} - \eta  {\alpha}(j) N {\nabla} f_j(\bv_j^{t-1};\xi_j^{t-1})$
   }
   
   } 
   \textbf{Output:} $\bv^K$ \;}
\end{algorithm2e} 
\begin{theorem}\label{thm:single loop}
  Let Assumptions~\ref{assump: smooth} to~\ref{assumption: bounded grad} to be satisfied. Assume $\balpha_i^*$ is the solution of~\eqref{eq:relaxed-alpha-final}. Then if we run  Algorithm~\ref{algorithm: Single PERM} with $\eta = \Theta\pare{\frac{\log(NKR^3)}{\mu R}}$ and $\gamma = \Theta\pare{\frac{\log(NKR^3)}{\mu R}}$, it will output the solution $\hat{\bv}_i$, $\forall i \in [N]$, such that with probability at least $1-p$, the following statement holds:
    \begin{align*}
       \E[\Phi({\balpha}^*_i, \hat{\bv}_i) - \Phi( {\balpha}^*_i,  \bv^*_i)]  & \leq O\pare{ \frac{LD^2}{NKR^3} }   + \tilde{O} \pare{   \pare{  \frac{\kappa^4L}{  R^2}  + \frac{NL}{\mu^2 R^2}   } G^2 {N\log(1/p)} + \frac{L \delta^2}{\mu R}     } \\
        &  + \kappa^2_\Phi L\tilde{O}\pare{  \frac{ \kappa^2\kappa_g^2  L^2  \bar\zeta_i(\bw^*)  DG}{R } +  R^2  \exp\pare{ -\frac{T_{\balpha}}{\kappa_g}}     +  \frac{  L  \kappa^2 \kappa_g^2    \bar\zeta_i(\bw^*)   \delta^2}{\mu^2 M  } },\vspace{-5mm}
    \end{align*}
    where $\kappa = \frac{L}{\mu}, \kappa_g = \frac{ n_{\max}}{2\lambda}, \kappa_\Phi = \frac{\sqrt{N}G}{\mu}$ and the expectation is taken over the randomness of stochastic samples in Algorithm~\ref{algorithm: Single PERM}.
    That is, to guarantee $\E[\Phi({\balpha}^*_i, \hat{\bv}_i) - \Phi( {\balpha}^*_i,  \bv^*_i)] \leq \epsilon$, we choose $M = O\pare{\frac{L^2 \kappa^2\kappa_g^2  \kappa^2_\Phi\bar\zeta_i(\bw^*)\delta^2}{\mu^2 \epsilon}}$, $R = O\pare{\max\left\{\frac{L\delta^2}{\mu \epsilon},\frac{\kappa^2_\Phi\kappa^2\kappa_g^2  L^3  \bar\zeta_i(\bw^*)  DG}{\epsilon } \right\} }$ and $T_{\balpha} = O\pare{\kappa_g \log\pare{\frac{LR^2}{\epsilon}}}$.
\end{theorem}
Compared to Theorem~\ref{thm:2stagec perm}, we achieve  a slightly worse rate, since we need a large mini-batch when we update global model $\bw$. However, the advantages of the single-loop algorithm are two-fold. First, as we mentioned in the previous subsection, we have the freedom to optimize $\hat\balpha_i$ to arbitrary accuracy, while in double loop algorithm (Algorithm~\ref{algorithm: alpha} + Algorithm~\ref{algorithm: KSGD}), once we get $\hat \balpha_i$, we do not have the chance to further refine it. Second, in practice, a single-loop algorithm is often easier to implement and can make better use of  caches by operating on data sequentially, leading to improved performance, especially on modern processors with complex memory hierarchies.

\vspace{-4pt}\subsection{Extension to heterogeneous model setting}\vspace{-4pt}
In the homogeneous model setting, we assumed a shared model space $\mathcal{W}$ for clients and the server. However, in real-world FL applications, devices have diverse resources and can only train models that match their capacities. We demonstrate that the \ref{eq:perm} paradigm can be extended to support learning in model-heterogeneous settings, where different models with varying capacities are used by the server and clients. Focusing on learning the global optimal model to estimate pairwise statistical discrepancies, we note that by utilizing partial training methods~\cite{alam2022fedrolex}, where at each communication round a sub-model with a size proportional to resources of each client is sampled from the server's global model (extracted  either random, static, or rolling) and is transmitted  to be updated locally. Upon receiving updated sub-models, the server can  simply  aggregate (average) heterogeneous sub-model updates sent from the clients to update the global model.  We can consider the complexity of models used by clients when estimating mixing parameters by solving a modified version of (\ref{eq:relaxed-alpha-final}) as: 
\begin{equation*}
     \sum_{j=1}^{N}{\alpha(j) \sqrt{{{\text{\sffamily VC}}(\mathcal{H}_j)}/{n_j}}} + \sum_{j=1}^{N}{\alpha(j)\|\nabla f_i(\bm{m}_i\odot\bm{w}^*) - \nabla f_j(\bm{m}_j\odot\bm{w}^*)\|^2}+ \lambda \sum_{j=1}^{N} \alpha(j)^2/n_j,\vspace{-2mm}
\end{equation*}
where we simply upper bounded the Rademacher complexity w.r.t. each data source in~(\ref{eqn:generalization}) with {\sffamily{VC}} dimension~\cite{mohri2018foundations}. Here $\bm{m}_i$ is the masking operator to extract a sub-model of the global model to compute local gradients at client $i$ based on its available resources.  By doing so, we can adjust mixing parameters based on the complexity of underlying models, as different sub-models of the global model (i.e., $\bm{m}_i\odot\bm{w}^*$ versus $\bm{m}_j\odot\bm{w}^*$) are used to compute drift between pair of gradients at the optimal solution. With regards to training personalized models with heterogeneous local models,  as we solve a distinct aggregated empirical loss for each client by interleaving permutations and shuffling models,  we can utilize different model spaces $\cW_i, i=1, \ldots, N$ for different clients that meet their available resources with aforementioned partial training strategies. 

\section{Experimental Results} \label{sec:exp}
In this section we benchmark the effectiveness of PERM on synthetic data with 50 clients, where it notably outshone other renowned methods as evident in Figure~\ref{fig:synthetic}.  Our experiments concluded with the CIFAR10 dataset, employing a 2-layer convolutional neural network, where PERM, despite a warm-up phase, demonstrated unmatched convergence performance (Figure~\ref{fig:cifar10}).  Additional experiments are reported in the appendix. Across all datasets, the PERM algorithm consistently showcased its robustness and unmatched efficiency in the realm of personalized federated learning.

\noindent\textbf{Experiment on synthetic data.}~To demonstrate the superior effectiveness of our proposed single-loop PERM algorithm compared to other existing personalization methods, we conducted an experiment using synthetic data generated according to the following specifications. We consider a scenario with a total of $N$ clients, where we draw samples from the distribution $\mathcal{N}(\bm{\mu}_1,\bm{\Sigma}_i)$ for half of the clients, denoted by $i \in [1,\frac{N}{2}]$, and from $\mathcal{N}(\bm{\mu}_2,\bm{\Sigma}_i)$ for the remaining clients, denoted by $i \in (\frac{N}{2},N]$. Following the approach outlined in~\cite{li2018federated}, we adopt a uniform variance for all samples, with ${\Sigma}_{k,k} = k^{-1.2}$. Subsequently, we generate a labeling model using the distribution $\mathcal{N}(\bm{\mu}_w,\bm{\Sigma}_w)$. 

Given a data sample $\bx\in\R^d$, the labels are generated as follows: clients $1,...,\frac{N}{2}$ assign labels based on $y = \sign(\bw^\top \bx)$, while clients $\frac{N}{2}+1,...,N$ assign labels based on $y = \sign(-\bw^\top \bx)$. For this specific experiment, we set $\mu_1=0.2$, $\mu_2=-0.2$, and $\mu_w=0.1$. The data dimension is $d=60$, and there are 2 classes in the output. We have a total of 50 clients, each generating $500$ samples following the aforementioned guidelines. We train a logistic regression model on each client's data.

\begin{figure}[t!]
  \centering
  \begin{subfigure}[b]{0.4\textwidth}
    \centering
    \includegraphics[width=\textwidth]{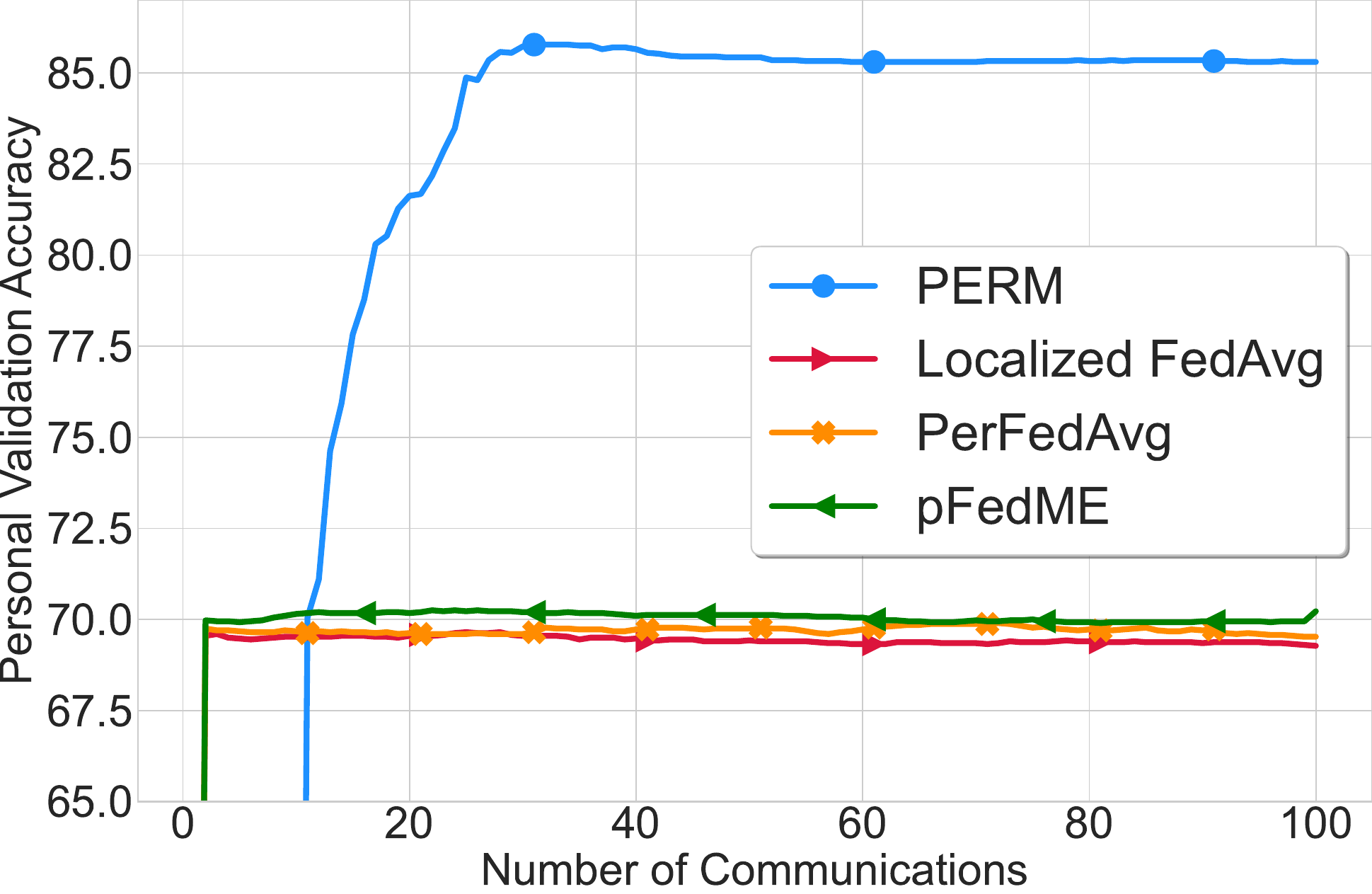}
    \caption{Personalized Accuracy}
    \label{fig:syn-acc}
  \end{subfigure}
  \hspace{20pt}
  \begin{subfigure}[b]{0.4\textwidth}
    \centering
   \includegraphics[width=\textwidth]{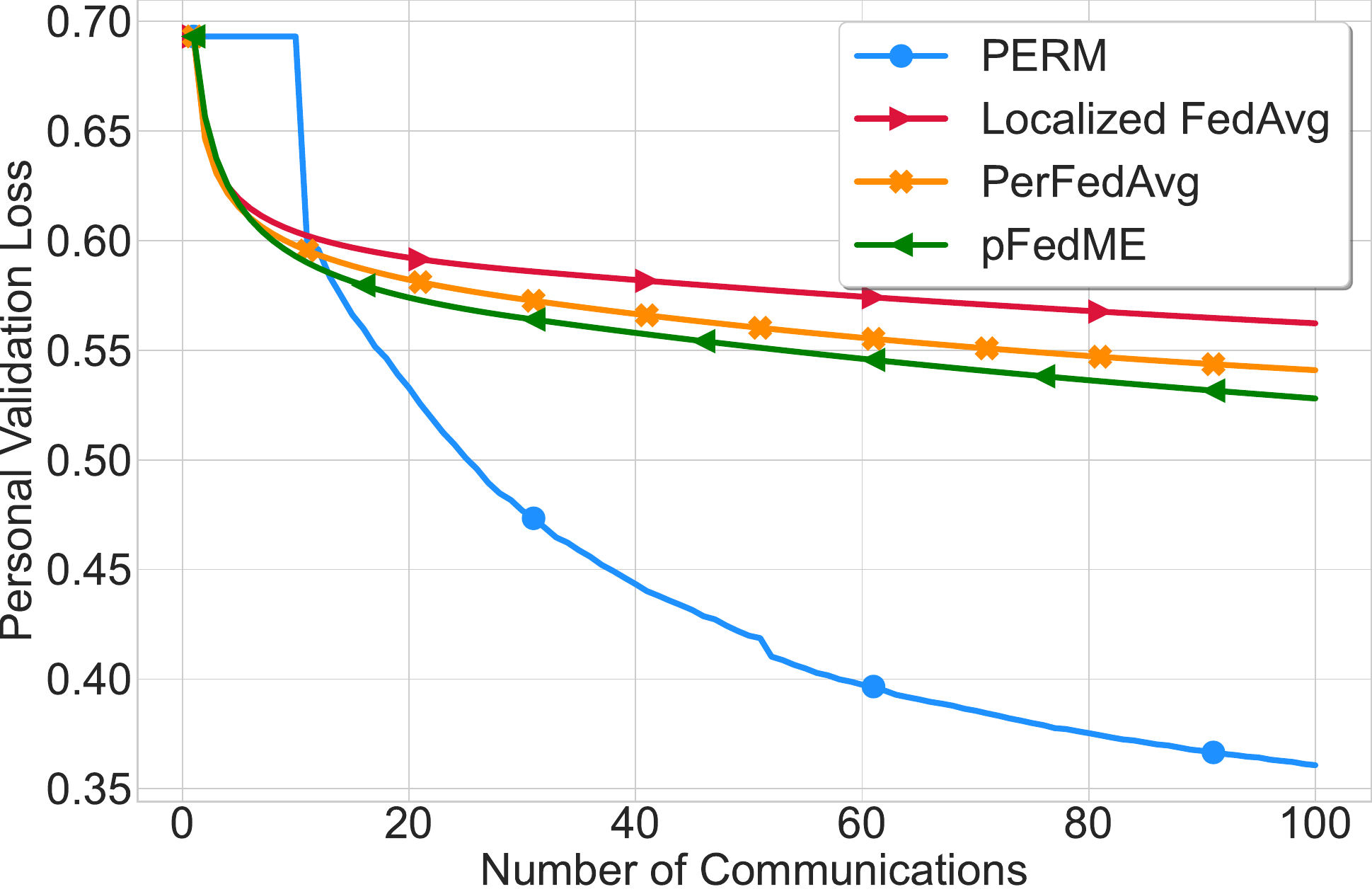}
    \caption{Personalized Loss}
    \label{fig:syn-loss}
  \end{subfigure}\vspace{-1mm}
  \caption{Comparative analysis of personalization methods, including our single-loop PERM algorithm, localized FedAvg, perFedAvg, and pFedME, with synthetic data. The disparity in personalized accuracy and loss highlights PERM's capability to leverage relevant client correlations.}
  \label{fig:synthetic}\vspace{-5mm}
\end{figure}

To demonstrate the superiority of our PERM algorithm, we conducted a performance comparison against other prominent personalized approaches, including the fined-tuned model of FedAvg~\cite{mcmahan2017communication} (referred to as localized FedAvg), perFedAg~\cite{fallah2020personalized}, and pFedME~\cite{t2020personalized}. The results in Figure~\ref{fig:synthetic} highlight PERM's efficient learning of personalized models for individual clients. In contrast, competing methods relying on globally trained models struggle to match PERM's effectiveness in highly heterogeneous scenarios, as seen in personalized accuracy and loss. This showcases PERM's exceptional ability to leverage relevant client learning. 


\begin{figure}[t!]
  \centering 
  \begin{subfigure}[b]{0.40\textwidth}
    \centering 
    \includegraphics[width=\textwidth]{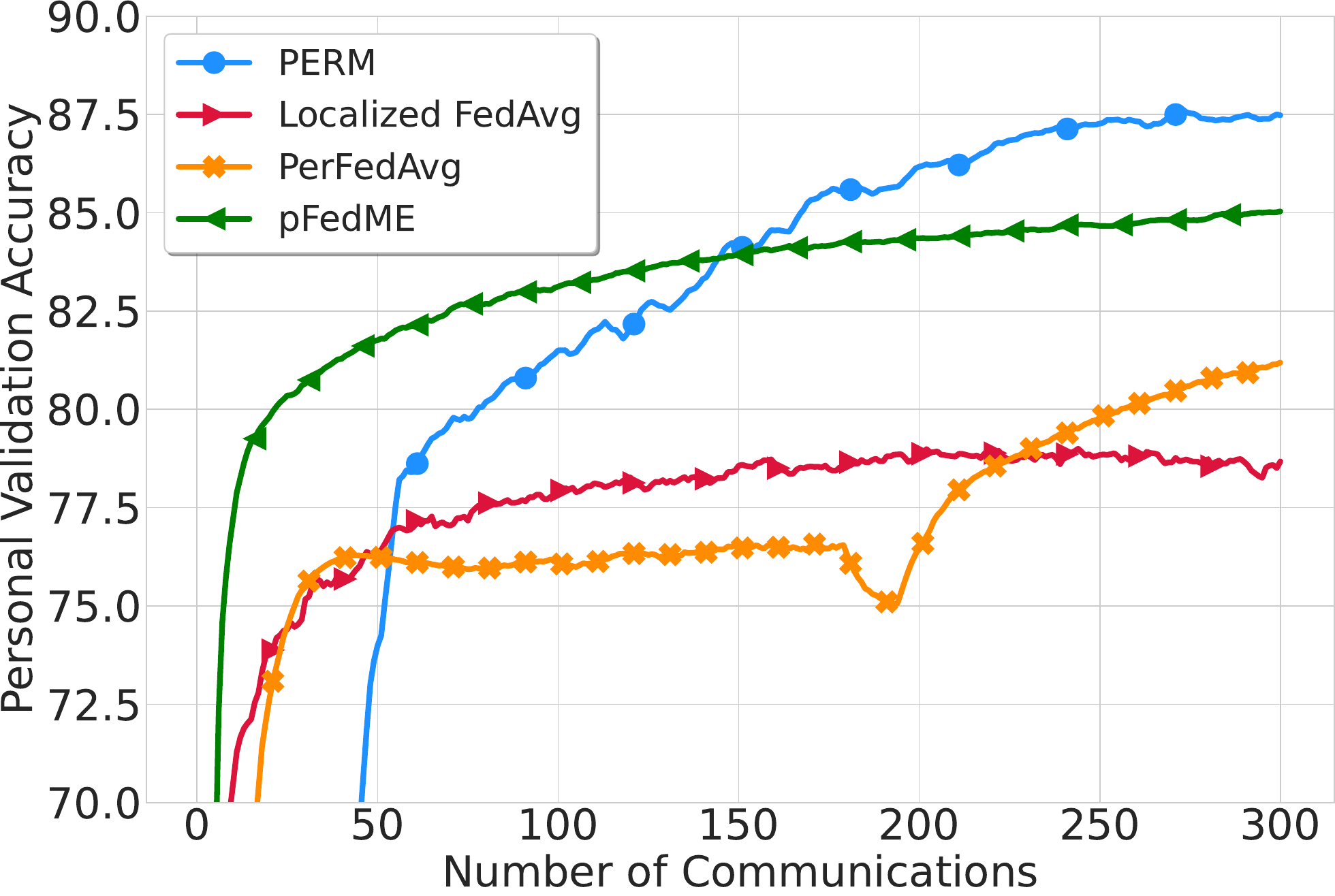}
    \caption{Personalized Accuracy}
    \label{fig:cifar10_acc}
  \end{subfigure}
  \hspace{20pt}
  \begin{subfigure}[b]{0.40\textwidth}
    \centering
    \includegraphics[width=\textwidth]{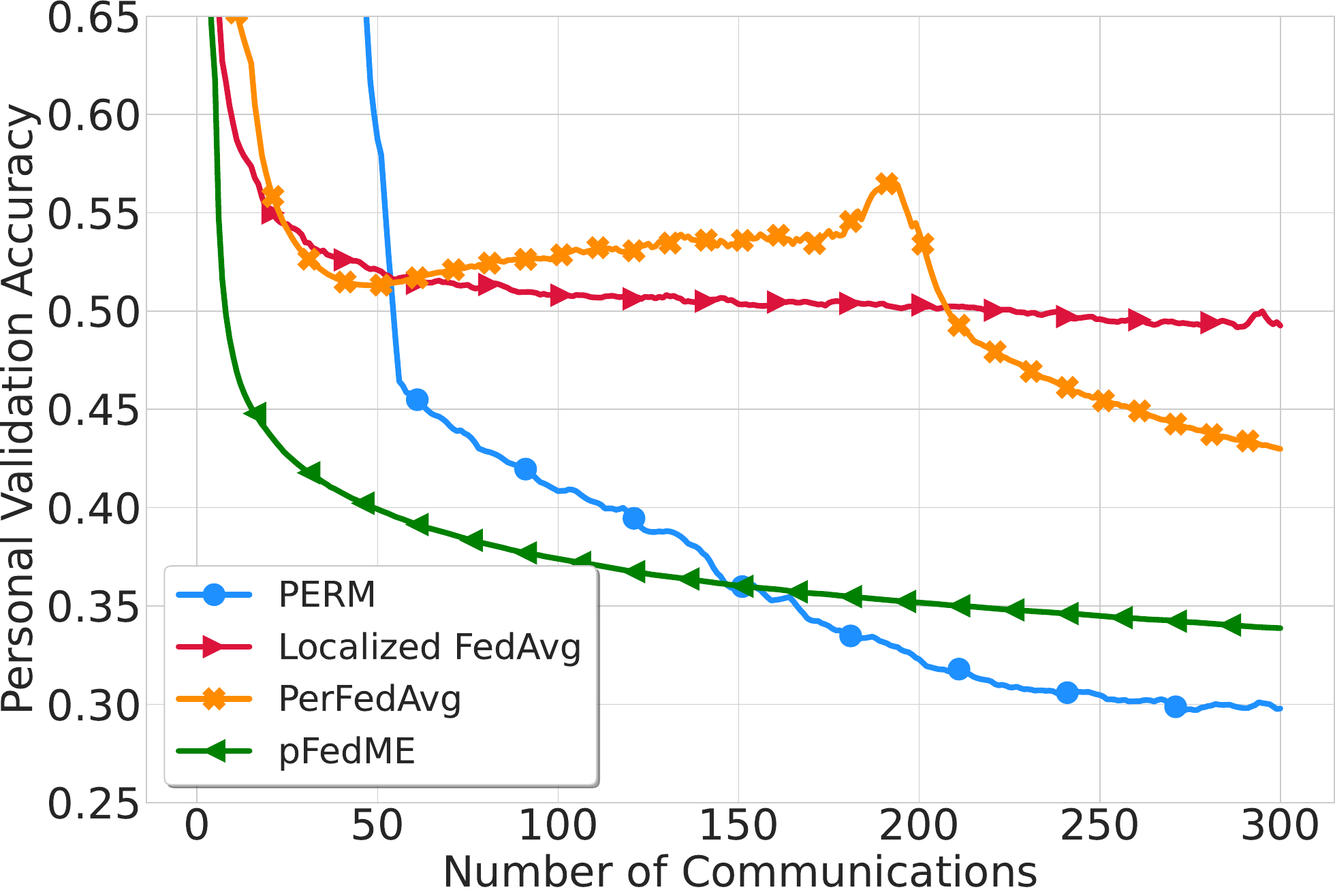}
    \caption{Personalized Loss}
    \label{fig:cifar10_loss}
  \end{subfigure}\vspace{-1mm}
  \caption{Comparative analysis  of our single-loop PERM algorithm, localized FedAvg, PFedMe, and perFedAg, on CIFAR10 dataset and a 2-layer CNN model. Each client has access to only 2 classes of data. PERM rapidly catches up after 10 rounds of warmup without personalization involved.}
  \label{fig:cifar10}\vspace{-4mm}
\end{figure}

\begin{wrapfigure}{r}{0.5\textwidth}
    \centering\vspace{-2mm}
    \includegraphics[width=0.5\textwidth]{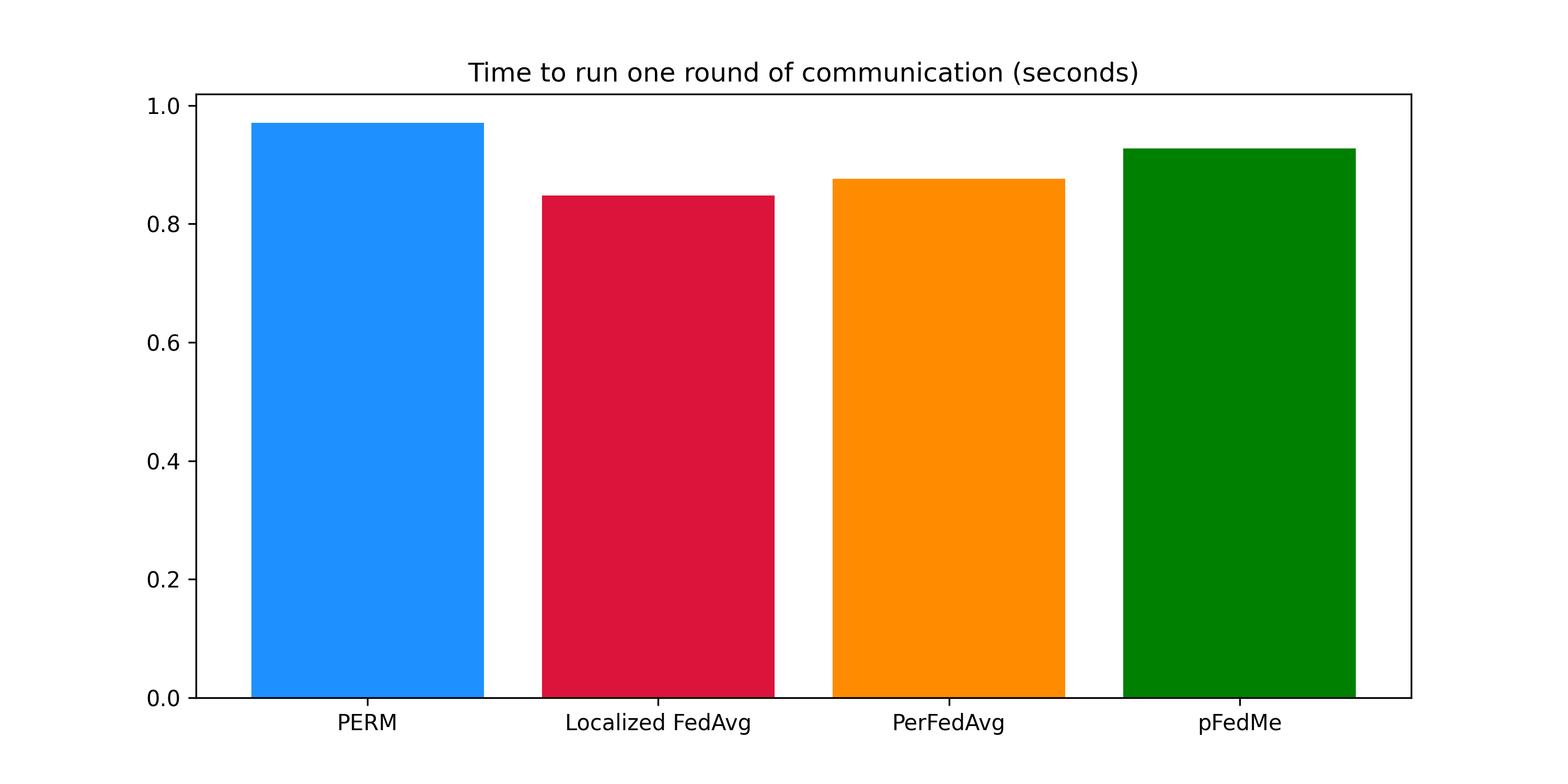}\vspace{-2mm}
    \caption{Runtime of different  algorithms in a limited environment. We compare  PERM (single loop), PerFedAvg, FedAvg, and pFedMe. PERM has a minimal overhead over FedAvg and is comparable to other personalization methods.}\vspace{-3mm}
    \label{fig:time}
\end{wrapfigure}

\noindent\textbf{Experiment on CIFAR10 dataset.}~We extend our experimentation to the CIFAR10 dataset using a 2-layer convolutional neural network. During this test, 50 clients participate, each limited to data from just 2 classes, resulting in a pronounced heterogeneous data distribution. We benchmark our algorithm against  PerFedAvg, PFedMe, and the localized FedAvg. As illustrated in Figure~\ref{fig:cifar10},  PERM  demonstrates superior convergence performance compared to other personalized strategies. It's noteworthy that PERM's initial personalized validation is significantly lower than that of approaches like PerFedAvg and PFedMe. This discrepancy stems from our choice to implement 10 communication rounds as a warm-up phase before initiating personalization, whereas other models embark on personalization right from the outset.

\noindent\textbf{Computational overhead.} In demonstrating the computational efficiency of the proposed PERM algorithm, we present a comparison of wall-clock time of completing one round of communication of PERM and other methods. Each method undertakes 20 local steps along with their distinct computations for personalization. As depicted in Figure~\ref{fig:time}, the PERM (single loop) algorithm's runtime is compared against personalization methods such as PerFedAvg, FedAvg, and pFedMe. Remarkably, PERM maintains a notably minimal computational overhead. The run-time is slightly worse due to overhead of estimating mixing parameters.

\vspace{-2mm}\section{Discussion \& Conclusion}\vspace{-2mm}
This paper introduces a new \textit{data\&system-aware} paradigm for learning from multiple heterogeneous data sources to achieve optimal statistical accuracy across all data distributions without imposing stringent constraints on computational resources shared by participating devices.  The proposed PERM schema, though simple, provides an efficient solution to enable each client to learn a personalized model by \textit{learning who to learn with} via personalizing the aggregation of data sources through an efficient empirical statistical discrepancy estimation module.  To efficiently solve all aggregated personalized losses, we propose a model shuffling idea to optimize all losses in parallel.  PERM can also be employed in other learning settings with multiple sources of data such as domain adaptation and multi-task learning to entail optimal statistical accuracy.   

We would like to embark on the scalability of PERM.  The compute burden on clients and servers is roughly the same as existing methods thanks to shuffling (except for extra overhead due to estimating mixing parameters which is the same as running FedAvg in a two-stage approach and an extra communication in an interleaved approach). The only hurdle would be the required \textit{memory at server} to maintain mixing parameters, which scales proportionally to the square of the number of clients,
 which can be alleviated by clustering devices which we leave as a future work. 


 \clearpage

 \section*{Acknowledgement}

This work was partially supported by NSF CAREER Award \#2239374 NSF CNS   Award \#1956276. 
\bibliographystyle{unsrt}  
\bibliography{references.bib}  

\begin{thebibliography}{10}

\bibitem{kairouz2019advances}
Peter Kairouz, H.~Brendan McMahan, Brendan Avent, Aurélien Bellet, Mehdi
  Bennis, Arjun~Nitin Bhagoji, Keith Bonawitz, Zachary Charles, Graham Cormode,
  Rachel Cummings, Rafael G.~L. D'Oliveira, Salim~El Rouayheb, David Evans,
  Josh Gardner, Zachary Garrett, Adrià Gascón, Badih Ghazi, Phillip~B.
  Gibbons, Marco Gruteser, Zaid Harchaoui, Chaoyang He, Lie He, Zhouyuan Huo,
  Ben Hutchinson, Justin Hsu, Martin Jaggi, Tara Javidi, Gauri Joshi, Mikhail
  Khodak, Jakub Konečný, Aleksandra Korolova, Farinaz Koushanfar, Sanmi
  Koyejo, Tancrède Lepoint, Yang Liu, Prateek Mittal, Mehryar Mohri, Richard
  Nock, Ayfer Özgür, Rasmus Pagh, Mariana Raykova, Hang Qi, Daniel Ramage,
  Ramesh Raskar, Dawn Song, Weikang Song, Sebastian~U. Stich, Ziteng Sun,
  Ananda~Theertha Suresh, Florian Tramèr, Praneeth Vepakomma, Jianyu Wang,
  Li~Xiong, Zheng Xu, Qiang Yang, Felix~X. Yu, Han Yu, and Sen Zhao.
\newblock Advances and open problems in federated learning.
\newblock {\em Foundations and Trends{\textregistered} in Machine Learning},
  2021.

\bibitem{mansour2020three}
Yishay Mansour, Mehryar Mohri, Jae Ro, and Ananda~Theertha Suresh.
\newblock Three approaches for personalization with applications to federated
  learning.
\newblock {\em arXiv preprint arXiv:2002.10619}, 2020.

\bibitem{li2021federated}
Chengxi Li, Gang Li, and Pramod~K Varshney.
\newblock Federated learning with soft clustering.
\newblock {\em IEEE Internet of Things Journal}, 9(10):7773--7782, 2021.

\bibitem{ghosh2020efficient}
Avishek Ghosh, Jichan Chung, Dong Yin, and Kannan Ramchandran.
\newblock An efficient framework for clustered federated learning.
\newblock {\em Advances in Neural Information Processing Systems},
  33:19586--19597, 2020.

\bibitem{ma2022convergence}
Jie Ma, Guodong Long, Tianyi Zhou, Jing Jiang, and Chengqi Zhang.
\newblock On the convergence of clustered federated learning.
\newblock {\em arXiv preprint arXiv:2202.06187}, 2022.

\bibitem{eichner2019semi}
Hubert Eichner, Tomer Koren, Brendan Mcmahan, Nathan Srebro, and Kunal Talwar.
\newblock Semi-cyclic stochastic gradient descent.
\newblock In {\em Proceedings of the 36th International Conference on Machine
  Learning, PMLR}, volume~97, 2019.

\bibitem{t2020personalized}
Canh T~Dinh, Nguyen Tran, and Tuan~Dung Nguyen.
\newblock Personalized federated learning with moreau envelopes.
\newblock {\em Advances in Neural Information Processing Systems}, 33, 2020.

\bibitem{huang2020personalized}
Yutao Huang, Lingyang Chu, Zirui Zhou, Lanjun Wang, Jiangchuan Liu, Jian Pei,
  and Yong Zhang.
\newblock Personalized federated learning: An attentive collaboration approach.
\newblock {\em arXiv preprint arXiv:2007.03797}, 2020.

\bibitem{fallah2020personalized}
Alireza Fallah, Aryan Mokhtari, and Asuman Ozdaglar.
\newblock Personalized federated learning: A meta-learning approach.
\newblock {\em Advances in Neural Information Processing Systems}, 2020.

\bibitem{smith2017federated}
Virginia Smith, Chao-Kai Chiang, Maziar Sanjabi, and Ameet~S Talwalkar.
\newblock Federated multi-task learning.
\newblock In {\em Advances in Neural Information Processing Systems}, pages
  4424--4434, 2017.

\bibitem{deng2020adaptive}
Yuyang Deng, Mohammad~Mahdi Kamani, and Mehrdad Mahdavi.
\newblock Adaptive personalized federated learning.
\newblock {\em arXiv preprint arXiv:2003.13461}, 2020.

\bibitem{hanzely2021personalized}
Filip Hanzely, Boxin Zhao, and Mladen Kolar.
\newblock Personalized federated learning: A unified framework and universal
  optimization techniques.
\newblock {\em arXiv preprint arXiv:2102.09743}, 2021.

\bibitem{lancaster2000incidental}
Tony Lancaster.
\newblock The incidental parameter problem since 1948.
\newblock {\em Journal of econometrics}, 95(2):391--413, 2000.

\bibitem{mcmahan2017communication}
Brendan McMahan, Eider Moore, Daniel Ramage, Seth Hampson, and Blaise~Aguera
  y~Arcas.
\newblock Communication-efficient learning of deep networks from decentralized
  data.
\newblock In {\em Artificial intelligence and statistics}, pages 1273--1282.
  PMLR, 2017.

\bibitem{karimireddy2020scaffold}
Sai~Praneeth Karimireddy, Satyen Kale, Mehryar Mohri, Sashank Reddi, Sebastian
  Stich, and Ananda~Theertha Suresh.
\newblock Scaffold: Stochastic controlled averaging for federated learning.
\newblock In {\em International Conference on Machine Learning}, pages
  5132--5143. PMLR, 2020.

\bibitem{hamer2020fedboost}
Jenny Hamer, Mehryar Mohri, and Ananda~Theertha Suresh.
\newblock Fedboost: A communication-efficient algorithm for federated learning.
\newblock In {\em International Conference on Machine Learning}, pages
  3973--3983. PMLR, 2020.

\bibitem{haddadpour2021federated}
Farzin Haddadpour, Mohammad~Mahdi Kamani, Aryan Mokhtari, and Mehrdad Mahdavi.
\newblock Federated learning with compression: Unified analysis and sharp
  guarantees.
\newblock In {\em International Conference on Artificial Intelligence and
  Statistics}, pages 2350--2358. PMLR, 2021.

\bibitem{sunfedspeed}
Yan Sun, Li~Shen, Tiansheng Huang, Liang Ding, and Dacheng Tao.
\newblock Fedspeed: Larger local interval, less communication round, and higher
  generalization accuracy.
\newblock In {\em The Eleventh International Conference on Learning
  Representations}, 2023.

\bibitem{yang2022anarchic}
Haibo Yang, Xin Zhang, Prashant Khanduri, and Jia Liu.
\newblock Anarchic federated learning.
\newblock In {\em International Conference on Machine Learning}, pages
  25331--25363. PMLR, 2022.

\bibitem{wang2022unified}
Shiqiang Wang and Mingyue Ji.
\newblock A unified analysis of federated learning with arbitrary client
  participation.
\newblock {\em Advances in Neural Information Processing Systems},
  35:19124--19137, 2022.

\bibitem{wu2022communication}
Chuhan Wu, Fangzhao Wu, Lingjuan Lyu, Yongfeng Huang, and Xing Xie.
\newblock Communication-efficient federated learning via knowledge
  distillation.
\newblock {\em Nature communications}, 13(1):1--8, 2022.

\bibitem{he2020group}
Chaoyang He, Murali Annavaram, and Salman Avestimehr.
\newblock Group knowledge transfer: Federated learning of large cnns at the
  edge.
\newblock {\em Advances in Neural Information Processing Systems},
  33:14068--14080, 2020.

\bibitem{lin2020ensemble}
Tao Lin, Lingjing Kong, Sebastian~U Stich, and Martin Jaggi.
\newblock Ensemble distillation for robust model fusion in federated learning.
\newblock {\em Advances in Neural Information Processing Systems},
  33:2351--2363, 2020.

\bibitem{itahara2021distillation}
Sohei Itahara, Takayuki Nishio, Yusuke Koda, Masahiro Morikura, and Koji
  Yamamoto.
\newblock Distillation-based semi-supervised federated learning for
  communication-efficient collaborative training with non-iid private data.
\newblock {\em IEEE Transactions on Mobile Computing}, 22(1):191--205, 2021.

\bibitem{diao2020heterofl}
Enmao Diao, Jie Ding, and Vahid Tarokh.
\newblock Heterofl: Computation and communication efficient federated learning
  for heterogeneous clients.
\newblock {\em arXiv preprint arXiv:2010.01264}, 2020.

\bibitem{horvath2021fjord}
Samuel Horvath, Stefanos Laskaridis, Mario Almeida, Ilias Leontiadis, Stylianos
  Venieris, and Nicholas Lane.
\newblock Fjord: Fair and accurate federated learning under heterogeneous
  targets with ordered dropout.
\newblock {\em Advances in Neural Information Processing Systems},
  34:12876--12889, 2021.

\bibitem{caldas2018expanding}
Sebastian Caldas, Jakub Kone{\v{c}}ny, H~Brendan McMahan, and Ameet Talwalkar.
\newblock Expanding the reach of federated learning by reducing client resource
  requirements.
\newblock {\em arXiv preprint arXiv:1812.07210}, 2018.

\bibitem{alam2022fedrolex}
Samiul Alam, Luyang Liu, Ming Yan, and Mi~Zhang.
\newblock Fedrolex: Model-heterogeneous federated learning with rolling
  sub-model extraction.
\newblock {\em Advances in Neural Information Processing Systems},
  35:29677--29690, 2022.

\bibitem{ben2010theory}
Shai Ben-David, John Blitzer, Koby Crammer, Alex Kulesza, Fernando Pereira, and
  Jennifer~Wortman Vaughan.
\newblock A theory of learning from different domains.
\newblock {\em Machine learning}, 79:151--175, 2010.

\bibitem{mansour2014robust}
Yishay Mansour and Mariano Schain.
\newblock Robust domain adaptation.
\newblock {\em Annals of Mathematics and Artificial Intelligence},
  71(4):365--380, 2014.

\bibitem{konstantinov2019robust}
Nikola Konstantinov and Christoph Lampert.
\newblock Robust learning from untrusted sources.
\newblock In {\em International conference on machine learning}, pages
  3488--3498. PMLR, 2019.

\bibitem{crammer2008learning}
Koby Crammer, Michael Kearns, and Jennifer Wortman.
\newblock Learning from multiple sources.
\newblock {\em Journal of Machine Learning Research}, 9(8), 2008.

\bibitem{even2022sample}
Mathieu Even, Laurent Massouli{\'e}, and Kevin Scaman.
\newblock On sample optimality in personalized collaborative and federated
  learning.
\newblock In {\em NeurIPS 2022-36th Conference on Neural Information Processing
  System}, 2022.

\bibitem{konevcny2016federated}
Jakub Kone{\v{c}}n{\`y}, H~Brendan McMahan, Felix~X Yu, Peter Richt{\'a}rik,
  Ananda~Theertha Suresh, and Dave Bacon.
\newblock Federated learning: Strategies for improving communication
  efficiency.
\newblock {\em arXiv preprint arXiv:1610.05492}, 2016.

\bibitem{mohri2019agnostic}
Mehryar Mohri, Gary Sivek, and Ananda~Theertha Suresh.
\newblock Agnostic federated learning.
\newblock In {\em International Conference on Machine Learning}, pages
  4615--4625. PMLR, 2019.

\bibitem{deng2020distributionally}
Yuyang Deng, Mohammad~Mahdi Kamani, and Mehrdad Mahdavi.
\newblock Distributionally robust federated averaging.
\newblock {\em Advances in neural information processing systems},
  33:15111--15122, 2020.

\bibitem{li2019feddane}
Tian Li, Anit~Kumar Sahu, Manzil Zaheer, Maziar Sanjabi, Ameet Talwalkar, and
  Virginia Smithy.
\newblock Feddane: A federated newton-type method.
\newblock In {\em 2019 53rd Asilomar Conference on Signals, Systems, and
  Computers}, pages 1227--1231. IEEE, 2019.

\bibitem{karimireddy2019scaffold}
Sai~Praneeth Karimireddy, Satyen Kale, Mehryar Mohri, Sashank~J Reddi,
  Sebastian~U Stich, and Ananda~Theertha Suresh.
\newblock Scaffold: Stochastic controlled averaging for on-device federated
  learning.
\newblock {\em International Conference on Machine Learning}, 119:5132--5143,
  2020.

\bibitem{haddadpour2019convergence}
Farzin Haddadpour and Mehrdad Mahdavi.
\newblock On the convergence of local descent methods in federated learning.
\newblock {\em arXiv preprint arXiv:1910.14425}, 2019.

\bibitem{yu2020salvaging}
Tao Yu, Eugene Bagdasaryan, and Vitaly Shmatikov.
\newblock Salvaging federated learning by local adaptation.
\newblock {\em arXiv preprint arXiv:2002.04758}, 2020.

\bibitem{sriperumbudur2009integral}
Bharath~K Sriperumbudur, Kenji Fukumizu, Arthur Gretton, Bernhard
  Sch{\"o}lkopf, and Gert~RG Lanckriet.
\newblock On integral probability metrics,$\backslash$phi-divergences and
  binary classification.
\newblock {\em arXiv preprint arXiv:0901.2698}, 2009.

\bibitem{hanneke2022no}
Steve Hanneke and Samory Kpotufe.
\newblock A no-free-lunch theorem for multitask learning.
\newblock {\em The Annals of Statistics}, 50(6):3119--3143, 2022.

\bibitem{shalev2014understanding}
Shai Shalev-Shwartz and Shai Ben-David.
\newblock {\em Understanding machine learning: From theory to algorithms}.
\newblock Cambridge university press, 2014.

\bibitem{xu2023unified}
Zi~Xu, Huiling Zhang, Yang Xu, and Guanghui Lan.
\newblock A unified single-loop alternating gradient projection algorithm for
  nonconvex--concave and convex--nonconcave minimax problems.
\newblock {\em Mathematical Programming}, pages 1--72, 2023.

\bibitem{dandi2022implicit}
Yatin Dandi, Luis Barba, and Martin Jaggi.
\newblock Implicit gradient alignment in distributed and federated learning.
\newblock In {\em Proceedings of the AAAI Conference on Artificial
  Intelligence}, volume~36, pages 6454--6462, 2022.

\bibitem{zhao2018federated}
Yue Zhao, Meng Li, Liangzhen Lai, Naveen Suda, Damon Civin, and Vikas Chandra.
\newblock Federated learning with non-iid data.
\newblock {\em arXiv preprint arXiv:1806.00582}, 2018.

\bibitem{wang2022unreasonable}
Jianyu Wang, Rudrajit Das, Gauri Joshi, Satyen Kale, Zheng Xu, and Tong Zhang.
\newblock On the unreasonable effectiveness of federated averaging with
  heterogeneous data.
\newblock {\em arXiv preprint arXiv:2206.04723}, 2022.

\bibitem{stich2018local}
Sebastian~U Stich.
\newblock Local sgd converges fast and communicates little.
\newblock In {\em International Conference on Learning Representations}, 2018.

\bibitem{woodworth2018graph}
Blake~E Woodworth, Jialei Wang, Adam Smith, Brendan McMahan, and Nati Srebro.
\newblock Graph oracle models, lower bounds, and gaps for parallel stochastic
  optimization.
\newblock In {\em Advances in neural information processing systems}, pages
  8496--8506, 2018.

\bibitem{woodworth2020minibatch}
Blake~E Woodworth, Kumar~Kshitij Patel, and Nati Srebro.
\newblock Minibatch vs local sgd for heterogeneous distributed learning.
\newblock {\em Advances in Neural Information Processing Systems},
  33:6281--6292, 2020.

\bibitem{mishchenko2022proxskip}
Konstantin Mishchenko, Grigory Malinovsky, Sebastian Stich, and Peter
  Richt{\'a}rik.
\newblock Proxskip: Yes! local gradient steps provably lead to communication
  acceleration! finally!
\newblock In {\em International Conference on Machine Learning}, pages
  15750--15769. PMLR, 2022.

\bibitem{yuan2020federated}
Honglin Yuan and Tengyu Ma.
\newblock Federated accelerated stochastic gradient descent.
\newblock {\em Advances in Neural Information Processing Systems},
  33:5332--5344, 2020.

\bibitem{haddadpour2019local}
Farzin Haddadpour, Mohammad~Mahdi Kamani, Mehrdad Mahdavi, and Viveck Cadambe.
\newblock Local sgd with periodic averaging: Tighter analysis and adaptive
  synchronization.
\newblock In {\em Advances in Neural Information Processing Systems}, pages
  11080--11092, 2019.

\bibitem{gorbunov2021local}
Eduard Gorbunov, Filip Hanzely, and Peter Richt{\'a}rik.
\newblock Local sgd: Unified theory and new efficient methods.
\newblock In {\em International Conference on Artificial Intelligence and
  Statistics}, pages 3556--3564. PMLR, 2021.

\bibitem{ahn2020sgd}
Kwangjun Ahn, Chulhee Yun, and Suvrit Sra.
\newblock Sgd with shuffling: optimal rates without component convexity and
  large epoch requirements.
\newblock {\em Advances in Neural Information Processing Systems},
  33:17526--17535, 2020.

\bibitem{cho2023convergence}
Yae~Jee Cho, Pranay Sharma, Gauri Joshi, Zheng Xu, Satyen Kale, and Tong Zhang.
\newblock On the convergence of federated averaging with cyclic client
  participation.
\newblock {\em arXiv preprint arXiv:2302.03109}, 2023.

\bibitem{mohri2018foundations}
Mehryar Mohri, Afshin Rostamizadeh, and Ameet Talwalkar.
\newblock {\em Foundations of machine learning}.
\newblock MIT press, 2018.

\bibitem{li2018federated}
Tian Li, Anit~Kumar Sahu, Manzil Zaheer, Maziar Sanjabi, Ameet Talwalkar, and
  Virginia Smith.
\newblock Federated optimization in heterogeneous networks.
\newblock {\em arXiv preprint arXiv:1812.06127}, 2018.

\bibitem{caldas2018leaf}
Sebastian Caldas, Peter Wu, Tian Li, Jakub Kone{\v{c}}n{\`y}, H~Brendan
  McMahan, Virginia Smith, and Ameet Talwalkar.
\newblock Leaf: A benchmark for federated settings.
\newblock {\em arXiv preprint arXiv:1812.01097}, 2018.

\bibitem{lin2019gradient}
Tianyi Lin, Chi Jin, and Michael~I Jordan.
\newblock On gradient descent ascent for nonconvex-concave minimax problems.
\newblock {\em arXiv preprint arXiv:1906.00331}, 2019.

\bibitem{schneider2016probability}
Markus Schneider.
\newblock Probability inequalities for kernel embeddings in sampling without
  replacement.
\newblock In {\em Artificial Intelligence and Statistics}, pages 66--74. PMLR,
  2016.

\end{thebibliography}

\clearpage


\appendix

\clearpage
\appendix

\section{Additional Experiments}
In addition to experiments on synthetic and CIFAR10 datasets reported before, we have also conducted experiments on the EMNIST dataset, highlighting PERM's capability to derive superior personalized models by tapping into inter-client data similarities. Additionally, further insights emerged from our tests on the MNIST dataset, revealing how PERM's learned mixture weights adeptly respond to both homogeneous and highly heterogeneous data scenarios.

\paragraph{Experiment on EMNIST dataset} In addition to the synthetic and CIFAR10 datasets discussed in the main body, we run experiments on the EMNIST dataset~\cite{caldas2018leaf}, which is naturally distributed in a federated setting. In this case, we chose 50 clients and use a 2-layer MLP model, each with 200 neurons. We compare the PERM algorithm with the localized model in FedAvg and perFedAvg~\cite{fallah2020personalized}. As it can be seen in Figure~\ref{fig:emnist}, PERM can learn a better personalized model by attending to each client's data according to the similarity of the data distribution between clients. The learned values of $\balpha$, in Figure~\ref{app:fig:alpha}, show that the clients are learning from each others' data, and not focused on their own data only. This signifies that the distribution of data among clients in this dataset is not highly heterogeneous. Note that, since we are using a subset of clients in the EMNIST dataset for the training (only 50 clients for 100 rounds of communication), the results would be sub-optimal. Nonetheless, the experiments are designed to show the effectiveness of different algorithms. As it can be concluded,  in terms of performance, PERM consistently excels beyond its peers, demonstrating exemplary results on various benchmark datasets.

\begin{figure}[h!]
  \centering
  \begin{subfigure}[b]{0.4\textwidth}
    \centering
    \includegraphics[width=\textwidth]{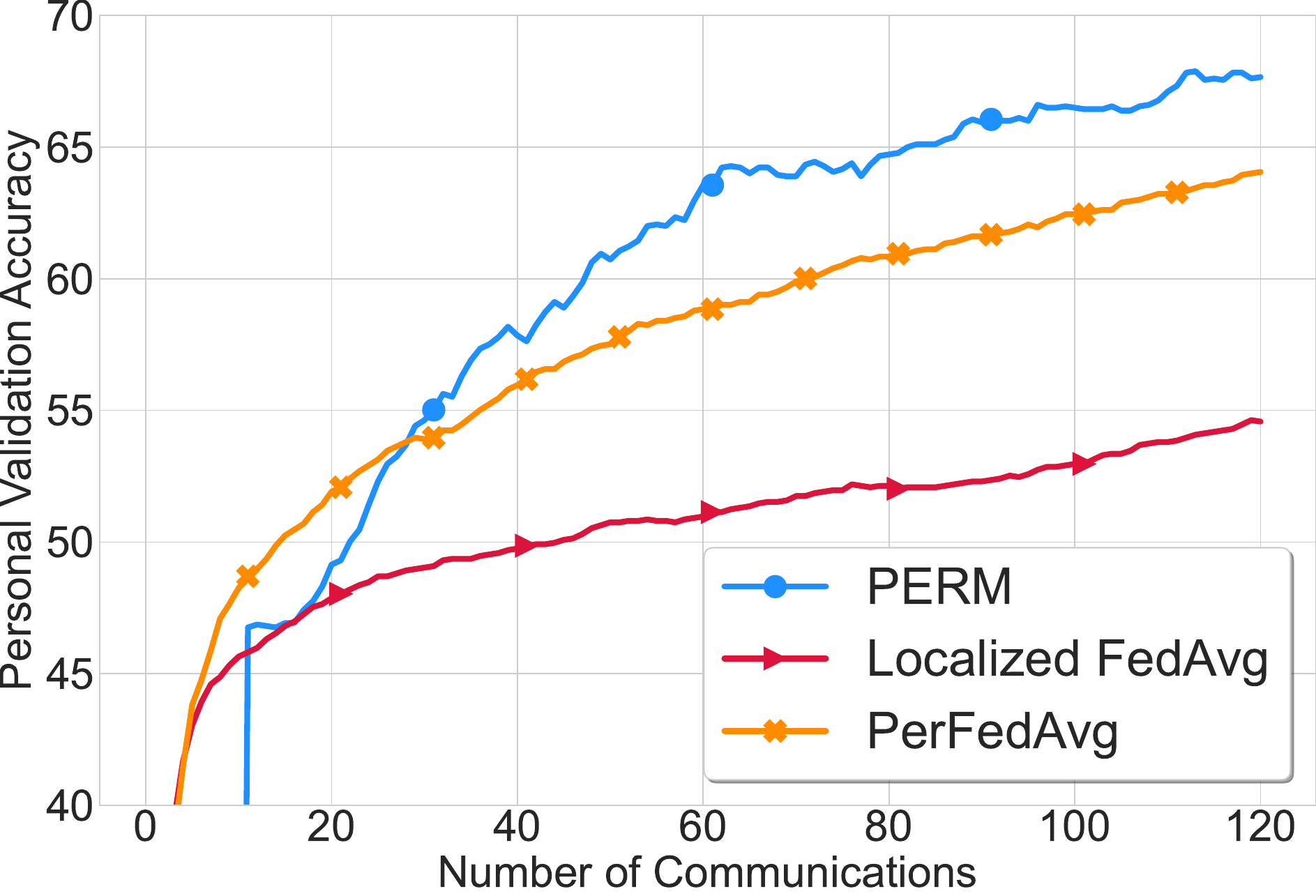}
    \caption{Personalized Accuracy}
    \label{fig:emnist-acc}
  \end{subfigure}
  \hspace{20pt}
  \begin{subfigure}[b]{0.4\textwidth}
    \centering
    \includegraphics[width=\textwidth]{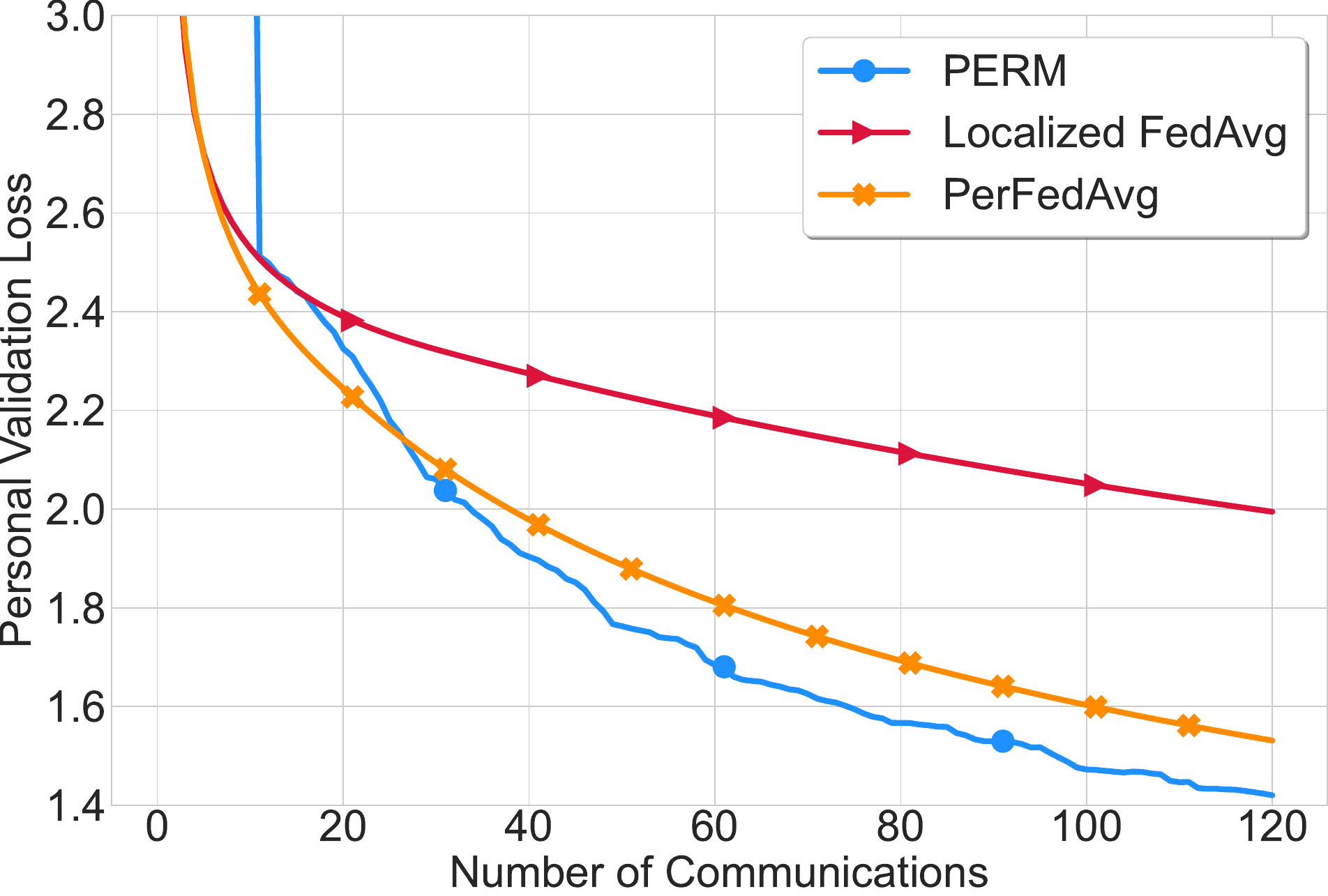}
    \caption{Personalized Loss}
    \label{fig:emnist-loss}
  \end{subfigure}
  \caption{Comparative Analysis of Personalization methods, including our single-loop PERM algorithm, localized FedAvg, and perFedAg, with EMNIST dataset. The disparity in personalized accuracy and loss highlights PERM's capability in leveraging relevant client correlations.}
  \label{fig:emnist}\vspace{-2mm}
\end{figure}
\begin{figure}[h!]
  \centering
  \begin{minipage}{0.5\textwidth}
    \includegraphics[width=0.75\textwidth]{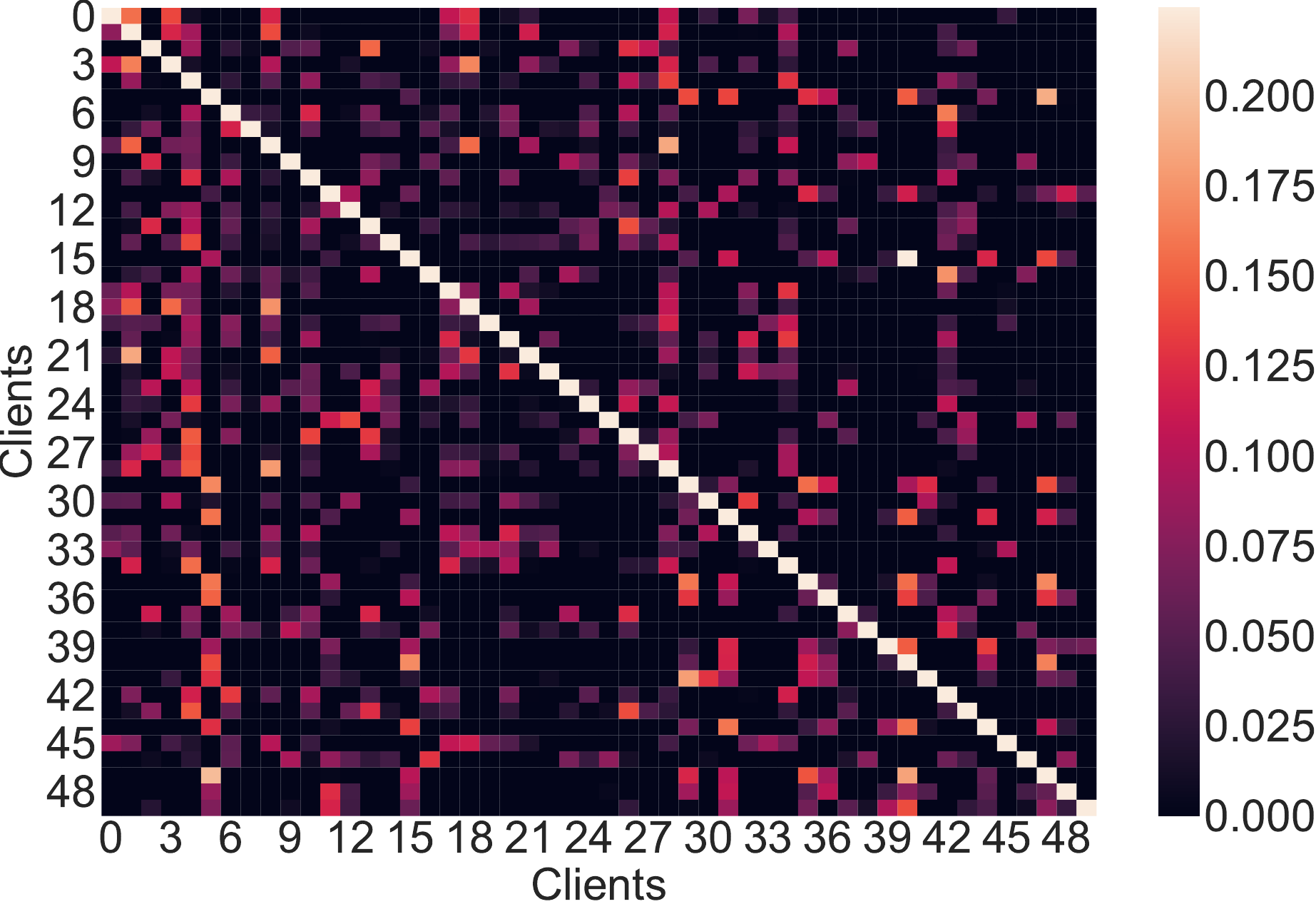}
  \end{minipage}%
  \begin{minipage}{0.4\textwidth}
    \captionof{figure}{The heat map of the learned $\bm{\alpha}$ values for the PERM algorithms on the EMNIST dataset with a 2-layer MLP model. The weights signify that clients mutually benefiting from one another's data, which also highlight that the distribution of data is not significantly heterogeneous in this dataset.}
    \label{app:fig:alpha}
  \end{minipage} \vspace{-5mm}
\end{figure}

\paragraph{The effectiveness of learned mixture weights}To show the effectiveness of the two-stage PERM algorithm, as well as the effects of heterogeneity on the distribution of data among clients on the learned weights $\balpha$ in the algorithm, we run this algorithm on MNIST dataset. We use 50 clients, and the model is an MLP, similar to the EMNIST experiment. In this case, we consider two cases: distributing the data randomly across clients (homogeneous) and only allocating 1 class per client (highly heterogeneous). As it can be seen from Figure~\ref{fig:mnist}, when the data distribution is homogeneous the learned values of $\balpha$ as diffused across clients. However, when the data is highly heterogeneous, the learned $\balpha$ values will be highly sparse, indicating that each client is mostly learning from its own data and some other clients with partial distribution similarity. Notably, the matrix predominantly exhibits sparsity, indicating that each client selectively leverages information solely from a subset of other clients. This discernible pattern reinforces the inherent confidence that each client is effectively learning from a limited but strategically chosen group of clients.

\begin{figure}[h!]
  \centering
  \begin{subfigure}[b]{0.4\textwidth}
    \centering
    \includegraphics[width=\textwidth]{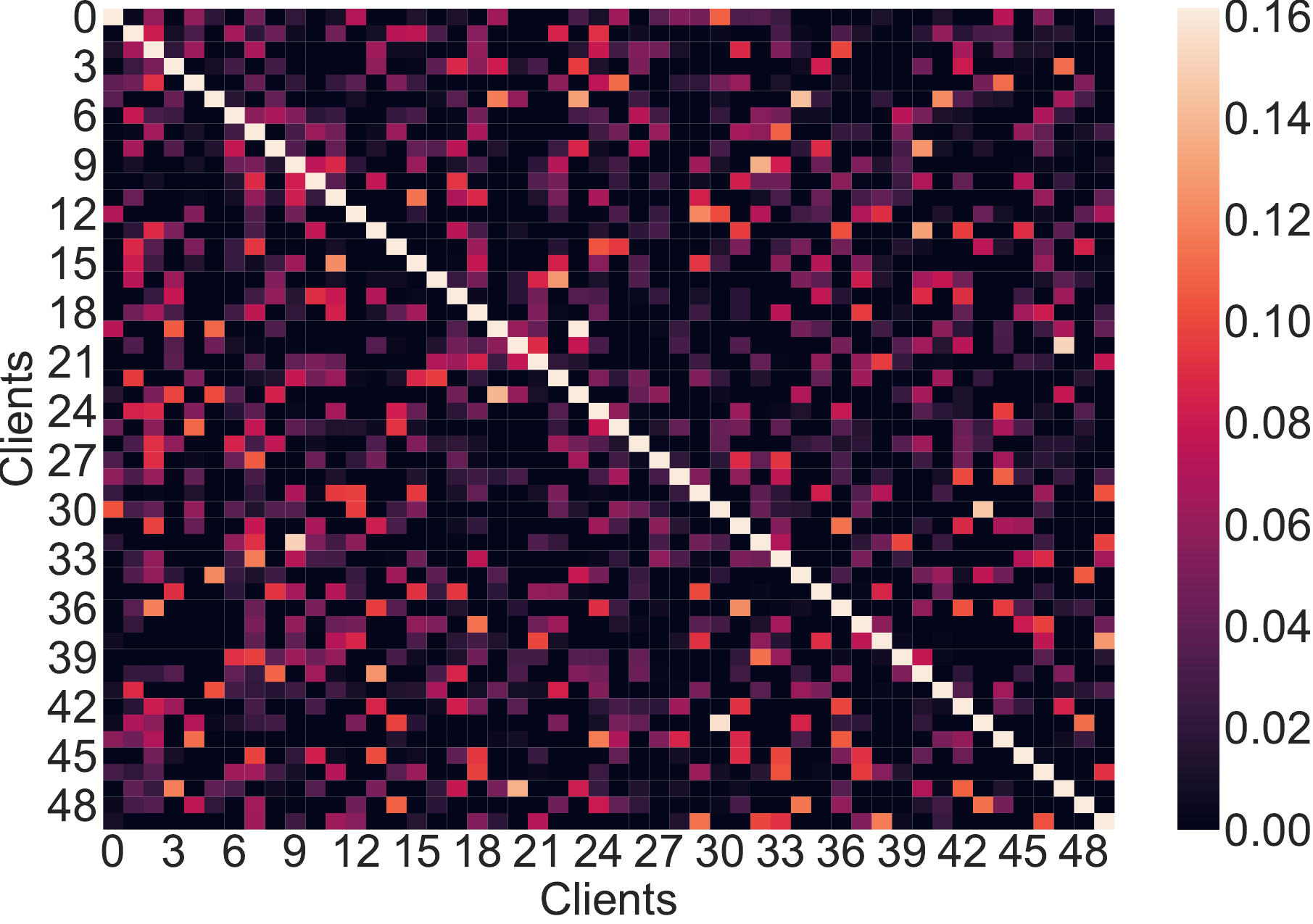}
    \caption{Homogeneous distribution}
    \label{fig:mnist_iid}
  \end{subfigure}
  \hspace{20pt}
  \begin{subfigure}[b]{0.4\textwidth}
    \centering
    \includegraphics[width=\textwidth]{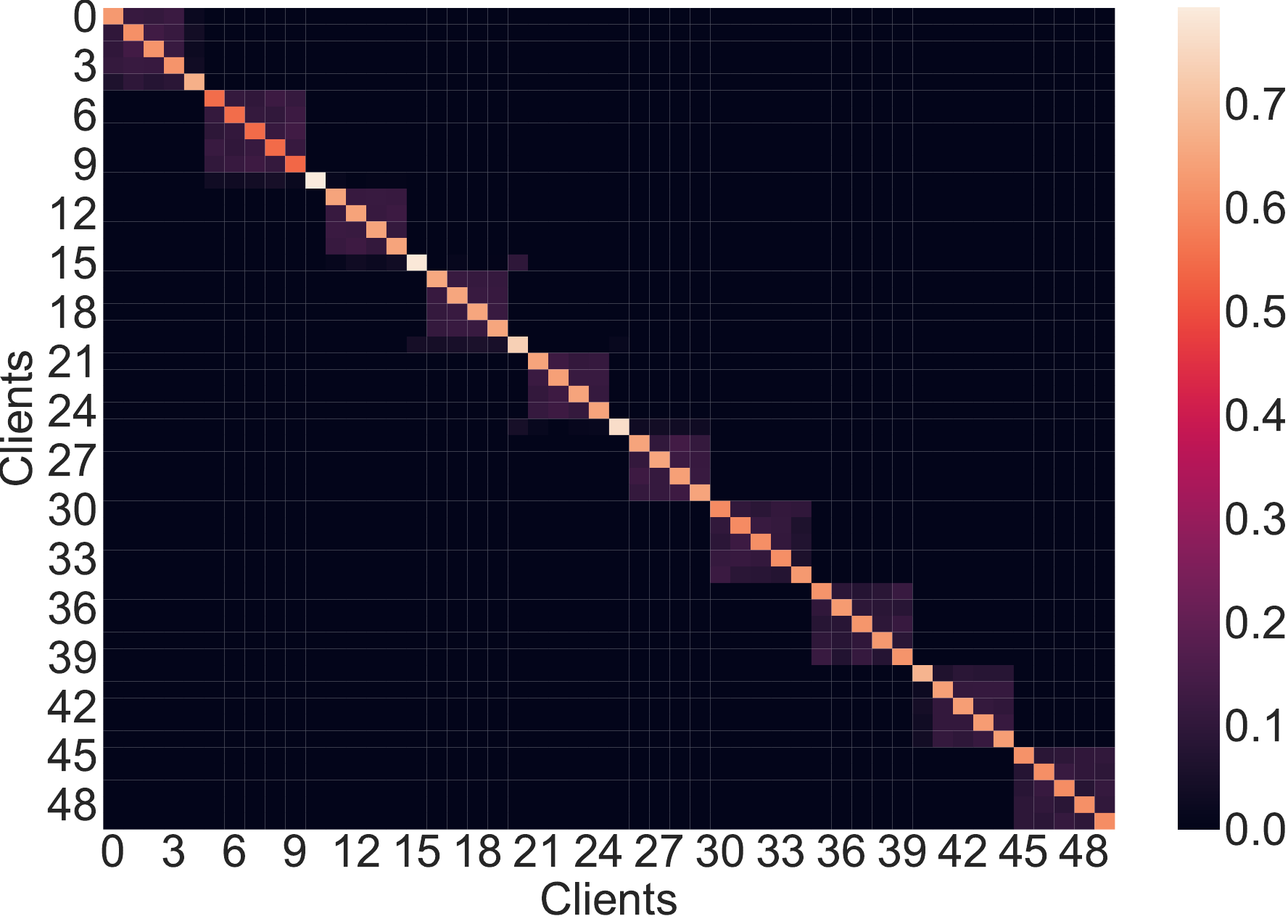}
    \caption{Highly heterogeneous distribution}
    \label{fig:mnist_noniid}
  \end{subfigure}
  \caption{Comparing the performance of two-stage PERM algorithm in learning $\alpha$ values on heterogeneous and homogeneous data distributions. We use MNIST dataset across 50 clients with homogeneous and heterogeneous distributions.}
  \label{fig:mnist}\vspace{-2mm}
\end{figure}

\section{Proof of Two Stages Algorithm}\label{sec:app:two-stages}

In this section we provide the proof of convergence of two-stage implementation of PERM (computing mixing parameters followed by learning personalized models via model shuffling using permutation-based variant of distributed SGD with periodic communication).






\subsection{Technical Lemmas}
\begin{lemma} \label{lem: lipschitz}
    Define $\bv^*(\balpha) := \arg\min_{\bv\in\cW} \Phi(\balpha, \bv) $, and assume $\Phi(\balpha,\cdot)$ is $\mu$-strongly convex and $\nabla_{\bv} \Phi(\balpha,\bv)$ is $L$ Lipschitz in $\balpha$. Then, $\bv^*(\cdot)$ is $\kappa$-Lipschitz where $\kappa = L/\mu$.
    \begin{proof}
        The proof is similar to Lin et al's result on minimax objective~\cite{lin2019gradient}. First, according to optimality conditions we have:
        \begin{align*}
            \langle \bv - \bv^*(\balpha), \nabla_{2} \Phi(\balpha, \bv^*(\balpha)) \rangle \geq 0,\\
            \langle \bv - \bv^*(\balpha'), \nabla_{2} \Phi(\balpha', \bv^*(\balpha')) \rangle \geq 0
        \end{align*}
        Substituting $\bv$ with $\bv^*(\balpha')$ and $\bv^*(\balpha)$ in the above first and second inequalities respectively yields:
        \begin{align*}
            \langle \bv^*(\balpha') - \bv^*(\balpha), \nabla_{2} \Phi(\balpha, \bv^*(\balpha)) \rangle \geq 0,\\
            \langle \bv^*(\balpha) - \bv^*(\balpha'), \nabla_{2} \Phi(\balpha', \bv^*(\balpha')) \rangle \geq 0
        \end{align*}
        Adding up the above two inequalities yields:
        \begin{align}
            \langle \bv^*(\balpha') - \bv^*(\balpha), \nabla_{2} \Phi(\balpha, \bv^*(\balpha))-\nabla_{2} \Phi(\balpha', \bv^*(\balpha'))  \rangle \geq 0,\label{eq: lipschitz 1} 
        \end{align}
        Since $\Phi(\balpha,\cdot)$ is $\mu$ strongly convex, we have:
        \begin{align}
            \langle \bv^*(\balpha') - \bv^*(\balpha), \nabla_2 \Phi(\balpha, \bv^*(\balpha')) - \nabla_2 \Phi(\balpha, \bv^*(\balpha)) \geq \mu \|\bv^*(\balpha') - \bv^*(\balpha)\|^2. \label{eq: lipschitz 2} 
        \end{align}
        Adding up (\ref{eq: lipschitz 1}) and (\ref{eq: lipschitz 2}) yields:
        \begin{align*}
            \langle \bv^*(\balpha') - \bv^*(\balpha), \nabla_2 \Phi(\balpha, \bv^*(\balpha')) - \nabla_2 \Phi(\balpha', \bv^*(\balpha')) \geq \mu \|\bv^*(\balpha') - \bv^*(\balpha)\|^2
        \end{align*}
        Finally, using $L$ smoothness of $\Phi$ will conclude the proof:
        \begin{align*}
            L\| \bv^*(\balpha') - \bv^*(\balpha)\| \| \balpha   -  \balpha' \| &\geq \mu \|\bv^*(\balpha') - \bv^*(\balpha)\|^2\\
            \Longleftrightarrow  \kappa \| \balpha   -  \balpha' \| &\geq  \|\bv^*(\balpha') - \bv^*(\balpha)\| 
        \end{align*}
    \end{proof}
\end{lemma}

\begin{lemma} [Optimality Gap]\label{lem:opt gap}
Let $\Phi(\balpha,\bv)$ be defined in (\ref{eq: phi definition}). Let $\hat{\bv} = \cP_{\cW}(\tilde{\bv} - \frac{1}{L} \nabla_{\bv} \Phi(\hat{\balpha},\tilde{\bv}))$. If we assume each $f_i$ is $L$-smooth, $\mu$-strongly convex and with gradient bounded by $G$, then the following statement holds true:
    \begin{align*}
      \Phi( {\balpha}^*,\hat{\bv}) -  \Phi( {\balpha}^*, \bv^* ) \leq   2L\|\tilde{\bv}  - {\bv}^*(\hat\balpha)\|^2+ \pare{2\kappa^2_{\Phi} L + \frac{4NG^2}{ L} }\| \hat\balpha  - \balpha^* \|^2 .
    \end{align*}
    where  $\kappa_\Phi = \frac{\sqrt{N}G}{\mu}$,  ${\bv}^* = \arg\min_{\bv \in \cW} \Phi( {\balpha}^*,  {\bv} )$.
    \end{lemma}

    \begin{proof}
    First we show that $\nabla_{\bv} \Phi(\balpha,\bv)$ is $ {\sqrt{N}G} $ Lipschitz in $\balpha$. To see this:
    \begin{align*}
        \norm{\nabla_{\bv} \Phi(\balpha,\bv) - \nabla_{\bv} \Phi(\balpha',\bv)} &= \norm{ \sum_{j=1}^N \alpha_i(j) \nabla f_j(\bv) - \sum_{j=1}^N \alpha'_i(j) \nabla f_j(\bv)  }\\
        &\leq \sqrt{N}G \norm{\balpha_i - \balpha'_i}.
    \end{align*}
    Hence due to Lemma~\ref{lem: lipschitz}, we know $\bv^*(\balpha)$ is $\kappa_{\Phi}:= \frac{\sqrt{N}G}{\mu}$ Lipschitz. According to property of projection, we have:
    \begin{align*}
        0 &\leq \inprod{\bv - \hat{\bv}}{L(\hat{\bv} - \tilde{\bv})+\nabla_{\bv} \Phi(\hat{\balpha},\tilde{\bv}) }\\
        &= \underbrace{\inprod{\bv - \hat{\bv}}{L(\hat{\bv} - \tilde{\bv})+\nabla_{\bv} \Phi( {\balpha}^*,\tilde{\bv}) } }_{T_1}+ \underbrace{\inprod{\bv - \hat{\bv}}{ \nabla_{\bv} \Phi(\hat{\balpha},\tilde{\bv}) -\nabla_{\bv} \Phi( {\balpha}^*,\tilde{\bv})}}_{T_2}.
    \end{align*}

For $T_1$, we notice:
\begin{align*}
   \inprod{\bv - \hat{\bv}}{L(\hat{\bv} - \tilde{\bv})+\nabla_{\bv} \Phi( {\balpha}^*,\tilde{\bv}) }  & =L \inprod{\bv - \tilde{\bv}} {\hat{\bv} - \tilde{\bv}} + \frac{1}{\eta} \inprod{ \tilde{\bv}-\hat{\bv}}{\hat{\bv} - \tilde{\bv}} +  \inprod{\bv - \hat{\bv}}{ \nabla_{\bv} \Phi( {\balpha}^*,\tilde{\bv}) }\\
   & =L \inprod{\bv - \tilde{\bv}} {\hat{\bv} - \tilde{\bv}} - L \norm{ \tilde{\bv}-\hat{\bv}}^2 +   \inprod{\bv - \hat{\bv}_i}{ \nabla_{\bv} \Phi( {\balpha}^*,\tilde{\bv}) } \\
   & \leq L ( \norm{\bv - \tilde{\bv}}^2 + \frac{1}{4}\norm{\hat{\bv} - \tilde{\bv}}^2 )- L \norm{ \tilde{\bv}-\hat{\bv}}^2 +  \underbrace{\inprod{\bv - \hat{\bv}}{ \nabla_{\bv} \Phi( {\balpha}^*,\tilde{\bv}) }}_{\spadesuit } 
\end{align*}
where at last step we used Young's inequality.
To bound $\spadesuit$,  we apply the $L$ smoothness and $\mu$ strongly convexity of $\Phi(\balpha,\cdot)$:
\begin{align*}
    \inprod{\bv - \hat{\bv}_i}{ \nabla_{\bv} \Phi( {\balpha}^*_i,\tilde{\bv}_i) } &= \inprod{\bv - \tilde{\bv}_i}{ \nabla_{\bv} \Phi( {\balpha}^*,\tilde{\bv}) } + \inprod{\tilde{\bv} - \hat{\bv}}{ \nabla_{\bv} \Phi( {\balpha}^*,\tilde{\bv}) }\\
    &\leq \Phi( {\balpha}^*, \bv ) - \Phi( {\balpha}^*,\tilde{\bv}) - \frac{\mu}{2}\norm{\tilde{\bv} -  {\bv} }^2+ \Phi( {\balpha}^*,\tilde{\bv}) - \Phi( {\balpha}^*,\hat{\bv}) + \frac{L}{2}\norm{\tilde{\bv} - \hat{\bv}}^2\\
    &\leq \Phi( {\balpha}^*, \bv ) - \Phi( {\balpha}^*,\hat{\bv})   - \frac{\mu}{2}\norm{\tilde{\bv} -  {\bv} }^2+  \frac{L}{2}\norm{\tilde{\bv} - \hat{\bv}}^2
\end{align*}
Putting above bound back yields:
\begin{align*}
     \inprod{\bv - \hat{\bv}}{\frac{1}{\eta}(\hat{\bv} - \tilde{\bv})+\nabla_{\bv} \Phi( {\balpha}^*,\tilde{\bv}) } \leq \Phi( {\balpha}^*, \bv ) - \Phi( {\balpha}^*,\hat{\bv})  + \frac{1}{2\eta}\norm{\tilde{\bv} -  {\bv} }^2 -\pare{\frac{3L}{4}- \frac{L}{2}}  \norm{\tilde{\bv} - \hat{\bv}}^2
\end{align*}
Now we switch to bounding $T_2$. Applying Cauchy-Schwartz yields:
\begin{align*}
    \inprod{\bv - \hat{\bv}_i}{ \nabla_{\bv} \Phi(\hat{\balpha}_i,\tilde{\bv}_i) -\nabla_{\bv} \Phi( {\balpha}^*_i,\tilde{\bv}_i)} &\leq \frac{L}{4}\norm{\bv - \tilde{\bv}_i}^2 + \frac{L}{4}\norm{\tilde{\bv} - \hat{\bv}_i}^2 + \frac{4}{L} \norm{\nabla_{\bv} \Phi(\hat{\balpha}_i,\tilde{\bv}_i) -\nabla_{\bv} \Phi( {\balpha}^*_i,\tilde{\bv}_i)}^2\\
    &\leq \frac{L}{4}\norm{\bv - \tilde{\bv}_i}^2 + \frac{L}{4}\norm{\tilde{\bv} - \hat{\bv}_i}^2 + \frac{4NG^2}{L} \norm{  \hat{\balpha}_i  -  {\balpha}^*_i }^2
\end{align*}
where at last step we apply $\sqrt{N}G$ smoothness of $\Phi(\cdot,\bv)$.
Putting pieces together yields:
\begin{align*}
    0 &\leq \Phi( {\balpha}^*_i, \bv ) - \Phi( {\balpha}^*_i,\hat{\bv}_i)  + \frac{L}{2}\norm{\tilde{\bv}_i -  {\bv} }^2    +\frac{L}{2}\norm{\bv - \tilde{\bv}_i}^2   + \frac{4NG^2}{  L} \norm{  \hat{\balpha}_i  -  {\balpha}^*_i }^2
\end{align*}
Re-arranging terms and setting $\bv = \bv^*(\balpha^*) = \arg\min_{\bv\in\cW} \Phi(\balpha^*,\bv)$ yields:
\begin{align*}
     \Phi( {\balpha}^*,\hat{\bv}) -  \Phi( {\balpha}^*, \bv^* ) \leq   L\norm{\tilde{\bv} -  {\bv}^* }^2     + \frac{4NG^2}{ L} \norm{  \hat{\balpha}  -  {\balpha}^* }^2.
\end{align*}

At last, due to the $\kappa_{\Phi}$-Lipschitzness  property of of $\bv^*(\cdot)$  as shown in Lemma~\ref{lem: lipschitz}, it follows that:
        \begin{align*}
            L\|\tilde{\bv}  - {\bv}^*(\balpha^*)\|^2 &\leq L\|\tilde{\bv} - {\bv}^*(\hat\balpha)\|^2+L\|{\bv}^*(\hat\balpha) - {\bv}^*(\balpha^*)\|^2 \\
            &\leq 2L\|\tilde{\bv}  - {\bv}^*(\hat\balpha)\|^2+2\kappa^2_{\Phi} L\| \hat\balpha  - \balpha^* \|^2,
        \end{align*}
        as desired.
     
    \end{proof}


\setcounter{algocf}{0}
\renewcommand{\thealgocf}{A\arabic{algocf}}
  \begin{algorithm2e}[t]
	\DontPrintSemicolon
    \caption{Discrepancy Estimation at Optimum}
	\label{algorithm: alpha}
 
    	\textbf{Input:} Number of clients $N$, number of local steps $K$ , number of communications rounds $R$
	\\

       \For{$r = 0,\ldots,R-1$} 
        {   
        \textbf{parallel} \For{\text{\em client} \ $i = 0,...,N-1$}	
        {
         Client $i$ initializes model $\bw^{r,0}_i = \bw^{r}_i$.\\
	   \For{$t = 0,...,K-1$}
	    {
        $ \bw^{r,t+1}_i = \bw^{r,t}_i - \gamma \nabla f_i(\bw^{r,t}_i;\xi_i^{r,t})$ where $\xi_i^{r,t}$ is a mini-batch  sampled from $\mathcal{S}_i$.\\
        } 
        Client $i$ sends $\bw^{r,K}_i$ to Server.\\
     }
     Server computes
     $\bw^{r+1} = \cP_{\cW}\left(\frac{1}{N}\sum_{i=1}^N \bw_i^{r,K}\right)$\\ 
     Server broadcasts  $\bw^{r+1}$ to all clients.
     
 } 
  Server computes $\hat{\balpha}_i, i=1, 2, \ldots, N$ by running $T_{\balpha}$ steps of GD on   $g_i(\bm{w}^R,\balpha)$. \\
     \textbf{Output:} $\hat{\balpha}_1, \ldots, \hat{\balpha}_{N}$ . 
\end{algorithm2e}

\subsection{Proof of Convergence of Theorem~\ref{thm: 2stage alpha} }

In this section we are going to prove the result in Theorem~\ref{thm: 2stage alpha}.   To this end, we need to show that mixing parameters we compute by first learning the global model and then solving the optimization problem in objective~\eqref{eq:relaxed-alpha-final} (as depicted in Algorithm~\ref{algorithm: alpha}) converges to optimal values. Notice that in Algorithm~\ref{algorithm: alpha} we do not solve $g_i(\bw^*,\balpha)$ directly, but optimize $g_i(\bw^R,\balpha)$ on $ \balpha$   for $T_{\balpha}$ iterations of GD. Hence,
firstly we need to show that optimizing the surrogate function will also guarantee the convergence of output of algorithm $\widehat{\balpha}$ to $\balpha^*$ by deriving a property of the objective in~(\ref{eq:relaxed-alpha-final}). Formally the property is captured by the following lemma.

\begin{lemma}\label{lem: lipschitz alpha}
Let $  g(\bw,\balpha):= \sum\nolimits_{j=1}^{N}{\alpha_j \norm{\nabla f_i(\bw) -\nabla f_j(\bw)}^2}+ \lambda \sum\nolimits_{j=1}^{N} \alpha_j^2/n_j $ and $\balpha^*_g(\bw) = \arg\min_{\balpha \in \Delta_N} g(\bw,\balpha)$. Let $\bw^R$ be the output of Algorithm~\ref{algorithm: alpha}. Then the following statement holds:
\begin{align*}
    \norm{\balpha^*_g(\bw^R) - \balpha^*_g(\bw^*)} 
    &\leq  \kappa_g^2\sum_{j=1}^N\pare{ 2\norm{\nabla f_i( \bw^*) -\nabla f_j(\bw^*)}^2 + 4L^2\norm{  \bw^R - \bw^*}^2} 4L\norm{ \bw^R-\bw^* }^2
\end{align*} 
where $\kappa_g := \frac{n_{\max}}{2\lambda }$.
\begin{proof}
 
Define function

\begin{align}
     W(\bz, \balpha) = \sum\nolimits_{j=1}^{N}{\alpha_j z_j}+ \lambda \sum\nolimits_{j=1}^{N} \alpha_j^2/n_j 
\end{align}

Apparently, $W(\bz, \balpha)$ is linear in $\bz$ and $2\frac{\lambda}{n_{\max}} $ strongly convex in $\bz$. Next we show that $\nabla_{\balpha} W(\bz,\balpha)$ is Lipschitz in $\bw$. To see this,
 \begin{align*}
        \norm{\nabla_{\balpha}  W(\bz, \balpha) - \nabla_{\balpha}  W(\bz', \balpha)} &= \norm{  [z_1,...,z_N]  -  [z'_1,...,z'_N]}\\
        &\leq \norm{\bz - \bz' }.
    \end{align*}

Then, according to Proposition~\ref{lem: lipschitz}, $\balpha^*_W(\bz) := \arg\min_{\balpha\in\Delta_N} W(\bz, \balpha)$ is $\kappa_g$ lipschitz in $\bz$ where $\kappa_g = \frac{n_{\max}}{2\lambda }$, i.e., $\norm{\balpha^*_W(\bz) - \balpha^*_W(\bz')} \leq \kappa_g\norm{ \bz - \bz'}$. Now, let us consider the objective ~(\ref{eq:relaxed-alpha-final}):
\begin{align*}
    g(\bw,\balpha):= \sum\nolimits_{j=1}^{N}{\alpha_j \norm{\nabla f_i(\bw) -\nabla f_j(\bw)}^2}+ \lambda \sum\nolimits_{j=1}^{N} \alpha_j^2/n_j 
\end{align*}
We define $\balpha^*_g(\bw) = \arg\min_{\balpha \in \Delta_N} g(\bw,\balpha)$.

We set 
\begin{align*}
 \bz^R &= \left[ \norm{\nabla f_i( \bw^R) -\nabla f_1(\bw^R) }^2,..., \norm{\nabla f_i( \bw^R) -\nabla f_N( \bw^R) }^2\right],\\
    \bz^* &= \left[ \norm{\nabla f_i(\bw^*) -\nabla f_1(\bw^*) }^2,..., \norm{\nabla f_i(\bw^*) -\nabla f_N(\bw^*) }^2\right].
\end{align*}
Then we know that 
\begin{align}
    \norm{\balpha^*_g(\bw^R) - \balpha^*_g(\bw^*)}^2 &= \norm{\balpha^*_W(\bz^R) - \balpha^*_W(\bz^*)}^2 \leq  \kappa_g^2 \norm{\bz^R - \bz^*}^2 \\
    &\leq  \kappa_g^2\sum_{j=1}^N \left|\norm{\nabla f_i( \bw^R) -\nabla f_j(\bw^R) }^2-\norm{\nabla f_i( \bw^*) -\nabla f_j(\bw^*) }^2  \right|^2  \label{eq:lipschitz alpha 1}\\
    &\leq \kappa_g^2 \sum_{j=1}^N \left|\pare{\nabla f_i( \bw^R) -\nabla f_j(\bw^R) +\nabla f_i( \bw^*) -\nabla f_j(\bw^*) } \right. \nonumber\\
    & \qquad \qquad\left. \times \pare{\nabla f_i( \bw^R) -\nabla f_j(\bw^R)-\nabla f_i( \bw^*) +\nabla f_j(\bw^*) }  \right|^2 \nonumber\\
    &\leq \kappa_g^2 \sum_{j=1}^N { \norm{\nabla f_i( \bw^R) -\nabla f_j(\bw^R) +\nabla f_i( \bw^*) -\nabla f_j(\bw^*) }^2 } 4L^2\norm{ \bw^R-\bw^* }^2\nonumber
\end{align} 
Since $\norm{\nabla f_i( \bw^R) -\nabla f_j(\bw^R)} \leq \norm{\nabla f_i( \bw^*) -\nabla f_j(\bw^*)} + 2L\norm{  \bw^R - \bw^*}$, we can conclude that
\begin{align*}
    \norm{\balpha^*_g(\bw^R) - \balpha^*_g(\bw^*)} 
    &\leq  \kappa_g^2\sum_{j=1}^N\pare{ 2\norm{\nabla f_i( \bw^*) -\nabla f_j(\bw^*)}^2 + 4L^2\norm{  \bw^R - \bw^*}^2} 4L\norm{ \bw^R-\bw^* }^2 .
\end{align*} 

   \end{proof}
\end{lemma}

With above lemma, to show the convergence of $\hat \balpha$ to $\balpha^*$, we do the following decomposition
\begin{align*}
    \norm{\hat{\balpha} -  {\balpha^*}}^2 &\leq 2   \norm{\hat{\balpha} -  {\balpha^*_g(\bw^R)}}^2 + 2   \norm{\balpha^*_g(\bw^R) - \balpha^*_g(\bw^*) }^2 \\
    &\leq 2 (1-\mu\eta_{\balpha} )^K + 2 \kappa_g^2\sum_{j=1}^N\pare{ 2\norm{\nabla f_i( \bw^*) -\nabla f_j(\bw^*)}^2 + 4L^2\norm{  \bw^R - \bw^*}^2} 4L\norm{ \bw^R-\bw^* }^2. 
\end{align*}
Now it remains to show the convergence of Local SGD {\em last iterate} $\bw^R$ to optimal solution $\bw^*$. By convention, we use $\bw^t = \frac{1}{N}\sum_{i=1}^N \bw^t_i$ to denote the virtual average iterates.

\begin{lemma}[One iteration analysis of Local SGD] \label{lem: one iteration local SGD}
Under the condition of Theorem~\ref{thm: 2stage alpha}, the following statement holds true for any $t \in [T]$:
\begin{align*}
        \E\norm{\bw^{r,t+1} - \bw^*}^2  
        &\leq (1-\mu\gamma)\E\norm{\bw^{r,t} - \bw^*}^2 -(2\gamma-4\gamma^2 L)\E\pare{F(\bw^*) - F(\bw^{r,t})}  \\
        &\quad  +(\gamma {L}+2\gamma^2L^2)  \frac{1}{N}\sum_{i=1}^N\norm{\bw^{r,t}_i - \bw^{r,t}}^2+ \gamma^2 \frac{\delta^2}{N}.
    \end{align*}
\begin{proof}
According to updating rule in Algorithm~\ref{algorithm: alpha}, we have the following identity:    
    \begin{align}
        \E\norm{\bw^{r,t+1} - \bw^*}^2 &= \E\norm{\bw^{r,t} - \bw^*}^2 - 2\gamma  \E\inprod{\frac{1}{N}\sum_{i=1}^N \nabla f_i (\bw^{r,t}_i;z_i^{r,t})}{\bw^{r,t} - \bw^*} + \gamma^2 \E \norm{\frac{1}{N}\sum_{i=1}^N \nabla f_i (\bw^{r,t}_i)}^2  \\
        &= \E\norm{\bw^{t} - \bw^*}^2 \underbrace{- 2\gamma   \inprod{\frac{1}{N}\sum_{i=1}^N \nabla f_i (\bw^{r,t}_i )}{\bw^{r,t} - \bw^*}}_{T_1} \nonumber \\
        &\quad + \gamma^2 \underbrace{\E \norm{\frac{1}{N}\sum_{i=1}^N \nabla f_i (\bw^{r,t}_i;z_i^{r,t})}^2}_{T_2} + \gamma^2 \frac{\delta^2}{N}. \label{eq:one iteration fedavg}
    \end{align}
    For $T_1$, since each $f_j$ is $L$ smooth and $\mu$ strongly convex, we have:
    \begin{align*}
        -2\gamma   \inprod{\frac{1}{N}\sum_{i=1}^N \nabla f_i (\bw^{r,t}_i )}{\bw^{t} - \bw^*} &= -2\gamma   \inprod{\frac{1}{N}\sum_{i=1}^N \nabla f_i (\bw^{t}_i )}{\bw^{t} - \bw^{t}_i+ \bw^{r,t}_i-\bw^*}\\
        & \leq -2\gamma   \inprod{\frac{1}{N}\sum_{i=1}^N \nabla f_i (\bw^{t}_i )}{\bw^{r,t} - \bw^{r,t}_i+ \bw^{r,t}_i-\bw^*} \\
        & \leq 2 \gamma\frac{1}{N}\sum_{i=1}^N\pare{f_i(\bw^*) - f_i(\bw^{r,t})-\frac{\mu}{2}\norm{\bw^{r,t}_i - \bw^*}^2 + \frac{L}{2}\norm{\bw^{r,t}_i - \bw^{r,t}}^2}.
    \end{align*}
    Due to Jensen's inequality we know: $-\frac{1}{N}\sum_{i=1}^N\frac{\mu}{2}\norm{\bw^t_i - \bw^*}^2 \leq - \frac{\mu}{2}\norm{\bw^t  - \bw^*}^2$. Hence we know:
    \begin{align*}
        -2\gamma   \inprod{\frac{1}{N}\sum_{i=1}^N \nabla f_i (\bw^{r,t}_i )}{\bw^{r,t} - \bw^*}
        & \leq 2\gamma\pare{F(\bw^*) - F(\bw^{r,t})-\frac{\mu}{2}\norm{\bw^{r,t}  - \bw^*}^2 + \frac{L}{2} \frac{1}{N}\sum_{i=1}^N\norm{\bw^{r,t}_i - \bw^{r,t}}^2}.
    \end{align*}
    For $T_2$, we have:
    \begin{align*}
        \E \norm{\frac{1}{N}\sum_{i=1}^N \nabla f_i (\bw^{r,t}_i)}^2 = 2\E \norm{\frac{1}{N}\sum_{i=1}^N \nabla f_i (\bw^{r,t}_i) -\nabla F( \bw^{r,t} ) }^2+2\E \norm{  \nabla F( \bw^{r,t} ) }^2\\
        \leq 2L^2\frac{1}{N}\sum_{i=1}^N\E \norm{  \bw^{r,t}_i- \bw^{r,t}  }^2 + 4L \pare{F( \bw^{r,t} ) - F( \bw^{*} )}.
    \end{align*}
    Now, plugging $T_1$ and $T_2$ back to~(\ref{eq:one iteration fedavg}) yields:
    \begin{align*}
        \E\norm{\bw^{r,t+1} - \bw^*}^2  
        &\leq (1-\mu\gamma)\E\norm{\bw^{r,t} - \bw^*}^2 -(2\gamma-4\gamma^2 L)\E\pare{F(\bw^*) - F(\bw^{r,t})}  \\
        &\quad  +(\gamma {L}+2\gamma^2L^2)  \frac{1}{N}\sum_{i=1}^N\norm{\bw^{r,t}_i - \bw^{r,t}}^2+ \gamma^2 \frac{\delta^2}{N}.
    \end{align*}
    \end{proof}

\end{lemma}

\begin{lemma}{\cite[Lemma~8]{woodworth2020minibatch}}\label{lem: local deviation}
For the iterates $\{\bw^{r,t}_i\}$ generated in Algorithm~\ref{algorithm: alpha}, the following statement holds true:
    \begin{align*}
        \frac{1}{N}\sum_{i=1}^N\norm{\bw^{r,t}_i - \bw^{r,t}}^2 \leq 3K \gamma^2 \delta^2 + 6K^2 \gamma^2 \zeta^2.
    \end{align*}
\end{lemma}

\begin{lemma} [Last iterate convergence of Local SGD] \label{lem:last iterate} 
Under the conditions of Theorem~\ref{thm: 2stage alpha}, the following statement holds true for the iterates in Algorithm~\ref{algorithm: alpha}:
    \begin{align*}
        \E\norm{\bw^R-\bw^*}^2 \leq (1-\mu\gamma)^{RK}\E\norm{\bw^{0} - \bw^*}^2  + \frac{1}{\mu\gamma}(\gamma {L}+2\gamma^2L^2) \pare{3K \gamma^2 \delta^2 + 6K^2 \gamma^2 \zeta^2}  +   \frac{\gamma \delta^2}{\mu N}
    \end{align*}
    \begin{proof}
    We first unroll the recursion in Lemma~\ref{lem: one iteration local SGD} from $t=K$ to $0$, within one communication round:
     \begin{align*}
        \E\norm{\bw^{r,K} - \bw^*}^2  
        &= (1-\mu\gamma)^K\E\norm{\bw^{r,0} - \bw^*}^2 -\sum_{t=0}^{K-1}(1-\mu\gamma)^{K-t} (2\gamma-4\gamma^2 L)\E\pare{F(\bw^*) - F(\bw^{r,t})}  \\
        &\quad  +\sum_{t=0}^{K-1}(1-\mu\gamma)^{K-t} (\gamma {L}+2\gamma^2L^2)  \frac{1}{N}\sum_{i=1}^N\norm{\bw^{r,t}_i - \bw^{r,t}}^2+ \sum_{t=0}^{K-1}(1-\mu\gamma)^{K-t} \gamma^2 \frac{\delta^2}{N}
    \end{align*}
    Since we choose $\gamma \leq \frac{1}{2L}$, we know $\sum_{t=0}^{K-1}(1-\mu\gamma)^{K-t} (2\gamma-4\gamma^2 L)\E\pare{F(\bw^*) - F(\bw^{r,t})}  \geq 0$. Plugging in Lemma~\ref{lem: local deviation} yields:
    \begin{align*}
        \E\norm{\bw^{R} - \bw^*}^2  
        &= (1-\mu\gamma)^{RK}\E\norm{\bw^{0} - \bw^*}^2  + \frac{1}{\mu\gamma}(\gamma {L}+2\gamma^2L^2) \pare{3K \gamma^2 \delta^2 + 6K^2 \gamma^2 \zeta^2}  +   \frac{\gamma \delta^2}{\mu N},
    \end{align*}
    Plugging in  $\gamma = \frac{\log(RK)}{\mu RK}$ gives the convergence rate:
    \begin{align*}
          \E\norm{\bw^{R} - \bw^*}^2  \leq \tilde{O}\pare{\frac{\E\norm{\bw^{0} - \bw^*}^2 }{RK} + \kappa \pare{ \frac{ \delta^2}{\mu^2 R^2K }   +  \frac{  \zeta^2}{\mu^2 R^2  }}  +   \frac{  \delta^2 }{\mu^2 N RK} } ,
    \end{align*}
    which concludes the proof.
    \end{proof}
\end{lemma}

Equipped with above results, we are now ready to provide the convergence of main theorem.

\begin{proof}[Proof of Theorem~\ref{thm: 2stage alpha}]
The proof simply follows  from Lemma~\ref{lem: lipschitz alpha}:
\begin{align*}
    \norm{\hat{\balpha} -  {\balpha^*}}^2 &\leq 2   \norm{\hat{\balpha} -  {\balpha^*_g(\bw^R)}}^2 + 2   \norm{\balpha^*_g(\bw^R) - \balpha^*_g(\bw^*) }^2 \\
    &\leq 2 (1-\mu\eta_{\balpha} )^{T_{\balpha}}  + 8L \kappa_g^2\sum_{j=1}^N\pare{ 2\norm{\nabla f_i( \bw^*) -\nabla f_j(\bw^*)}^2 + 4L^2\norm{  \bw^R - \bw^*}^2} \norm{ \bw^R-\bw^* }^2 \\
    &\leq 2 (1-\mu\eta_{\balpha} )^{T_{\balpha}}  + 8L\kappa_g^2\pare{ 2\bar{\zeta}_i(\bw^*) +  4N L^2\norm{  \bw^R - \bw^*}^2}  \norm{ \bw^R-\bw^* }^2 
\end{align*}
    Plugging in the convergence of $\norm{\bw^R - \bw^*}^2$ from Lemma~\ref{lem:last iterate}, and the stepsize $\eta_{\balpha} = \frac{1}{L_g}$ for $\balpha$ yields:
    \begin{align*}
       \E  \|\balpha^R_i -  \balpha^*_i\|^2 \leq \tilde{O}\pare{ \exp(-\frac{T_{\balpha}}{\kappa_g})+  \kappa_g^2 \bar{\zeta}_i(\bw^*)   L^2\pare{\frac{D^2 }{RK} + \kappa \pare{ \frac{ \delta^2}{\mu^2 R^2K }   +  \frac{  \zeta^2}{\mu^2 R^2  }}  +   \frac{  \delta^2 }{\mu^2 N RK}  }}.
    \end{align*} 
\end{proof}

\subsection{Proof of Convergence of Shuffling Local SGD}
In this section, we are going to prove the convergence of proposed shuffled variant of Local SGD (Theorem~\ref{thm:2stagec perm}). The whole proof framework follows the analysis of vanilla shuffling SGD, but notice that there are two differences. First, in vanilla shuffling SGD, in each epoch, algorithm only updates on each component function $f_j$ once, while here we have to take $K$ steps of SGD update on each component function. Second, we are considering a weighted sum objective in contrary to  averaged objective in~\cite{cho2023convergence}, which means we need to rescale the objective when we apply without-replacement concentration inequality.  Even though our algorithm solves models for $N$ clients, for the sake of simplicity, throughout the proof we only show the convergence of one client's model. The algorithm from one client point of view is described in Algorithm~\ref{algorithm: KSGD one client}, where we drop the client index for notational convenience.
 \begin{algorithm2e}[t]
	\DontPrintSemicolon
    \caption{{Shuffling Local SGD (\texttt{One Client)}}}
	\label{algorithm: KSGD one client}
 
    	\textbf{Input:} Clients $0,...,N-1$, Number of Local Steps $K$ , Number of Epoch $R$, Mixing parameter $\hat{\balpha}$ 
	\\ 
       \textbf{Epoch} \For{$r = 0,...,R-1$} 
        {  
        
        Server generates permutation $\sigma_r: [N] \mapsto [N]$.\\
      
 {
        Client sets initial model $\bv^{r,0}  = \bv^{r} $.\\
	   \For{$j = 0,...,N-1$}
	    { 
        Server sends $\bv^{r,j}$ to Client $\sigma_r(j)$.  \\
        $ \bv^{r,j+1}  = \texttt{SGD-Update}(\bv^{r,j} , \eta, \sigma_r(j),K,\hat{\balpha} )$.\\
        
        } 
        Client $i$ does projection: $\bv^{r+1}  = \cP_{\cW}(\bv^{r,N} )$.\\  
        
     }
     
 } 
 \textbf{Output:} $\hat{\bv}  = \bv^{R}$. 
 
 \setcounter{AlgoLine}{0}
  \SetKwProg{myproc}{\texttt{SGD-Update}($\bv,\eta,j,K, {\balpha}$)}{}{}
  \myproc{}{
  {
   Initialize $\bv^0 = \bv$\\
   \For{$t = 0,...,K-1 $}
   {$\bv^t = \bv^{t-1} - \eta  {\alpha}(j) N {\nabla} f_j(\bv^{t-1};\xi^{t-1})$
   }
   
   } 
  \textbf{Output} $\bv^K $ \;}
\end{algorithm2e} 
\begin{proposition}\label{prop:updating rule}
    
    Assume a sequence $\{\bw^t\}_{t=1}^K$ is obtained by 
    $$\bw^{t} = \bw^{t-1} - \eta \alpha N \nabla f(\bw^{t-1};\xi^{t-1}), \ t = 1,\ldots,K,$$
 then we have
    \begin{align*}
        \bw^{t+1} =\bw^0 &-\left( \sum_{\tau=0}^{t} \prod_{t'=t}^{\tau+1} \left(\mI - \alpha N \eta \mH_{t'}\right) \right) \eta  \alpha N \nabla f(\bw^0) -  \sum_{\tau=0}^{t} \prod_{t'=t}^{\tau+1} \left(\mI - \alpha N \eta \mH_{t'}\right)   \eta \alpha N \bdelta^t , \quad \forall 0\leq t \leq K-1,
    \end{align*}
   where $\bdelta^t := \nabla f(\bw^t;\xi^t) - \nabla f(\bw^t)$, and
   by convention, we define $\prod_{j=a}^b \mA_j = \mI $ if $a < b$.
    \begin{proof}
        According to updating rule, we have:
        \begin{align*}
            \bw^{t+1} - \bw^0 &= \bw^t - \bw^0 - \eta \alpha N\nabla f(\bw^t;\xi^t) \\
            &= \bw^t - \bw^0 - \eta \alpha N\nabla f(\bw^t ) -  \eta \alpha N \bdelta^t\\
            &=\bw^t - \bw^0 - \eta \alpha N\nabla f(\bw^0) - \eta  \alpha N(\nabla f(\bw^t)-\nabla f(\bw^0) )-  \eta \alpha N \bdelta^t.
        \end{align*}
        Since $f$ is $L$ smooth, and according to Mean Value Theorem, there is a matrix $\mH_t$ satisfying $\mu \mI \preceq \mH_t \preceq L\mI$, such that $\nabla f(\bw^t)-\nabla f(\bw^0) = \mH_t (\bw^t - \bw^0)$. Hence we have:
        \begin{align*}
            \bw^{t+1} - \bw^0 &= \left(\mI - \eta \alpha N \mH_{t}\right)(\bw^t - \bw^0) - \eta \alpha N\nabla f(\bw^0)-  \eta \alpha N \bdelta^t.
        \end{align*}
        Unrolling the recursion from $t$ to $0$ will conclude the proof.
    \end{proof}
\end{proposition}

The following lemma establishes the updating rule of models between epochs $r$ and $r+1$. For notational convenience, whenever there is no confusion, we drop the superscript $r$ in $\sigma^r$.
\begin{lemma}[One epoch updating rule]\label{lem:epoch update}
Let $\bv^r$ and $\bv^{r+1}$ be two iterates generated by Shuffling Local SGD (Algorithm~\ref{algorithm: KSGD one client}), then the following updating rule holds:
    \begin{align*}
         \bv^{r+1} &= \bv^r - \sum_{j=1}^{N} \prod_{j'=N}^{j+1}(\mI - \mQ_{j'}\mH_{j'}) (\mQ_{j}\nabla f_{\sigma(j)}(\bv^r)  - \bdelta_j),
    \end{align*}
    where
    \begin{align*}
        \mQ_j &:= \left( \sum_{\tau=0}^{K-1} \prod_{t'=K-1}^{\tau+1} \left(\mI - \eta  \hat{\alpha}(\sigma(j)) N  \mH_{t'}\right) \right) \eta  \hat{\alpha}(\sigma(j)) N  ,\\
        \bdelta_j &: = \sum_{\tau=0}^{K-1} \prod_{t'=t}^{\tau+1} \left(\mI - \hat{\alpha}(\sigma(j)) N \eta \mH_{t'}\right)   \eta  \hat{\alpha}(\sigma(j)) N \bdelta^t_{\sigma(j)},
    \end{align*}
    by convention, we define $\prod_{j=a}^b \mA_j = \mI $ if $a < b$.

\begin{proof}
    According to Proposition~\ref{prop:updating rule}, we have
    \begin{align*}
        \bv^{r,j+1} = \bv^{r,j}& -  \left( \sum_{\tau=0}^{K-1} \prod_{t'=t}^{\tau+1} \left(\mI - \hat{\alpha}(\sigma(j)) N \eta \mH_{t'}\right) \right) \eta  \hat{\alpha}(\sigma(j)) N \nabla f_{ \sigma(j)}(\bv^{r,j})\\
         &-  \sum_{\tau=0}^{K-1} \prod_{t'=t}^{\tau+1} \left(\mI - \hat{\alpha}(\sigma(j)) N \eta \mH_{t'}\right)   \eta  \hat{\alpha}(\sigma(j)) N \bdelta^t_{\sigma(j)}.
    \end{align*}
    Plugging our definition of $\mQ_j$ and $\bdelta_j$ yields:
     \begin{align*}
        \bv^{r,j+1} - \bv^r = \bv^{r,j} - \bv^r & -  \mQ_j  \nabla f_{\sigma(j)}(\bv^{r,j}) -  \bdelta_j.
    \end{align*}
    Following the same reasoning in the proof of Proposition~\ref{prop:updating rule} will conclude the proof.
\end{proof}
    
\end{lemma}

\begin{lemma}[Summation by parts]\label{lem: summation by parts} Let $\mA_j$ and $\mB_j$ be complex valued matrices. Then the following fact holds:
    \begin{align*}
        \sum_{j=1}^{N} \mA_j \mB_j = \mA_{N} \sum_{j=1}^{N}\mB_j - \sum_{n=1}^{N-1} (\mA_{n+1} - \mA_{n}) \sum_{j=1}^n \mB_j.
    \end{align*}

\end{lemma}
\begin{proposition}[Spectral bound of polynomial expansion]\label{prop:binomial}
Given a collection of matrices $\{\mA_t\}$ and $\{\mB_t\}$, such that $ \mA_t \preceq L \mI$ and $\mB_t \preceq L \mI$, the following bound hold:
    \begin{align*}
        &\left\|\prod_{t=l}^{h} \left(\mI -    a\mA_{t}\right)  - \mI\right\| \leq \sum_{m=1}^{h-l}  \left( \frac{e(h-l)}{m} \right)^m (aL)^m,\\
        &\left\|\prod_{t=l}^{h} \left(\mI -    a\mA_{t}\right)  - \prod_{t=l}^{h} \left(\mI -    b\mB_{t}\right)\right\| \leq \sum_{m=1}^{h-l}  \left( \frac{e(h-l)}{m} \right)^m (aL)^m + \sum_{m=1}^{h-l}  \left( \frac{e(h-l)}{m} \right)^m (bL)^m.
    \end{align*} 
    \begin{proof}
     We start with proving the first statement. Expanding the product yields:
     \begin{align*}
         \prod_{t=l}^{h} \left(\mI -    a\mA_{t}\right) = \mI + \sum_{m=1}^{h-l}(-1)^m a^m \sum_{|S| = m, |S|\subseteq \{l,...,h\}} \prod_{m'\in S} \mA_{m'}.
     \end{align*}
     Hence we have:
     \begin{align*}
          \left\|\prod_{t=l}^{h} \left(\mI -    a\mA_{t}\right)  - \mI\right\| =  \left\|\sum_{m=1}^{h-l}(-1)^m a^m \sum_{|S| = m, |S|\subseteq \{l,...,h\}} \prod_{m'\in S} \mA_{m'}\right\| \leq\sum_{m=1}^{h-l} {h-l \choose m} (aL)^m
     \end{align*}
     According to the upper bound for binomial coefficients: ${h-l \choose m} \leq \left( \frac{e(h-l)}{m} \right)^m$, we have:
     \begin{align*}
         \sum_{m=1}^{h-l} {h-l \choose m} (aL)^m \leq \sum_{m=1}^{h-l}  \left( \frac{e(h-l)}{m} \right)^m (aL)^m.
     \end{align*}

     Then we switch to the second one. Using the same expanding product yields:
     \begin{align*}
         &\left\|\prod_{t=l}^{h} \left(\mI -    a\mA_{t}\right)  - \prod_{t=l}^{h} \left(\mI -    b\mB_{t}\right)\right\| \\
         = &\left\|\sum_{m=1}^{h-l}(-1)^m a^m \sum_{|S| = m, |S|\subseteq \{l,...,h\}} \prod_{m'\in S} \mA_{m'} - \sum_{m=1}^{h-l}(-1)^m b^m \sum_{|S| = m, |S|\subseteq \{l,...,h\}} \prod_{m'\in S} \mB_{m'} \right\| \\
         \leq&\sum_{m=1}^{h-l} {h-l \choose m} (aL)^m+ \sum_{m=1}^{h-l} {h-l \choose m} (bL)^m\\
         \leq & \sum_{m=1}^{h-l}  \left( \frac{e(h-l)}{m} \right)^m (aL)^m + \sum_{m=1}^{h-l}  \left( \frac{e(h-l)}{m} \right)^m (bL)^m.
     \end{align*}
    \end{proof}
\end{proposition}

The following concentration result is the key to bound variance during shuffling updating. The original result holds for the average of gradients, and we will later on generalize it to an arbitrary weighted sum of gradients.

\begin{lemma}[{\cite[Theorem~2]{schneider2016probability}}]  
\label{lem:thm1-sub2}
Suppose $n \geq 2$. Let $\bg_1, \bg_2, \dots, \bg_n \in \R^d$ satisfy $\norm{\bg_j}\leq G$ for all $j$. Let $\bar\bg = \frac{1}{n} \sum_{j=1}^n \bg_j$. Let $\sigma \in  \mS_n$ be a uniform random permutation of $n$ elements. Then, for $i \leq n$, with probability at least $1-p$, we have
\begin{equation*}
    \norm{\frac{1}{i} \sum_{j=1}^i \bg_{\sigma(j)} - \bar \bg} \leq 
    G \sqrt{\frac{8(1-\frac{i-1}{n}) \log \frac{2}{p}}{i}}.
\end{equation*}
\end{lemma} 

\begin{lemma}[Concentration of partial sum of gradients]\label{lem:concentration of grad}
Given a uniformly randomly generated permutation $\sigma$, and simplex vector $\balpha$, if we assume each $\sup_{\bv\in\cW}\|\nabla f_{j}(\bv)\| \leq G$, then the following statement holds true: 
    \begin{align*}
        \left\|\sum_{j=0}^n \hat{\alpha}(\sigma(j)) \nabla f_{\sigma(j)}(\bv^r)\right\| \leq G\sqrt{8n\log(1/p)} + \frac{n}{N} \left\|\   \nabla \Phi(\hat{\balpha}, \bv^r)\right\|.
    \end{align*}
    \begin{proof}
    The proof works by re-writing weighted sum of vectors to average of the these vectors:
        \begin{align*}
            \left\|\sum_{j=0}^n \hat{\alpha}(\sigma(j)) \nabla f_{\sigma(j)}(\bv^r)\right\| &= \frac{1}{N}\left\|\sum_{j=0}^n \hat{\alpha}(\sigma(j)) N \nabla f_{\sigma(j)}(\bv^r)\right\|\\
            &=\frac{1}{N}\left(\left\|\sum_{j=0}^n \hat{\alpha}(\sigma(j)) N \nabla f_{\sigma(j)}(\bv^r) - n \nabla \Phi(\hat{\balpha}, \bv^r)\right\|+n\left\|\   \nabla \Phi(\hat{\balpha}, \bv^r)\right\| \right)\\
            & \leq G\sqrt{8n\log(1/p)} + \frac{n}{N} \left\|\   \nabla \Phi(\hat{\balpha}, \bv^r)\right\|.
        \end{align*}
    \end{proof}
\end{lemma}

\begin{proposition}[Spectral norm bound of $\mQ$]\label{prop:spectral bound Q}
Let $\mQ_j$ be defined in (\ref{eq: Q def}). Then the following bound for the spectral norm of $\mQ_j$ holds true for all $j\in[N]$:
    \begin{align*}
        \left\|\mQ_j\right\| \leq \eta \hat{\alpha}(\sigma(j)) N K(1+\eta N L)^K
    \end{align*}
    \begin{proof}
    The proof can be completed by writing down the definitin of $\mQ_j$ and applying Cauchy-Schwartz inequality:
    \begin{align*}
          \norm{\mQ_j} &=  \norm{ \left( \sum_{\tau=0}^{K-1} \prod_{t'=K-1}^{\tau+1} \left(\mI - \eta  \hat{\alpha}(\sigma(j)) N'  \mH_{t'}\right) \right) \eta  \hat{\alpha}(\sigma(j)) N }\\
          &\leq \eta  \hat{\alpha}(\sigma(j)) N  \sum_{\tau=0}^{K-1} \prod_{t'=K-1}^{\tau+1}\norm{   \left(\mI - \eta  \hat{\alpha}(\sigma(j)) N'  \mH_{t'}\right)  }\\
          &\leq \eta  \hat{\alpha}(\sigma(j)) N  \sum_{\tau=0}^{K-1} \prod_{t'=K-1}^{\tau+1}  ( 1 + \eta  \hat{\alpha}(\sigma(j)) N'  L)  \\
          &\leq \eta \hat{\alpha}(\sigma(j)) N K(1+\eta N L)^K.
    \end{align*}
    The last step is due to we choose $\eta$ such that $\eta N L \leq \frac{1}{K}$.

    \end{proof}
\end{proposition}

The following lemma establishes the bound regarding cumulative update between two epochs, namely, $\bv^{r+1} -\bv^{r}$. In particular, Lemma~\ref{lem:gradient bound} below shows that: (a)  in shuffling Local SGD, our update from $\bv^r$ to $\bv^{r+1}$ approximates performing $NK$ times of gradient descent with  $\hat{\alpha}(j) N \nabla f_{\sigma(j)}(\bv^r)$, namely, the bias is controlled, and (b) the update itself is bounded, and can be related to the norm of full gradient.

\begin{lemma}\label{lem:gradient bound}
During the dynamic of Algorithm~\ref{algorithm: KSGD one client}, the following statements hold true with probability at least $1-p$:

(a) 
\begin{align*}
      \left \| \sum_{j=1}^{N} \mQ_j \nabla f_{\sigma(j)}(\bv^r) - \eta NK\sum_{j=1}^{N} \hat{\alpha}(j) \nabla f_{\sigma(j)}(\bv^r) \right\|^2 &\leq 10\eta^2N^2 K^2 \left(\frac{e}{4R-e} \right)^2
        \left\|\nabla \Phi(\hat{\balpha},\bv^r) \right\|^2\\
        &\quad +   128\eta^2N^3 K^2 \left(\frac{e}{4R-e} \right)^2 G^2 \log(1/p).
    \end{align*}
 
(b) for any $N'$ such that $0 \leq N' < N$
    \begin{align*}
      \left \| \sum_{j=1}^{N'-1} \mQ_j \nabla f_{\sigma(j)}(\bv^r)  \right\| &\leq 3e\eta NK \pare{ \norm{\nabla \Phi(\hat{\balpha},\bv^r)} +  G\sqrt{8N\log(1/p)} },
    \end{align*}
    where
    \begin{align}
        \mQ_j :=   \left( \sum_{\tau=0}^{K-1} \prod_{t'=K-1}^{\tau+1} \left(\mI - \eta  \hat{\alpha}(\sigma(j)) N'  \mH_{t'}\right) \right) \eta  \hat{\alpha}(\sigma(j)) N. \label{eq: Q def}
    \end{align} 

\begin{proof}
We start with proving statement (a).
    Let $\mA_j = \frac{\mQ_j}{\hat{\alpha}(\sigma(j))}$ and $\mB_j = \hat{\alpha}(\sigma(j)) \nabla f_{\sigma(j)}(\bv^r)$, applying the identity of summation by parts yields:
    \begin{align*}
        \sum_{j=1}^{N} \mQ_j \nabla f_{\sigma(j)}(\bv^r) = \frac{\mQ_{N-1}}{\hat{\alpha}(\sigma(N-1))} \sum_{j=1}^{N} \hat{\alpha}(\sigma(j))\nabla f_{\sigma(j)}(\bv^r) - \sum_{n=1}^{N-1} \left( \frac{\mQ_{n+1}}{\hat{\alpha}(\sigma(n+1))}-\frac{\mQ_{n}}{\hat{\alpha}(\sigma(n))}\right) \sum_{j=1}^n \hat{\alpha}(\sigma(j)) \nabla f_{\sigma(j)}(\bv^r) 
    \end{align*}
     {\small   \begin{align*}
      &\left \| \sum_{j=1}^{N} \mQ_j \nabla f_{\sigma(j)}(\bv) - \eta NK\sum_{j=1}^{N} \hat{\alpha}(j) \nabla f_j(\bv) \right\|^2 \\
      &\leq   2\underbrace{\left\| \left(\frac{\mQ_{N-1}}{\hat{\alpha}(\sigma(N-1))} -\eta NK\mI\right)\sum_{j=1}^{N} \hat{\alpha}(\sigma(j))\nabla f_{\sigma(j)}(\bv^r)  \right\|^2}_{T_1} + 2\underbrace{\left\| \sum_{n=1}^{N-1} \left( \frac{\mQ_{n+1}}{\hat{\alpha}(\sigma(n+1))}-\frac{\mQ_{n}}{\hat{\alpha}(\sigma(n))}\right) \sum_{j=1}^n \hat{\alpha}(\sigma(j)) \nabla f_{\sigma(j)}(\bv^r)  \right\|^2}_{T_2}.
    \end{align*} }
    According to Proposition~\ref{prop:binomial}, we have:
    \begin{align*}
        \left\|\prod_{t'=K-1}^{\tau+1} \left(\mI - \eta  \hat{\alpha}(\sigma(j)) N  \mH_{t'}\right)  - \mI  \right\| \leq \sum_{m=1}^{K-2-\tau}  \left( \frac{e(K-2-\tau)}{m} \eta  \hat{\alpha}(\sigma(j)) N L\right)^m . 
    \end{align*}
    Since we choose $\eta \leq \frac{1}{4NKRL}$, we have:
     \begin{align}
        \left\|\prod_{t'=K-1}^{\tau+1} \left(\mI - \eta  \hat{\alpha}(\sigma(j)) N \mH_{t'}\right)  - \mI  \right\| \leq \sum_{m=1}^{K-2-\tau}  \left( \frac{e }{4Rm}  \right)^m  \leq \frac{e}{4R-e},\label{eq:grad bound 1}
    \end{align}
    where we use the fact that $\sum_{m=1}^{K-2-\tau}  \left( \frac{e }{4Rm}  \right)^m \leq \sum_{m=1}^{K-2-\tau}  \left( \frac{e }{4R}  \right)^m \leq  \frac{e}{4R}\frac{1}{1-e/4R}$. 
    Hence we know:
    \begin{align*}
        T_1 &\leq \left\| \left(\frac{\mQ_{N-1}}{\hat{\alpha}(\sigma(N-1))} -\eta NK\mI\right)\right\|^2
        \left\|\sum_{j=1}^{N} \hat{\alpha}(\sigma(j))\nabla f_{\sigma(j)}(\bv^r)  \right\|^2\\
        &\leq \left\| \left( \sum_{\tau=0}^{K-1} \prod_{t'=K-1}^{\tau+1} \left(\mI - \eta  \hat{\alpha}(\sigma(N-1)) N  \mH_{t'}\right) \right) \eta    N-\eta N K \mI\right\|^2
        \left\|\sum_{j=1}^{N} \hat{\alpha}(\sigma(j))\nabla f_{\sigma(j)}(\bv^r)  \right\|^2\\
        &\leq \eta^2N^2 K\sum_{\tau=0}^{K-1}\left\|  \prod_{t'=K-1}^{\tau+1} \left(\mI - \eta  \hat{\alpha}(\sigma(N-1)) N  \mH_{t'}\right)   -   \mI\right\|^2
        \left\|\sum_{j=1}^{N} \hat{\alpha}(\sigma(j))\nabla f_{\sigma(j)}(\bv^r)  \right\|^2\\
        &\leq \eta^2N^2 K^2 \left(\frac{e}{4R-e} \right)^2
        \left\|\sum_{j=1}^{N} \hat{\alpha}(\sigma(j))\nabla f_{\sigma(j)}(\bv^r)  \right\|^2.
    \end{align*}
    Thus we have:
    \begin{align*}
        T_1 \leq \eta^2N^2 K^2 \left(\frac{e}{4R-e} \right)^2
        \left\|\nabla \Phi(\hat{\balpha},\bv^r) \right\|^2.
    \end{align*}
  
    For $T_2$, we first examine the bound of $\frac{\mQ_{n+1}}{\hat{\alpha}(\sigma(n+1))}-\frac{\mQ_{n}}{\hat{\alpha}(\sigma(n))}$:
    {\small \begin{align*}
       &\left\| \frac{\mQ_{n+1}}{\hat{\alpha}(\sigma(n+1))}-\frac{\mQ_{n}}{\hat{\alpha}(\sigma(n))}\right\| = \left\| \left( \sum_{\tau=0}^{K-1} \prod_{t'=K-1}^{\tau+1} \left(\mI - \eta  \hat{\alpha}(\sigma(n+1)) N  \mH_{t'}\right) \right) \eta   N -  \left( \sum_{\tau=0}^{K-1} \prod_{t'=K-1}^{\tau+1} \left(\mI - \eta  \hat{\alpha}(\sigma(n)) N  \mH_{t'}\right) \right) \eta   N \right\|\\
       & \qquad \qquad \qquad \qquad \quad = \eta N \left\| \left( \sum_{\tau=0}^{K-1} \prod_{t'=K-1}^{\tau+1} \left(\mI - \eta  \hat{\alpha}(\sigma(n+1)) N  \mH_{t'}\right) \right)   -  \left( \sum_{\tau=0}^{K-1} \prod_{t'=K-1}^{\tau+1} \left(\mI - \eta  \hat{\alpha}(\sigma(n)) N  \mH_{t'}\right) \right)   \right\|\\
        & \qquad \qquad \qquad \qquad \quad = \eta N \left\| \left( \sum_{\tau=0}^{K-1} \prod_{t'=K-1}^{\tau+1} \left(\mI - \eta  \hat{\alpha}(\sigma(n+1)) N  \mH_{t'}\right)  \right)   -  \left( \sum_{\tau=0}^{K-1} \prod_{t'=K-1}^{\tau+1} \left(\mI - \eta  \hat{\alpha}(\sigma(n)) N  \mH_{t'}\right)  \right)   \right\|\\
       &\qquad \qquad \qquad \qquad \quad\leq  \eta N  \sum_{\tau=0}^{K-1} \left( \sum_{m=1}^{K-2-\tau} \left(\frac{e(K-2-\tau)}{m} \eta \hat{\alpha}(\sigma(n))NL\right)^m  + \sum_{m=1}^{K-2-\tau} \left(\frac{e(K-2-\tau)}{m} \eta \hat{\alpha}(\sigma(n+1))NL\right)^m 
       \right).
    \end{align*}}
    where we evoke Proposition~\ref{prop:binomial} at last step.
     Given that $\eta \leq \frac{1}{4NKRL}$ we have:
    \begin{align*}
         \left\| \frac{\mQ_{n+1}}{\hat{\alpha}(\sigma(n+1))}-\frac{\mQ_{n}}{\hat{\alpha}(\sigma(n))}\right\|  
       &\leq  \eta N  \sum_{\tau=0}^{K-1}\left( \sum_{m=1}^{K-2-\tau} \left(\frac{e}{4Rm} \hat{\alpha}(\sigma(n)) \right)^m  + \sum_{m=1}^{K-2-\tau} \left(\frac{e}{4Rm}\hat{\alpha}(\sigma(n+1)) L\right)^m 
       \right)\\
       & \leq \eta N K  \left(    \frac{\hat{\alpha}(\sigma(n))e}{4R-e}     +   \frac{\hat{\alpha}(\sigma(n+1))e}{4R-e}    
       \right).
    \end{align*}
    where we use the reasoining in (\ref{eq:grad bound 1}).
    Hence for $\sqrt{T_2}$:
    \begin{align*}
        \sqrt{T_2} &\leq \eta N K   \sum_{n=1}^{N-1} \left(  \frac{\hat{\alpha}(\sigma(n))e}{4R-e}     +   \frac{\hat{\alpha}(\sigma(n+1))e}{4R-e}    
       \right)\left\|\sum_{j=1}^n \hat{\alpha}(\sigma(j))  \nabla f_{\sigma(j)}(\bv^r)  \right\|\\
        & \leq \eta N K  \frac{ e}{4R-e} \sum_{n=1}^{N-1}\ \left(   \hat{\alpha}(\sigma(n))     +    \hat{\alpha}(\sigma(n+1))  
       \right)\left(G\sqrt{8n\log(1/p)} + \frac{n}{N} \left\|\   \nabla \Phi(\hat{\balpha}, \bv^r)\right\|\right)\\
        & \leq \eta N K  \frac{ 2e}{4R-e}  \left( G\sqrt{8N\log(1/p)} +  \left\|\nabla \Phi(\hat{\balpha}, \bv^r)\right\| \right).
    \end{align*}
    where at last step we evoke Lemma~\ref{lem:concentration of grad}.
So we can conclude $T_2\leq 2\eta^2 N^2 K^2 \left( \frac{ 2e}{4R-e} \right)^2 \left( G^2 {8N\log(1/p)} + \left\|\nabla \Phi(\hat{\balpha}, \bv^r)\right\|^2 \right)$. Putting the bounds of $T_1$ and $T_2$ together will conclude the proof for (a). 

Now we switch to proving (b). Once again by the summation of parts identity we have:
 {\small\begin{align*}
        \sum_{j=1}^{N' } \mQ_j \nabla f_{\sigma(j)}(\bv^r) = \frac{\mQ_{N' }}{\hat{\alpha}(\sigma(N' ))} \sum_{j=1}^{N' } \hat{\alpha}(\sigma(j))\nabla f_{\sigma(j)}(\bv^r) - \sum_{n=1}^{N'-1} \left( \frac{\mQ_{n+1}}{\hat{\alpha}(\sigma(n+1))}-\frac{\mQ_{n}}{\hat{\alpha}(\sigma(n))}\right) \sum_{j=1}^n \hat{\alpha}(\sigma(j)) \nabla f_{\sigma(j)}(\bv^r) .
    \end{align*}}
Taking the norm of both side yields:
\begin{align*}
    \left \|  \sum_{j=1}^{N' } \mQ_j \nabla f_{\sigma(j)}(\bv^r)\right\| &= \underbrace{\norm{\frac{\mQ_{N' }}{\hat{\alpha}(\sigma(N' ))} \sum_{j=1}^{N'} \hat{\alpha}(\sigma(j))\nabla f_{\sigma(j)}(\bv^r) }}_{B} \\
    & \quad + \underbrace{\norm{\sum_{n=1}^{N'-1} \left( \frac{\mQ_{n+1}}{\hat{\alpha}(\sigma(n+1))}-\frac{\mQ_{n}}{\hat{\alpha}(\sigma(n))}\right) \sum_{j=1}^n \hat{\alpha}(\sigma(j)) \nabla f_{\sigma(j)}(\bv^r)}}_{C}.
\end{align*}

Plugging our developed bound for $\norm{ {\mQ_{N' }} }$ and $\norm{\sum_{n=1}^{N'-1} \left( \frac{\mQ_{n+1}}{\hat{\alpha}(\sigma(n+1))}-\frac{\mQ_{n}}{\hat{\alpha}(\sigma(n))}\right) }$ yields:
\begin{align*}
    B &\leq \norm{\frac{\mQ_{N' }}{\hat{\alpha}(\sigma(N' ))}  }
    \norm{\sum_{j=1}^{N' } \hat{\alpha}(\sigma(j))\nabla f_{\sigma(j)}(\bv^r)}\\
    &\leq \eta   NK (1+\eta NL)^K \pare{G\sqrt{8N'\log(1/p)} + \frac{N'}{N}\norm{\nabla \Phi(\hat{\balpha},\bv^r)}}.
\end{align*}
where at last step we evoke Lemma~\ref{lem:concentration of grad}.
And for C, we use the similar reasoning:
\begin{align*}
    C &\leq\sum_{n=1}^{N'-1} \norm{ \left( \frac{\mQ_{n+1}}{\hat{\alpha}(\sigma(n+1))}-\frac{\mQ_{n}}{\hat{\alpha}(\sigma(n))}\right)}\norm{ \sum_{j=1}^n \hat{\alpha}(\sigma(j)) \nabla f_{\sigma(j)}(\bv^r)}\\
    &\leq\sum_{n=1}^{N'-1} \eta NK \frac{e}{4R-e}\pare{ \hat{\alpha}(\sigma(n+1))+\hat{\alpha}(\sigma(n))} \pare{G\sqrt{8n\log(1/p)} + \frac{n}{N}\norm{\nabla \Phi(\hat{\balpha},\bv^r)}}\\
    & \leq 2\eta NK \frac{e}{4R-e} \pare{G\sqrt{8N\log(1/p)} +  \norm{\nabla \Phi(\hat{\balpha},\bv^r)}}.
\end{align*}
 Putting these pieces together yields:
 \begin{align*}
      \left \|  \sum_{j=1}^{N'} \mQ_j \nabla f_{\sigma(j)}(\bv^r)\right\| \leq 3e\eta NK \pare{ \norm{\nabla \Phi(\hat{\balpha},\bv^r)} +  G\sqrt{8N\log(1/p)} }.
 \end{align*}
\end{proof}
\end{lemma}

\begin{lemma}\label{lem: gradient bound 2}
During the dynamic of Algorithm~\ref{algorithm: KSGD one client}, the following statements hold true with probability at least $1-p$:
    \begin{align*}
        &\left\| \sum_{n=1}^{N-1} \left(\prod_{j'=N}^{n+2}(\mI - \mQ_{j'}\mH_{j'}) \right) \mQ_{n+1}\mH_{n+1}\sum_{j=1}^n \mQ_j \nabla f_{\sigma(j)}(\bv^r) \right\|^2  \\
        &\leq  18e^6 \eta^4 N^4 K^4 L^4 \pare{ \norm{\nabla \Phi(\hat{\balpha},\bv^r)}^2 +  8G^2 {N\log(1/p)} }
    \end{align*}
    \begin{proof}
    We first apply Cauchy-Schwartz inequality:
        \begin{align*}
            &\left\| \sum_{n=1}^{N-1} \left(\prod_{j'=N}^{n+2}(\mI - \mQ_{j'}\mH_{j'}) \right) \mQ_{n+1}\mH_{n+1}\sum_{j=1}^n \mQ_j \nabla f_{\sigma(j)}(\bv^r) \right\| \\
            \leq & \sum_{n=1}^{N-1}\left\|  \left(\prod_{j'=N}^{n+2}(\mI - \mQ_{j'}\mH_{j'}) \right)\right\| \left\| \mQ_{n+1}\mH_{n+1}\right\| \left\| \sum_{j=1}^n \mQ_j \nabla f_{\sigma(j)}(\bv^r) \right\| \\
            \leq &    \left (1+\eta NK +\eta^2 N^2 K L\right)^{2N}  \eta  NLK(1+\eta N L)^K L\sum_{n=1}^{N-1} \hat{\alpha}(\sigma(n+1)) \left\| \sum_{j=1}^n \mQ_j \nabla f_{\sigma(j)}(\bv^r) \right\|\\ 
            \leq & e^2 \eta NKL^2 \sum_{n=1}^{N-1} \hat{\alpha}(\sigma(n+1)) \left\| \sum_{j=1}^n \mQ_j \nabla f_{\sigma(j)}(\bv^r) \right\|.
        \end{align*}
        We proceed by applying the bound from Lemma~\ref{lem:gradient bound} (b):
        \begin{align*}
            \left\| \sum_{j=1}^n \mQ_j \nabla f_{\sigma(j)}(\bv^r) \right\| &\leq 3e\eta NK \pare{ \norm{\nabla \Phi(\hat{\balpha},\bv^r)} +  G\sqrt{8N\log(1/p)} }.
        \end{align*}
        Therefore, it follows that:
        \begin{align*}
             &\left\| \sum_{n=1}^{N-1} \left(\prod_{j'=N}^{n+2}(\mI - \mQ_{j'}\mH_{j'}) \right) \mQ_{n+1}\mH_{n+1}\sum_{j=1}^n \mQ_j \nabla f_{\sigma(j)}(\bv^r) \right\| \\
             &\leq e^2 \eta NKL^2 \sum_{n=1}^{N-1} \hat{\alpha}(\sigma(n+1)) \cdot 3e\eta NK \pare{ \norm{\nabla \Phi(\hat{\balpha},\bv^r)} +  G\sqrt{8N\log(1/p)} }\\
             & \leq 3e^3 \eta^2 N^2 K^2 L^2 \pare{ \norm{\nabla \Phi(\hat{\balpha},\bv^r)} +  G\sqrt{8N\log(1/p)} }
        \end{align*}
 
    \end{proof}
\end{lemma}

\begin{lemma}[Noise bound] During the dynamic of Algorithm~\ref{algorithm: KSGD one client}, the following statement for gradient noises holds true with probability at least $1-p$:
\begin{align*}
     \E\left \|\sum_{j=1}^{N} \prod_{j'=N}^{j+1}(\mI - \mQ_{j'}\mH_{j'} ) \bdelta_j\right\| \leq \eta    N   K e^2 \delta, 
\end{align*}
where
\begin{align*}
     \bdelta_j &: = \sum_{\tau=0}^{K-1} \prod_{t'=t}^{\tau+1} \left(\mI - \hat{\alpha}(\sigma(j)) N \eta \mH_{t'}\right)   \eta  \hat{\alpha}(\sigma(j)) N \bdelta^t_{\sigma(j)}.
\end{align*}
\begin{proof}
    According to triangle  and Cauchy-Schwartz inequalities we have:
    \begin{align*}
       \left \|\sum_{j=1}^{N} \prod_{j'=N}^{j+1}(\mI - \mQ_{j'}\mH_{j'} ) \bdelta_j\right\| &\leq \sum_{j=1}^{N} \prod_{j'=N}^{j+1}\left \|(\mI - \mQ_{j'}\mH_{j'} ) \right\|\left \|\bdelta_j\right\| \\
       & \leq \sum_{j=1}^{N}  \left (1+ (\eta \hat{\alpha}(\sigma(j))N K(1+\eta NL)^K)L\right)^N  \left \|\bdelta_j\right\|\\
       & \leq  \sum_{j=1}^{N} \underbrace{ \left (1+ (\eta \hat{\alpha}(\sigma(j))N K(1+\eta NL)^K)L\right)^N}_{\leq e}  \cdot \eta\hat{\alpha}(\sigma(j)) N K\underbrace{ (1+\eta NL)^K}_{\leq e} \delta\\
       & \leq   \eta    N   K e^2 \delta. 
    \end{align*}
    \end{proof}
\end{lemma}

\subsection{Proof of Theorem~\ref{thm:2stagec perm}} 
\begin{proof}
For notational convenience, let us define
    \begin{align*} 
        \bg^r &:= \sum_{j=1}^{N} \prod_{j'=N}^{j+1}(\mI - \mQ_{j'}\mH_{j'}) \mQ_{j}{\nabla} f_{\sigma(j)}(\bv^r),\\
        \bdelta^r &:=\sum_{j=1}^{N} \prod_{j'=N}^{j+1}(\mI - \mQ_{j'}\mH_{j'} ) \bdelta_j.
    \end{align*}
    Then we recall the updating rule of $\bv$ (Lemma~\ref{lem:epoch update}):
    \begin{align*}
        \bv^{r+1} =\cP_{\cW} \left(\bv^r - \bg^r - \bdelta^r\right)
    \end{align*}
    Hence we have:
    \begin{align*}
        \E \norm{\bv^{r+1} - \bv^*(\hat{\balpha})}^2 &= \E \norm{\cP_{\cW} \left(\bv^r - \bg^r - \bdelta^r- \bv^*(\hat{\balpha})\right)}^2\\
        &\leq \E \norm{  \bv^r - \bg^r - \bdelta^r - \bv^*(\hat{\balpha})}^2\\
        &\leq \E \norm{  \bv^r - \bv^*(\hat{\balpha})   }^2 - 2\E \langle \bg^r,  \bv^r - \bv^*(\hat{\balpha})\rangle + \E\norm{\bg^r}^2 + \E\norm{\bdelta^r}^2\\
        &\leq \E \norm{  \bv^r - \bv^*(\hat{\balpha})   }^2 - 2\E \langle \eta NK \nabla \Phi(\hat{\balpha}, \bv^r),  \bv^r - \bv^*(\hat{\balpha})\rangle- 2\E \langle \bg^r-\eta NK\nabla \Phi(\hat{\balpha}, \bv^r),  \bv^r - \bv^*(\hat{\balpha})\rangle \\
        & \quad + \E\norm{\bg^r}^2 + \E\norm{\bdelta^r}^2.
    \end{align*}
Now, applying strongly convexity of $\Phi(\hat{\balpha},\cdot)$ and Cauchy-Schwartz inequality yields:
\begin{align*}
    \E \norm{\bv^{r+1} - \bv^*(\hat{\balpha})}^2  
        &\leq  (1-\mu\eta NK)\E \norm{  \bv^r - \bv^*(\hat{\balpha})   }^2 - \eta NK\E  [  \Phi(\hat{\balpha}, \bv^r)  -  \Phi(\hat{\balpha},\bv^*(\hat{\balpha}))]\\
        & \quad +\frac{1}{2}\pare{\frac{1}{\mu \eta NK}\E \| \bg^r-\eta NK\nabla \Phi(\hat{\balpha}, \bv^r)\|^2 +  {\mu \eta NK} \E\| \bv^r - \bv^*(\hat{\balpha})\|^2} \\
        & \quad + \E\norm{\bg^r}^2 + \E\norm{\bdelta^r}^2\\
        &\leq  (1-\frac{1}{2}\mu\eta NK)\E \norm{  \bv^r - \bv^*(\hat{\balpha})   }^2 - \eta NK\E  [  \Phi(\hat{\balpha}, \bv^r)  -  \Phi(\hat{\balpha},\bv^*(\hat{\balpha}))]\\
        & \quad +  \frac{1}{2\mu \eta NK}\E \| \bg^r-\eta NK\nabla \Phi(\hat{\balpha}, \bv^r)\|^2 \\
        & \quad + 2\E\norm{\bg^r - \eta NK\nabla \Phi(\hat{\balpha}, \bv^r)}^2 + 2\E\norm{ \eta NK\nabla \Phi(\hat{\balpha}, \bv^r)}^2  + \E\norm{\bdelta^r}^2.
\end{align*}
Since $\Phi(\hat{\balpha},\cdot)$ is $L$ smooth, we have: $\E\norm{  \nabla \Phi(\hat{\balpha}, \bv^r)}^2 \leq 2L\E  [  \Phi(\hat{\balpha}, \bv^r)  -  \Phi(\hat{\balpha},\bv^*(\hat{\balpha}))] $. Therefore, we have:
\begin{align}
    \E \norm{\bv^{r+1} - \bv^*(\hat{\balpha})}^2   
        &\leq  (1-\frac{1}{2}\mu\eta NK)\E \norm{  \bv^r - \bv^*(\hat{\balpha})   }^2 - (\eta NK - 4\eta^2N^2K^2L)\E  [  \Phi(\hat{\balpha}, \bv^r)  -  \Phi(\hat{\balpha},\bv^*(\hat{\balpha}))]\\
        & \quad + \pare{ \frac{1}{2\mu \eta NK}+ 2}\E \| \bg^r-\eta NK\nabla \Phi(\hat{\balpha}, \bv^r)\|^2    + \E\norm{\bdelta^r}^2. \label{eq:one iteration}
\end{align} 
Now, we examine the term $\| \bg^r-\eta NK\nabla \Phi(\hat{\balpha}, \bv^r)\|^2$. First according to summation by part (Lemma~\ref{lem: summation by parts}) by letting $\mA_j := \prod_{j'=N}^{j+1}(\mI - \mQ_{j'}\mH_{j'})$ and $\mB_j = \mQ_j \nabla f_{\sigma(j)}(\bv^r)$,  we have:
    \begin{align*}
    \bg^r&=\sum_{j=1}^{N} \prod_{j'=N}^{j+1}(\mI - \mQ_{j'}\mH_{j'}) \mQ_{j}\nabla f_{\sigma(j)}(\bv^r)\\
    & =  \sum_{j=1}^{N} \mA_j \mB_j  = \sum_{j=1}^{N} \mQ_j \nabla f_{\sigma(j)}(\bv^r) - \sum_{n=1}^{N-1} \left(\prod_{j'=N}^{n+2}(\mI - \mQ_{j'}\mH_{j'}) - \prod_{j'=N}^{n+1}(\mI - \mQ_{j'}\mH_{j'})\right) \sum_{j=1}^n \mQ_j \nabla f_{\sigma(j)}(\bv^r)\\
        &=\sum_{j=1}^{N} \mQ_j \nabla f_{\sigma(j)}(\bv^r) - \sum_{n=1}^{N-1} \left(\prod_{j'=N}^{n+2}(\mI - \mQ_{j'}\mH_{j'}) \right) \mQ_{n+1}\mH_{n+1}\sum_{j=1}^n \mQ_j \nabla f_{\sigma(j)}(\bv^r).
    \end{align*} 
    Hence we have:
{\small\begin{align*} 
&\| \bg^r-\eta NK\nabla \Phi(\hat{\balpha}, \bv^r)\|^2\\
        =&\left \| \eta NK\sum_{j=1}^{N} \hat{\alpha}(\sigma(j))  \nabla f_{\sigma(j)}(\bv^r) -  \left(\sum_{j=1}^{N} \mQ_j \nabla f_{\sigma(j)}(\bv^r) - \sum_{n=1}^{N-1} \left(\prod_{j'=N}^{n+2}(\mI - \mQ_{j'}\mH_{j'}) \right) \mQ_{n+1}\mH_{n+1}\sum_{j=1}^n \mQ_j \nabla f_{\sigma(j)}(\bv^r)\right) \right\|^2\\
         \stackrel{(1)}{=}&2\left \|\left(\eta NK \sum_{j=1}^{N} \hat{\alpha}(\sigma(j))  \nabla f_{\sigma(j)}(\bv^r) -  \sum_{j=1}^{N} \mQ_j \nabla f_{\sigma(j)}(\bv^r)\right) \right\|^2  +2\left\| \sum_{n=1}^{N-1} \left(\prod_{j'=N}^{n+2}(\mI - \mQ_{j'}\mH_{j'}) \right) \mQ_{n+1}\mH_{n+1}\sum_{j=1}^n \mQ_j \nabla f_{\sigma(j)}(\bv^r) \right\|^2\\
        \stackrel{(2)}{\leq} & \pare{20\eta^2N^2 K^2 \left(\frac{e}{4R-e} \right)^2+36e^6 \eta^4 N^4 K^4 L^4}
        \left\|\nabla \Phi(\hat{\balpha},\bv^r) \right\|^2 +   256\eta^2N^3 K^2 \left(\frac{e}{4R-e} \right)^2 G^2 \log(1/p) \\
        &   + 244e^6 \eta^4 N^4 K^4 L^4  G^2 {N\log(1/p)} \\
         \stackrel{(3)}{\leq}  &\pare{20\eta^2N^2 K^2 \left(\frac{e}{4R-e} \right)^2+36e^6 \eta^4 N^4 K^4 L^4}
        2L\pare{\Phi(\hat{\balpha}, \bv^r)  -  \Phi(\hat{\balpha},\bv^*(\hat{\balpha}))}     \\
        &   + \pare{244e^6 \eta^4 N^4 K^4 L^4+256\eta^2N^3 K^2 \left(\frac{e}{4R-e} \right)^2} G^2 {N\log(1/p)},
    \end{align*}}
where in (1) we apply Jensen's inequality, in (2) we plug in Lemma~\ref{lem:gradient bound} (a), and Lemma~\ref{lem: gradient bound 2}, and in (3) we use the $L$-smoothness of $\Phi$. 
    Plugging above bound back in (\ref{eq:one iteration}) yields:
{\small\begin{align*}
    &\E \norm{\bv^{r+1} - \bv^*(\hat{\balpha})}^2\\  
        &\leq  (1-\frac{1}{2}\mu\eta NK)\E \norm{  \bv^r - \bv^*(\hat{\balpha})   }^2+  \eta^2    N^2   K^2 e^4 \delta^2 \\
        &  - \underbrace{\pare{\eta NK - 4\eta^2N^2K^2L- \pare{ \frac{1}{2\mu \eta NK}+ 2}\pare{20\eta^2N^2 K^2 \left(\frac{e}{4R-e} \right)^2-36e^6 \eta^4 N^4 K^4 L^4}}}_{T_1}\E  [  \Phi(\hat{\balpha}, \bv^r)  -  \Phi(\hat{\balpha},\bv^*(\hat{\balpha}))] \\
        &   + \pare{ \frac{1}{2\mu \eta NK}+ 2}\pare{244e^6 \eta^4 N^4 K^4 L^4+256\eta^2N^3 K^2 \left(\frac{e}{4R-e} \right)^2} G^2 {N\log(1/p)}.
\end{align*} }
    Since we choose $\eta = \frac{4\log(\sqrt{NK}R)}{\mu NKR}$, and large enough epoch number:
    \begin{align*}
        R \geq \max \left\{ \pare{\frac{40}{\mu}+1}e, 16 \log(\sqrt{NK}R),64\kappa \log(\sqrt{NK}R) \right\},
    \end{align*}
    we know that $T_1 \leq 0$. We thus have:
   \begin{align*}
    &\E \norm{\bv^{r+1} - \bv^*(\hat{\balpha})}^2\\  
        &\leq  (1-\frac{1}{2}\mu\eta NK)\E \norm{  \bv^r - \bv^*(\hat{\balpha})   }^2+  \eta^2    N^2   K^2 e^4 \delta^2 \\
        &\quad + \pare{ \frac{1}{2\mu \eta NK}+ 2}\pare{244e^6 \eta^4 N^4 K^4 L^4+256\eta^2N^3 K^2 \left(\frac{e}{4R-e} \right)^2} G^2 {N\log(1/p)}
\end{align*} 
  Unrolling the recursion from $r=R$ to $0$:
    \begin{align*}
    &\E \norm{\bv^{R} - \bv^*(\hat{\balpha})}^2\\  
        &\leq  (1-\frac{1}{2}\mu\eta NK)^R\E \norm{  \bv^0 - \bv^*(\hat{\balpha})   }^2+  \frac{2}{\mu}\eta  N    K  e^4 \delta^2  \\
        &   \quad + \frac{1}{\mu}\pare{ \frac{1}{2\mu \eta NK}+ 2}\pare{488e^6 \eta^3 N^3 K^3 L^4+512\eta N^2 K  \left(\frac{e}{4R-e} \right)^2} G^2 {N\log(1/p)}.
\end{align*}  
    Plugging in our choice of $\eta$ will conclude the proof:
    \begin{align*}
         &\E \norm{\bv^{R} - \bv^*(\hat{\balpha})}^2 \leq  \tilde{O}\pare{ \frac{ \E \norm{  \bv^0 - \bv^*(\hat{\balpha})   }^2}{NKR^2}+  \frac{\delta^2}{\mu^2  R}   +   \pare{ \frac{L^4+N}{\mu^4R^2}  } G^2 {N\log(1/p)}}.
    \end{align*}
    Finally, according to Lemma~\ref{lem:opt gap} we can complete the proof: 
      \begin{align*}
        \Phi( {\balpha}^*_i, \hat{\bv}_i) - \Phi( {\balpha}^*_i,  {\bv}^*_i) 
        &\leq 2L\norm{{\bv}^R_i -   {\bv}^*(\hat{\balpha}_i)}^2 +  \pare{2\kappa^2_{\Phi} L + \frac{4NG^2}{ L} } \| \hat{\balpha}_i - \balpha^* \|^2\\ 
         & \leq  \tilde{O}\pare{ \frac{L D^2}{NKR^2}+  \frac{L\delta^2}{\mu^2  R}   +   \pare{ \frac{L^4+N}{\mu^4R^2}  } LG^2 {N\log(1/p)}} \\
         & \quad +  \kappa^2_\Phi L\tilde{O}\pare{ \exp\pare{-\frac{T_{\balpha}}{\kappa_g}}+  \kappa_g^2 \bar{\zeta}_i(\bw^*)   L^2\pare{\frac{D^2 }{RK} +     \frac{\kappa  \zeta^2}{\mu^2 R^2  } +   \frac{  \delta^2 }{\mu^2 N RK}  }},
    \end{align*}
    where we plug in the convergence result from Theorem~\ref{thm: 2stage alpha} at last step.
\end{proof}

\section{Proof of Convergence of Single Loop Algorithm}\label{app:sec:proof:single}
In this section, we turn to presenting the proof of single loop PERM algorithm (Algorithm~\ref{algorithm: Single PERM}) where the learning of mixing parameters  and personalized models are coupled. Compared to Algorithm~\ref{algorithm: KSGD one client}, here during the optimization of model, the mixing parameters are also being updated. As a result, we need  to decouple the two updates which makes the analysis more involved. We begin with some technical lemmas that support the proof of main result.

\subsection{Technical Lemmas}

\begin{proposition}[Basic Properties of SGD on Smooth Strongly Convex Function]\label{prop: SGD convergence}
    Let $\bw^t$ to be the $t$th iterate of minibatch SGD on smooth and strongly convex function $F$, with minibatch size $M$ and learning rate $\gamma$. Also assume the variance is bounded by $\delta$. Then the following statements hold true after $T$ iterations of SGD:
    \begin{align}
        \mathbb{E}\|\nabla F(\bw^T)\|^2 \leq   2L\left (1-\mu \gamma\right)^T (F(\bw^0) - F(\bw^*)) +   \frac{2\gamma \kappa \delta^2}{ M} \label{eq: SGD convergence 1}  
    \end{align}
     \begin{align}
        \mathbb{E}\|\bw^{T+1} - \bw^{T}\|^2   \leq 2\gamma^2L\left (1-\mu \gamma\right)^T (F(\bw^0) - F(\bw^*)) +   \frac{2\gamma^3 \kappa \delta^2}{ M} + \frac{\gamma^2\delta^2}{M}\label{eq: SGD convergence 2}
    \end{align}
    \begin{align}
        \mathbb{E}\|\bw^{T} - \bw^{*}\|^2 \leq \frac{2}{\mu}\left (1-\mu \gamma\right)^T (F(\bw^0) - F(\bw^*)) + 2\gamma \frac{\delta^2}{\mu^2 M}.\label{eq: SGD convergence 3}
    \end{align}
\end{proposition}
 
\begin{lemma}[Bounded iterates difference of $\balpha$]\label{lem: iterate diff alpha}
Let $\{\balpha^r_i\}$ be iterates generated by Algorithm~\ref{algorithm: Single PERM}, then under conditions of Theorem~\ref{thm:single loop}, the following statement holds:
\begin{align*}
    \|\balpha^r_i - \balpha^{r-1}_i\|^2 &\leq  6\pare{1-\frac{1}{\kappa_g}}^{T_{\balpha}}   +O\pare{\kappa_g^2  L^2  \bar\zeta_i(\bw^*)   }  \pare{ \gamma^2L\left (1-\mu \gamma\right)^r (F(\bw^0) - F(\bw^*)) +   \frac{ \gamma^3 \kappa \delta^2}{ M} + \frac{\gamma^2\delta^2}{M}}
\end{align*} 
\begin{proof}
Define
\begin{align*}
 \bz^r &= \left[ \norm{\nabla f_i( \bw^r) -\nabla f_1(\bw^r) }^2,..., \norm{\nabla f_i( \bw^r) -\nabla f_N( \bw^r) }^2\right].
\end{align*}

According to updating rule of $\balpha$ in Algorithm~\ref{algorithm: Single PERM} and Lemma~\ref{lem: lipschitz alpha} we have:
\begin{align*}
    \|\balpha^r_i - \balpha^{r-1}_i\|^2  &\leq 3\|\balpha^r_i - \balpha^*_{g_i} (\bw^{r})\|^2+3\norm{\balpha^*_{g_i} (\bw^{r-1}) - \balpha^*_{g_i} (\bw^{r})}^2 +3\|\balpha^*_{g_i} (\bw^{r-1})  - \balpha^{r-1}_i\|^2 \\
    &\leq 6(1-\mu_g \eta_{\balpha})^{T_{\balpha}}    +3\norm{\balpha^*_{g_i} (\bw^{r-1}) - \balpha^*_{g_i} (\bw^{r})}^2 \\
    & \leq 6(1-\mu_g \eta_{\balpha})^{T_{\balpha}}    +3\kappa_g^2 \norm{ \bz^{r-1}  - \bz^{r} }^2 \\
    & \leq 6(1-\mu_g \eta_{\balpha})^{T_{\balpha}}    +3\kappa_g^2 \sum_{j=1}^N { \norm{\nabla f_i( \bw^r) -\nabla f_j(\bw^r) +\nabla f_i( \bw^{r-1}) -\nabla f_j(\bw^{r-1}) }^2 } 4L^2\norm{ \bw^r-\bw^{r-1} }^2\\
\end{align*}
 where the third inequality follows from~\eqref{eq:lipschitz alpha 1}.
Since $\norm{\nabla f_i( \bw^r) -\nabla f_j(\bw^r)} \leq \norm{\nabla f_i( \bw^*) -\nabla f_j(\bw^*)} + 2L\norm{  \bw^r - \bw^*}$, we can conclude that 
\begin{align*}
    \|\balpha^r_i - \balpha^{r-1}_i\|^2  
    & \leq 6(1-\mu_g \eta_{\balpha})^{T_{\balpha}}    \\
    &+12L^2\kappa_g^2 \sum_{j=1}^N \pare{ 8\norm{\nabla f_i( \bw^*) -\nabla f_j(\bw^*)}^2 + 8L^2\norm{  \bw^r - \bw^*}^2+8L^2\norm{  \bw^{r-1} - \bw^*}^2}\norm{ \bw^r-\bw^{r-1} }^2\\
    & \leq 6\pare{1-\frac{1}{\kappa_g}}^{T_{\balpha}}   +O\pare{\kappa_g^2  L^2  \bar\zeta_i(\bw^*) \norm{ \bw^r-\bw^{r-1} }^2 } \\
      & \leq 6\pare{1-\frac{1}{\kappa_g}}^{T_{\balpha}}   +O\pare{\kappa_g^2  L^2  \bar\zeta_i(\bw^*)   }  \pare{ \gamma^2L\left (1-\mu \gamma\right)^r (F(\bw^0) - F(\bw^*)) +   \frac{ \gamma^3 \kappa \delta^2}{ M} + \frac{\gamma^2\delta^2}{M}}\\
\end{align*}
where at last step we plug in Proposition~\ref{prop: SGD convergence} (\ref{eq: SGD convergence 2}).
    \end{proof}

\end{lemma}

\begin{lemma}[Convergence of $\balpha$]\label{lem:conv alpha perm}
Let $\{\hat\balpha_i\}_{i=1}^{N}$ be the mixing parameters generated by Algorithm~\ref{algorithm: Single PERM}. Then under the conditions of Theorem~\ref{thm:single loop}, the following statement holds:
\begin{align*}
    \norm{\hat{\balpha}_i - \balpha^* }^2 \leq  2 (1-\frac{1}{\kappa_g})^{T_{\balpha}} + O\pare{\kappa_g^2  \bar{\zeta}_i(\bw^*)   L^2\frac{2}{\mu}\left (1-\mu \gamma\right)^T   + 2\gamma \frac{\delta^2}{\mu^2 M}}, i \in [N]
\end{align*}
\begin{proof}
We notice the following decomposition:
\begin{align*}
    \norm{\hat{\balpha}_i - \balpha^* }^2& =   \norm{ {\balpha}^R_i - \balpha^*_g(\bw^*) }^2  \\
    &\leq 2\norm{ {\balpha}^R_i - \balpha^*_g(\bw^R) }^2   +  2\norm{ \balpha^*_{g_i}(\bw^R)  - \balpha^*_{g_i}(\bw^*) }^2 \\
   & \leq 2 (1-\frac{1}{\kappa_g})^{T_{\balpha}} + O\pare{\kappa_g^2\pare{ \bar{\zeta}_i(\bw^*) +  N L^2\norm{  \bw^R - \bw^*}^2} 4L\norm{ \bw^R-\bw^* }^2 }\\
   &\leq 2 (1-\frac{1}{\kappa_g})^{T_{\balpha}} + O\pare{\kappa_g^2  \bar{\zeta}_i(\bw^*)   L^2\frac{2}{\mu}\left (1-\mu \gamma\right)^T   + 2\gamma \frac{\delta^2}{\mu^2 M}},
\end{align*}
where in the second inequality we apply Lemma~\ref{lem: lipschitz alpha}, and in the third inequality we use Proposition~\ref{prop: SGD convergence} (\ref{eq: SGD convergence 3}).
\end{proof}

\end{lemma}

 \subsection{Proof of Theorem~\ref{thm:single loop}}

\begin{proof}
According to Lemma~\ref{lem:opt gap}, we have:
 
   \begin{align*}
        \Phi( {\balpha}^*_i, \hat{\bv}_i) - \Phi( {\balpha}^*_i,  {\bv}^*_i) \leq  2{L} \norm{\bv^R_i  -  \bv^*(\hat\balpha_i)}^2 + \pare{2\kappa^2_{\Phi} L + \frac{4NG^2}{ L} }\norm{ \hat\balpha_i  -   \balpha^*_i  }^2. 
    \end{align*} 
We first examine the convergence of $\norm{{\bv}^R_i -  \bv^* ( \hat{\balpha}_i)}^2$. Applying Cauchy-Schwartz inequality yields:
\begin{align}
        \|\bv^{r+1} - \bv^*(\balpha^{r+1})\|^2 &\leq \left(1+ \frac{1}{4a-2} \right ) \|\bv^{r+1} - \bv^*(\balpha^{r})\|^2 + \left(1+  4a-2 \right ) \|\bv^*(\balpha^{r+1}) - \bv^*(\balpha^{r})\|^2 \nonumber \\
        &\leq \left(1+ \frac{1}{4a-2} \right ) \|\bv^{r+1} - \bv^*(\balpha^{r})\|^2 + \left(1+ 4a-2 \right ) \kappa_{\Phi}^2\| \balpha^{r+1} -  \balpha^{r}\|^2\label{eq:pf thm3 1}
\end{align}
where $a = \frac{1}{\mu \eta NK }$, and last step is due to that $\bv^*(\balpha)$ is $\kappa_{\Phi}:= \frac{\sqrt{N}G}{\mu}$ Lipschitz, as proven in Lemma~\ref{lem:opt gap}
. Similar to the proof of Theorem~\ref{thm:2stagec perm}, we first define
    \begin{align*} 
        \bg^r &:= \sum_{j=1}^{N} \prod_{j'=N-1}^{j+1}(\mI - \mQ_{j'}\mH_{j'}) \mQ_{j}{\nabla} f_{\sigma(j)}(\bv^r),\\
        \bdelta^r &:=\sum_{j=1}^{N} \prod_{j'=N-1}^{j+1}(\mI - \mQ_{j'}\mH_{j'} ) \bdelta_j.
    \end{align*}
    Then we recall the updating rule of $\bv$:
    \begin{align*}
        \bv^{r+1} =\cP_{\cW} \left(\bv^r - \bg^r - \bdelta^r\right).
    \end{align*}
    Hence we have:
    \begin{align*}
        \E \norm{\bv^{r+1} - \bv^*( {\balpha}^r)}^2 &= \E \norm{\cP_{\cW} \left(\bv^r - \bg^r - \bdelta^r- \bv^*({\balpha}^r)\right)}^2\\
        &\leq \E \norm{  \bv^r - \bg^r - \bdelta^r - \bv^*({\balpha}^r)}^2\\
        &\leq \E \norm{  \bv^r - \bv^*({\balpha}^r)   }^2 - 2\E \langle \bg^r,  \bv^r - \bv^*({\balpha}^r)\rangle + \E\norm{\bg^r}^2 + \E\norm{\bdelta^r}^2\\
        &\leq \E \norm{  \bv^r - \bv^*({\balpha}^r)   }^2 - 2\E \langle \eta NK \nabla \Phi({\balpha}^r, \bv^r),  \bv^r - \bv^*({\balpha}^r)\rangle\\
        &\quad - 2\E \langle \bg^r-\eta NK\nabla \Phi({\balpha}^r, \bv^r),  \bv^r - \bv^*({\balpha}^r)\rangle   + \E\norm{\bg^r}^2 + \E\norm{\bdelta^r}^2.
    \end{align*}
Now, applying strongly convexity of $\Phi({\balpha}^r,\cdot)$ and Cauchy-Schwartz inequality yields:
\begin{align*}
    \E \norm{\bv^{r+1} - \bv^*({\balpha}^r)}^2  
        &\leq  (1-\mu\eta NK)\E \norm{  \bv^r - \bv^*({\balpha}^r)   }^2 - \eta NK\E  [  \Phi({\balpha}^r, \bv^r)  -  \Phi({\balpha}^r,\bv^*(\hat{\balpha}))]\\
        & \quad +\frac{1}{2}\pare{\frac{1}{\mu \eta NK}\E \| \bg^r-\eta NK\nabla \Phi(\hat{\balpha}, \bv^r)\|^2 +  {\mu \eta NK} \E\| \bv^r - \bv^*(\hat{\balpha})\|^2} \\
        & \quad + \E\norm{\bg^r}^2 + \E\norm{\bdelta^r}^2\\
        &\leq \pare{1-\frac{1}{2}\mu\eta NK}\E \norm{  \bv^r - \bv^*({\balpha}^r)   }^2 - \eta NK\E  [  \Phi({\balpha}^r, \bv^r)  -  \Phi({\balpha}^r,\bv^*(\hat{\balpha}))]\\
        & \quad +  \frac{1}{2\mu \eta NK}\E \| \bg^r-\eta NK\nabla \Phi({\balpha}^r, \bv^r)\|^2 \\
        & \quad + 2\E\norm{\bg^r - \eta NK\nabla \Phi({\balpha}^r, \bv^r)}^2 + 2\E\norm{ \eta NK\nabla \Phi({\balpha}^r, \bv^r)}^2  + \E\norm{\bdelta^r}^2.
\end{align*}
where in the first inequality we applied Cauchy-Schwartz inequality and strongly convexity.
Since $\Phi({\balpha}^r,\cdot)$ is $L$ smooth, we have: $\E\norm{  \nabla \Phi({\balpha}^r, \bv^r)}^2 \leq 2L\E  [  \Phi({\balpha}^r, \bv^r)  -  \Phi({\balpha}^r,\bv^*({\balpha}^r))] $. Therefore, it follows that:
\begin{align}
    \E \norm{\bv^{r+1} - \bv^*(\hat{\balpha})}^2   
        &\leq  \pare{1-\frac{1}{2}\mu\eta NK}\E \norm{  \bv^r - \bv^*({\balpha}^r)   }^2 - (\eta NK - 4\eta^2N^2K^2L)\E  [  \Phi({\balpha}^r, \bv^r)  -  \Phi({\balpha}^r,\bv^*({\balpha}^r))]\nonumber\\
        & \quad + \pare{ \frac{1}{2\mu \eta NK}+ 2}\E \| \bg^r-\eta NK\nabla \Phi({\balpha}^r, \bv^r)\|^2    + \E\norm{\bdelta^r}^2 \label{eq:one iteration}
\end{align} 
Now, we examine the term $\| \bg^r-\eta NK\nabla \Phi({\balpha}^r, \bv^r)\|^2$ in the right hand side of abovee inequality. First according to summation by part (Lemma~\ref{lem: summation by parts}): we let $\mA_j := \prod_{j'=N-1}^{j+1}(\mI - \mQ_{j'}\mH_{j'})$ and $\mB_j = \mQ_j \nabla f_{\sigma(j)}(\bv^r)$, then we have:
    \begin{align*}
    \bg^r&=\sum_{j=1}^{N} \prod_{j'=N}^{j+1}(\mI - \mQ_{j'}\mH_{j'}) \mQ_{j}\nabla f_{\sigma(j)}(\bv^r)\\
    & =  \sum_{j=1}^{N} \mA_j \mB_j  = \sum_{j=1}^{N} \mQ_j \nabla f_{\sigma(j)}(\bv^r) - \sum_{n=1}^{N-1} \left(\prod_{j'=N}^{n+2}(\mI - \mQ_{j'}\mH_{j'}) - \prod_{j'=N}^{n+1}(\mI - \mQ_{j'}\mH_{j'})\right) \sum_{j=1}^n \mQ_j \nabla f_{\sigma(j)}(\bv^r)\\
        &=\sum_{j=1}^{N} \mQ_j \nabla f_{\sigma(j)}(\bv^r) - \sum_{n=1}^{N-1} \left(\prod_{j'=N}^{n+2}(\mI - \mQ_{j'}\mH_{j'}) \right) \mQ_{n+1}\mH_{n+1}\sum_{j=1}^n \mQ_j \nabla f_{\sigma(j)}(\bv^r).
    \end{align*} 
    Hence we have:
\begin{align*} 
&\| \bg^r-\eta NK\nabla \Phi(\hat{\balpha}, \bv^r)\|^2\\
&=\left\|  \eta NK \nabla \Phi(\hat{\balpha}, \bv^r)-  \sum_{j=1}^{N} \prod_{j'=N-1}^{j+1}(\mI - \mQ_{j'}\mH_{j'}) \mQ_{j}\nabla f_{\sigma(j)}(\bv^r)\right\|^2\\
        &=\left \| \eta NK\sum_{j=1}^{N} \hat{\alpha}(\sigma(j))  \nabla f_{\sigma(j)}(\bv^r) -  \left(\sum_{j=1}^{N} \mQ_j \nabla f_{\sigma(j)}(\bv^r) - \sum_{n=1}^{N-1} \left(\prod_{j'=N}^{n+2}(\mI - \mQ_{j'}\mH_{j'}) \right) \mQ_{n+1}\mH_{n+1}\sum_{j=0}^n \mQ_j \nabla f_{\sigma(j)}(\bv^r)\right) \right\|^2\\
        &   \stackrel{(1)}{\leq} 2\left \|\left(\eta NK \sum_{j=1}^{N} \hat{\alpha}(\sigma(j))  \nabla f_{\sigma(j)}(\bv^r) -  \sum_{j=1}^{N} \mQ_j \nabla f_{\sigma(j)}(\bv^r)\right) \right\|^2\\
        &\quad +2\left\| \sum_{n=1}^{N-1} \left(\prod_{j'=N}^{n+2}(\mI - \mQ_{j'}\mH_{j'}) \right) \mQ_{n+1}\mH_{n+1}\sum_{j=1}^n \mQ_j \nabla f_{\sigma(j)}(\bv^r) \right\|^2\\
        &  \stackrel{(2)}{\leq}\pare{20\eta^2N^2 K^2 \left(\frac{e}{4R-e} \right)^2+36e^6 \eta^4 N^4 K^4 L^4}
        \left\|\nabla \Phi(\hat{\balpha},\bv^r) \right\|^2 +   256\eta^2N^3 K^2 \left(\frac{e}{4R-e} \right)^2 G^2 \log(1/p) \\
        & \quad + 244e^6 \eta^4 N^4 K^4 L^4  G^2 {N\log(1/p)} \\
        &  \stackrel{(3)}{\leq} \pare{20\eta^2N^2 K^2 \left(\frac{e}{4R-e} \right)^2+36e^6 \eta^4 N^4 K^4 L^4}
        2L\pare{\Phi(\hat{\balpha}, \bv^r)  -  \Phi(\hat{\balpha},\bv^*(\hat{\balpha}))}     \\
        & \quad + \pare{244e^6 \eta^4 N^4 K^4 L^4+256\eta^2N^3 K^2 \left(\frac{e}{4R-e} \right)^2} G^2 {N\log(1/p)} \\
    \end{align*}
 
where in (1) we apply Jensen's inequality, in (2) we plug in Lemma~\ref{lem:gradient bound} (a), and Lemma~\ref{lem: gradient bound 2}, and in (3) we use the $L$-smoothness of $\Phi$.
    Plugging above bound back in (\ref{eq:one iteration}) yields:
\begin{align*}
    &\E \norm{\bv^{r+1} - \bv^*({\balpha}^r)}^2\\  
        &\leq  (1-\frac{1}{2}\mu\eta NK)\E \norm{  \bv^r - \bv^*({\balpha}^r)   }^2+  \eta^2    N^2   K^2 e^4 \delta^2 \\
        &\quad   - \pare{\eta NK - 4\eta^2N^2K^2L- \pare{ \frac{1}{2\mu \eta NK}+ 2}\pare{20\eta^2N^2 K^2 \left(\frac{e}{4R-e} \right)^2-36e^6 \eta^4 N^4 K^4 L^4}}\\
        &\quad  \times \E  [  \Phi({\balpha}^r, \bv^r)  -  \Phi({\balpha}^r,\bv^*({\balpha}^r))] \\
        &   \quad + \pare{ \frac{1}{2\mu \eta NK}+ 2}\pare{244e^6 \eta^4 N^4 K^4 L^4+256\eta^2N^3 K^2 \left(\frac{e}{4R-e} \right)^2} G^2 {N\log(1/p)}.
\end{align*} 
    Since we choose $\eta = \frac{4\log(\sqrt{NK}R)}{\mu NKR}$, and 
    \begin{align*}
        R \geq \max \left\{ \frac{3}{8}e, \sqrt[3]{\frac{64\kappa^2 \log (\sqrt{NK} R) e^6  }{9\mu}} \right\},
    \end{align*}
     hence we have:
   \begin{align*}
    &\E \norm{\bv^{r+1} - \bv^*({\balpha}^r)}^2\\  
        &\leq  (1-\frac{1}{2}\mu\eta NK)\E \norm{  \bv^r - \bv^*({\balpha}^r)   }^2+  \eta^2    N^2   K^2 e^4 \delta^2  -  \frac{1}{2}\eta NK \underbrace{\E  [  \Phi(\hat{\balpha}, \bv^r)  -  \Phi(\hat{\balpha},\bv^*(\hat{\balpha}))]}_{\geq 0} \\
        &   \quad + \pare{ \frac{1}{2\mu \eta NK}+ 2}\pare{244e^6 \eta^4 N^4 K^4 L^4+256\eta^2N^3 K^2 \left(\frac{e}{4R-e} \right)^2} G^2 {N\log(1/p)}\\
         &\leq  \pare{1-\frac{1}{2}\mu\eta NK} \E \norm{  \bv^r - \bv^*(\hat{\balpha})   }^2+  \eta^2    N^2   K^2 e^4 \delta^2   \\
        &   \quad + \pare{ \frac{1}{2\mu \eta NK}+ 2}\pare{244e^6 \eta^4 N^4 K^4 L^4+256\eta^2N^3 K^2 \left(\frac{e}{4R-e} \right)^2} G^2 {N\log(1/p)}.
\end{align*} 
Putting above inequality back to (\ref{eq:pf thm3 1}) yields:
\begin{align*}
        \|\bv^{r+1} - \bv^*(\balpha^{r+1})\|^2 &\leq \left(1- \frac{1}{4a} \right ) \|\bv^{r} - \bv^*(\balpha^{r})\|^2+ 2  \eta^2    N^2   K^2 e^4 \delta^2 +   \left(1+ 4a-2 \right ) \kappa_{\Phi}^2\| \balpha^{r+1} -  \balpha^{r}\|^2\\
        &+ 2\pare{ \frac{1}{2\mu \eta NK}+ 2}\pare{244e^6 \eta^4 N^4 K^4 L^4+256\eta^2N^3 K^2 \left(\frac{e}{4R-e} \right)^2} G^2 {N\log(1/p)} \\
        &\leq \left(1- \frac{1}{4a} \right ) \|\bv^{r} - \bv^*(\balpha^{r})\|^2 + 2\eta^2    N^2   K^2 e^4 \delta^2    \\
        &+ 2\pare{ \frac{1}{2\mu \eta NK}+ 2}\pare{244e^6 \eta^4 N^4 K^4 L^4+256\eta^2N^3 K^2 \left(\frac{e}{4R-e} \right)^2} G^2 {N\log(1/p)} \\
        &+ O\pare{ \frac{\kappa_\Phi^2}{\mu \eta NK}   \left( \pare{1-\frac{1}{\kappa_g}}^{T_{\balpha}}   +  \kappa_g^2  L^2  \bar\zeta_i(\bw^*)   \pare{ \gamma^2L\left (1-\mu \gamma\right)^r DG +   \frac{ \gamma^3 \kappa \delta^2}{ M} + \frac{\gamma^2\delta^2}{M}}\right)}
 \end{align*}
 where at second inequality we plug in Lemma~\ref{lem: iterate diff alpha}.
Unrolling the recursion from $r=R$ to $0$, and plugging in $\eta = \frac{4\log(NK R^3)}{\mu NK R}$ yields:
\begin{align*}
       & \|\bv^{R} - \bv^*(\balpha^{R})\|^2  \\
        &\leq \left(1- \frac{1}{4}\mu\eta NK \right )^R \|\bv^{0} - \bv^*(\balpha^{0})\|^2   + \frac{1}{\mu}\eta     N    K  e^4 \delta^2   \\
        &+ 8\frac{1}{\mu}\pare{ \frac{1}{2\mu \eta NK}+ 2}\pare{244e^6 \eta^3  N^3 K^3 L^4+256\eta N^2 K  \left(\frac{e}{4R-e} \right)^2} G^2 {N\log(1/p)   } \\
        &+ O\pare{ \frac{\kappa_\Phi^2}{\mu \eta NK}  \sum_{r=0}^R\left(1- \frac{1}{4a} \right )^{R-r} \left( \pare{1-\frac{1}{\kappa_g}}^{T_{\balpha}}   +  \kappa_g^2  L^2  \bar\zeta_i(\bw^*)   \pare{ \gamma^2L\left (1-\mu \gamma\right)^r DG +   \frac{ \gamma^3 \kappa \delta^2}{ M} + \frac{\gamma^2\delta^2}{M}}\right)}\\
         &\leq O\pare{ \frac{\|\bv^{0} - \bv^*(\balpha^{0})\|^2}{NKR^3} }   + \tilde{O} \pare{   \pare{  \frac{\kappa^4}{  R^2}  + \frac{N}{\mu^2 R^2}   } G^2 {N\log(1/p)}+ \frac{ \delta^2}{\mu R}  } \\
        &+ \tilde{O}\pare{ R\kappa_\Phi^2 \sum_{r=0}^R\left(1- \frac{\log(NKR^3)}{R}   \right )^{R-r} \left( \pare{1-\frac{1}{\kappa_g}}^{T_{\balpha}}   +  \kappa_g^2  L^2  \bar\zeta_i(\bw^*)   \pare{ \gamma^2L\left (1-\mu \gamma\right)^r DG +   \frac{ \gamma^3 \kappa \delta^2}{ M} + \frac{\gamma^2\delta^2}{M}}\right)}
\end{align*}

Plugging in $\gamma = \frac{ \log(NK R^3)}{\mu  R}$ yields:
\begin{align*}
      \|\bv^{R} - \bv^*(\balpha^{R})\|^2  &\leq O\pare{ \frac{\|\bv^{0} - \bv^*(\balpha^{0})\|^2}{NKR^3} }   + \tilde{O} \pare{   \pare{  \frac{\kappa^4}{  R^2}  + \frac{N}{\mu^2 R^2}   } G^2 {N\log(1/p)} + \frac{ \delta^2}{\mu R}     } \\
        &+ \tilde{O}\left( \kappa_\Phi^2 {R }   \left( \gamma^2L^2\sum_{r=0}^R\left(1- \frac{\log(NKR^3)}{R}   \right )^{R } \kappa_g^2  L^2  \bar\zeta_i(\bw^*)   DG \right.\right. \\ 
        &\qquad \qquad \left. \left.+ \sum_{r=0}^R\left(1- \frac{\log(NKR^3)}{R}   \right )^{R-r}\pare{\pare{1-\frac{1}{\kappa_g}}^{T_{\balpha}} +\frac{\kappa_g^2  L^2  \bar\zeta_i(\bw^*)  \gamma^2\delta^2}{M}   } \right)\right)\\
         &\leq O\pare{ \frac{D^2}{NKR^3} }   + \tilde{O} \pare{   \pare{  \frac{\kappa^4}{  R^2}  + \frac{N}{\mu^2 R^2}   } G^2 {N\log(1/p)} + \frac{ \delta^2}{\mu R}     } \\
        &+ \tilde{O}\pare{  \frac{\kappa_\Phi^2\kappa^2\kappa_g^2  L^2  \bar\zeta_i(\bw^*)  DG}{R } + \kappa_\Phi^2 R^2\pare{\pare{1-\frac{1}{\kappa_g}}^{T_{\balpha}}    +  \frac{ \kappa_g^2  \kappa^2  \bar\zeta_i(\bw^*)   \delta^2}{\mu^2 M R^2}    } }
\end{align*}
Since $\hat{\bv}_i = \bv^R$ and $\hat{\balpha}_i = \balpha^R$, we have the convergence of $\norm{\hat{\bv}_i -  \bv^* ( \hat{\balpha}_i)}^2 $. Plugging this convergence rate together with the convergence of $\| \hat{\balpha}_i - \balpha^* \|^2$ from Lemma~\ref{lem:conv alpha perm}:
\begin{align*}
    \norm{\hat{\balpha}_i - \balpha^* }^2& \leq   O\pare{   2 (1-\frac{1}{\kappa_g})^{T_{\balpha}} + O\pare{\kappa_g^2  \bar{\zeta}_i(\bw^*)   L^2\frac{2}{\mu}\left (1-\mu \gamma\right)^R   + 2\gamma \frac{\delta^2}{\mu^2 M}}} \\
    & \leq \tilde O\pare{   \pare{1-\frac{1}{\kappa_g}}^{T_{\balpha}} + O\pare{\kappa_g^2  \bar{\zeta}_i(\bw^*)   L^2\frac{2}{\mu}\frac{1}{R}  +   \frac{\delta^2}{\mu^3 R M}}}
\end{align*}
together with applying Lemma~\ref{lem:opt gap}  leads to:
\begin{align*}
        \Phi( {\balpha}^*_i, \hat{\bv}_i) - \Phi( {\balpha}^*_i,  {\bv}^*_i) &\leq  2{L} \norm{\hat{\bv}_i  -     \bv^*(\hat\balpha_i)}^2 + \pare{2\kappa^2_{\Phi} L + \frac{4NG^2}{ L} }\norm{ \hat\balpha_i  -   \balpha^*_i  }^2\\ 
        & \leq O\pare{ \frac{LD^2}{NKR^3} }   + \tilde{O} \pare{   \pare{  \frac{\kappa^4L}{  R^2}  + \frac{NL}{\mu^2 R^2}   } G^2 {N\log(1/p)} + \frac{L \delta^2}{\mu R}     } \\
        &\quad + \tilde{O}\pare{  \frac{\kappa^2_{\Phi}\kappa^2\kappa_g^2  L^3  \bar\zeta_i(\bw^*)  DG}{R } + \kappa^2_{\Phi}LR^2 {\pare{1-\frac{1}{\kappa_g}}^{T_{\balpha}}    +  \frac{L\kappa^2_{\Phi} \kappa_g^2  \kappa^2  \bar\zeta_i(\bw^*)   \delta^2}{\mu^2 M  }    } }\\
        &\quad+ \pare{2\kappa^2_{\Phi} L + \frac{4NG^2}{ L} }\tilde O\pare{   \pare{1-\frac{1}{\kappa_g}}^{T_{\balpha}} +  {\frac{ \kappa_g^2  \bar{\zeta}_i(\bw^*) \kappa  L }{  R}  +   \frac{\delta^2}{\mu^3 R M}}} \\
           & \leq O\pare{ \frac{LD^2}{NKR^3} }   + \tilde{O} \pare{   \pare{  \frac{\kappa^4L}{  R^2}  + \frac{NL}{\mu^2 R^2}   } G^2 {N\log(1/p)} + \frac{L \delta^2}{\mu R}     } \\
        &\quad + \tilde{O}\pare{  \frac{\kappa^2_{\Phi}\kappa^2\kappa_g^2  L^3  \bar\zeta_i(\bw^*)  DG}{R } + \kappa^2_{\Phi}LR^2 {\pare{1-\frac{1}{\kappa_g}}^{T_{\balpha}}    +  \frac{  L^2 \kappa^2\kappa_g^2  \kappa^2_\Phi  \bar\zeta_i(\bw^*)   \delta^2}{\mu^2 M  }    } }.
\end{align*}
thus completing the proof.
\end{proof}






\end{document}